\documentclass[onefignum,onetabnum]{siamart171218}

\newcommand{\E}{\mathbb{E}}

\usepackage{amssymb,amsfonts,amsmath}

\usepackage[mathscr]{eucal}

\usepackage{stmaryrd}

\usepackage{amsbsy}

\usepackage{bm}

\usepackage{braket}

\usepackage{xparse}

\NewDocumentCommand \RangeSet { G{N} } {[\mathinner{#1}]}

\NewDocumentCommand \LinearSymbol {} {\oast}
\NewDocumentCommand \LeftSymbol {} {\blacktriangleleft}
\NewDocumentCommand \RightSymbol {} {\blacktriangleright}
\NewDocumentCommand \UnfoldSymbol {} {\oslash}

\NewDocumentCommand \SingleSize {s O{J} G{n}} {\IfBooleanTF{#1}{\bar}{} {#2}_{#3}}
\NewDocumentCommand \LinearSize {s O{J}} {\IfBooleanTF{#1}{\bar}{} #2^{\LinearSymbol}}
\NewDocumentCommand \LeftSize {s O{J} G{n}} {\IfBooleanTF{#1}{\bar}{} {#2}_{#3}^{\LeftSymbol}}
\NewDocumentCommand \RightSize {s O{J} G{n}} {\IfBooleanTF{#1}{\bar}{} {#2}_{#3}^{\RightSymbol}}
\NewDocumentCommand \UnfoldSize {s O{J} G{n}} {\IfBooleanTF{#1}{\bar}{} {#2}_{#3}^{\UnfoldSymbol}}
\NewDocumentCommand \FullSize{s O{J} G{N}} {%
  \IfBooleanTF{#1}%
  {\SingleSize*[#2]{0} \times \SingleSize*[#2]{1} \times \cdots \times \SingleSize*[#2]{#3-1}}%
  {\SingleSize[#2]{0} \times \SingleSize[#2]{1} \times \cdots \times \SingleSize[#2]{#3-1}}%
}

\NewDocumentCommand \SizeVec {s O{J}} {\mathbf{\IfBooleanTF{#1}{\bar}{} {#2}}}

\NewDocumentCommand \SingleIndex {s O{j} G{n}} {\IfBooleanTF{#1}{\bar}{} {#2}_{#3}}
\NewDocumentCommand \FullIndex{s O{j} G{N}} {%
  \IfBooleanTF{#1}%
  {(\SingleIndex*[#2]{0}, \SingleIndex*[#2]{1}, \dots, \SingleIndex*[#2]{#3-1})}%
  {(\SingleIndex[#2]{0}, \SingleIndex[#2]{1}, \dots, \SingleIndex[#2]{#3-1})}%
}

\NewDocumentCommand \IndexVec {s O{j}} {\mathbf{\IfBooleanTF{#1}{\bar}{} {#2}}}

\NewDocumentCommand \FullSubscript{s O{j} G{N}} {%
  \IfBooleanTF{#1}%
  {\SingleIndex*[#2]{0} \SingleIndex*[#2]{1} \dots \SingleIndex*[#2]{#3-1}}%
  {\SingleIndex[#2]{0} \SingleIndex[#2]{1} \dots \SingleIndex[#2]{#3-1}}%
}

\NewDocumentCommand \FacMat {s O{U} G{n}} {{\mathbf{\IfBooleanTF{#1}{\bar}{} #2}}^{(#3)}}
\NewDocumentCommand \FacSize {O{I} G{n}} {#1_{#2} \times R_{#2}}

\NewDocumentCommand \Tensor {s O{Y}} {\boldsymbol{\IfBooleanTF{#1}{\bar}{}
{\mathscr{\MakeUppercase{#2}}}}}
\NewDocumentCommand \ColumnBlock {O{\Mz{Y}{n}} G{\ell}} {#1 \left[#2\right]}

\newcommand{\Tra}{{\sf T}}

\newcommand{\V}[2][]{{\bm{#1\mathbf{\MakeLowercase{#2}}}}} 
\newcommand{\M}[2][]{{\bm{#1\mathbf{\MakeUppercase{#2}}}}} 
\newcommand{\Mz}[3][]{\M[#1]{#2}_{(#3)}}

\usepackage{subcaption}
\usepackage{booktabs}
\usepackage{adjustbox}
\usepackage{csvsimple}
\usepackage{bbm}
\usepackage{comment}
\newcommand{\nno}[1]{{\mathbb{R}_{\ge 0}}} % nonegative orthant
\newcommand{\update}[1]{\text{Update}}
\DeclareMathOperator*{\argmin}{arg\,min}

\newenvironment{varalgorithm}[1]
{\algorithm[tb]\renewcommand{\thealgorithm}{#1}}
  {\endalgorithm}
\usepackage{lipsum}
\usepackage{amsfonts}
\usepackage{graphicx}
\usepackage{epstopdf}
\usepackage{algpseudocode}
\ifpdf
  \DeclareGraphicsExtensions{.eps,.pdf,.png,.jpg}
\else
  \DeclareGraphicsExtensions{.eps}
\fi

\newsiamremark{remark}{Remark}
\newsiamremark{hypothesis}{Hypothesis}
\crefname{hypothesis}{Hypothesis}{Hypotheses}
\newsiamthm{claim}{Claim}

\crefformat{equation}{Eq.~(#2#1#3)}

\headers{Randomized Algorithms for Symmetric NMF}{K. Hayashi, S. G. Aksoy, G. Ballard, H. Park}

\title{Randomized Algorithms for Symmetric Nonnegative Matrix Factorization
\thanks{Created on 08/12/2024.
\funding{Koby Hayashi acknowledges support from the United States Department of Energy through the Computational Sciences Graduate Fellowship (DOE CSGF) under grant number: DE-SC0020347. 
The authors would like to acknowledge the support provided by the National Science Foundation through grants OAC-2106920 and CCF-1942892. Information Release PNNL-SA-193926.}}}

\author{Koby Hayashi\thanks{School of Computational Science and Engineering, Georgia Institute of Technology, Atlanta, GA, USA
  (\email{khayashi9@gatech.edu}, \email{hpark@cc.gatech.edu}).}
\and Sinan G. Aksoy\thanks{Pacific Northwest National Laboratory, Seattle, WA, USA 
  (\email{sinan.aksoy@pnnl.gov})}
  \and Grey Ballard\thanks{Dept. of Computer Science, Wake Forest University, Winston-Salem, NC, USA (\email{ballard@wfu.edu})}
\and Haesun Park\footnotemark[2] }

\usepackage{amsopn}

\makeatletter
\newcommand*{\addFileDependency}[1]{% argument=file name and extension
  \typeout{(#1)}% latexmk will find this if $recorder=0 (however, in that case, it will ignore #1 if it is a .aux or .pdf file etc and it exists! if it doesn't exist, it will appear in the list of dependents regardless)
  \@addtofilelist{#1}% if you want it to appear in \listfiles, not really necessary and latexmk doesn't use this
  \IfFileExists{#1}{}{\typeout{No file #1.}}% latexmk will find this message if #1 doesn't exist (yet)
}
\makeatother

\newcommand*{\myexternaldocument}[1]{%
    \externaldocument{#1}%
    \addFileDependency{#1.tex}%
    \addFileDependency{#1.aux}%
}
\ifpdf
\hypersetup{
  pdftitle={Algorithms for Randomized Symmetric Nonnegative Matrix Factorization},
  pdfauthor={K. Hayashi, S. G. Aksoy, G. Ballard, H. Park}
}
\fi

\myexternaldocument{ex_supplement}

\begin{document}

% trying to add indents to algorithmic
% https://newbedev.com/indenting-lines-of-code-in-algorithm

\maketitle

% REQUIRED
\begin{abstract}
Symmetric Nonnegative Matrix Factorization (SymNMF) is a technique in data analysis and machine learning that approximates a symmetric matrix with a product of a nonnegative, low-rank matrix and its transpose.
To design faster and more scalable algorithms for SymNMF we develop two randomized algorithms for its computation.
The first algorithm uses randomized matrix sketching to compute an initial low-rank approximation to the input matrix and proceeds to rapidly compute a SymNMF of the approximation.
The second algorithm uses randomized leverage score sampling to approximately solve constrained least squares problems.
Many successful methods for SymNMF rely on (approximately) solving sequences of constrained least squares problems.
We prove theoretically that leverage score sampling can approximately solve nonnegative least squares problems to a chosen accuracy with high probability.
Additionally, we prove sampling complexity results for previously proposed hybrid sampling techniques which deterministically include high leverage score rows.
This hybrid scheme is crucial for obtaining speeds ups in practice.
Finally we demonstrate that both methods work well in practice by applying them to graph clustering tasks on large real world data sets.
These experiments show that our methods approximately maintain solution quality and achieve significant speed ups for both large dense and large sparse problems.
\end{abstract}

% REQUIRED

% REQUIRED

\section{Introduction}
We propose the first randomized algorithms for Symmetric Nonnegative Matrix Factorization (SymNMF).
Nonnegative Matrix Factorization \\ (NMF) is an important method in data analysis with applications to data visualization, text mining, feature learning, information fusion and more \cite{nature_NMF,DaKuang_SymNMF,doc_clust_NMF,jnmf_kdd,dcnmf_rundong}.
SymNMF is a variant of NMF where the input matrix is symmetric and the output low-rank approximation is also constrained to be symmetric \cite{DaKuang_SymNMF,Dropping_Sym4SymNMF}.
Applications of SymNMF include (hyper)graph clustering, image segmentation, and information fusion \cite{mega_VLDB,JNMF_jogo,edvw_hayashi,choi_dongjin_infofusion,choi2023wellfactor}.
Several randomized algorithms for nonsymmetric NMF have been previously proposed and shown to be effective for dense and small sparse problems \cite{RP_NMF2010,Comp_NMF2016,rnd_HALS}, but as far as we are aware there is no prior work on randomized algorithms for SymNMF.
Along the way we also prove two interesting results generalizing existing guarantees for leverage score sampling for overdetermined least squares problems to all convex overdetermined least squares problems (e.g. nonnegative least squares) and prove theoretical guarantees for hybrid leverage score sampling schemes applied to such problems.
Our contributions include:
\begin{itemize}
    \item a randomized algorithm for SymNMF we call ``Low-rank Approximated Input SymNMF" (LAI-SymNMF),
    \item a randomized algorithm based on leverage score sampling for least squares problems we call LvS-SymNMF,
    \item novel theoretical analysis of leverage score sampling for the \textit{Nonnegative Least Squares} problem and more generally convex least squares problems,
    \item theoretical analysis of a hybrid leverage score sampling scheme for convex least squares problem,
    \item and experiments on two large, real world clustering tasks.
\end{itemize}

The rest of the paper is organized as follows. 
\Cref{sec:prelims_and_rw}, which discusses background material including non-randomized SymNMF algorithms, reviews existing randomized NMF methods and other related work such as randomized methods for other low-rank matrix and tensor decompositions.
\Cref{sec:LAI_NMF} introduces our first proposed algorithm LAI-SymNMF.
LAI-SymNMF uses randomized methods to rapidly compute an initial, unconstrained low-rank approximation before proceeding to compute an NMF of this LAI.
\Cref{sec:LVS_symNMF} presents a row-sampling algorithm called LvS-SymNMF. This method solves a sequence of nonnegative least squares (NLS) problems using a technique called leverage score sampling to accelerate the solver.
Additionally, we use a hybrid approach from \cite{lev_scores4_CP} that involves both deterministic and randomized sampling based on the leverage scores.
Our novel theoretical analysis of this hybrid strategy gives the sample complexity needed to achieve an accuracy guarantee relative to the NLS residual, with high probability, and we empirically show its advantage over purely randomized sampling.
\Cref{sec:experiments} presents experimental results for the proposed algorithms on two real world data sets.
Each data set is represented as a graph and the SymNMF output is used to cluster the vertices.
Both methods achieve significant speed ups over deterministic methods ranging from $5$ to $7.5\times$ and are able to maintain accuracy in terms of normalized residual norms and cluster quality.
\section{Preliminaries}
\label{sec:prelims_and_rw}
We begin by briefly discussing the NMF and SymNMF problems followed by an introduction to various methods for computing NMF's and SymNMF's.
Then we discuss the ``Randomized Range Finder'' and leverage score sampling for least squares problems.
The section also contains a statement of \cref{thm:nls_bnd} which provides novel bounds for Leverage Score sampling applied to Nonnegative Least Squares Problems.
Various notation is given in \cref{tab:notation_table}.
\begin{table}[]
\centering
\resizebox{\columnwidth}{!}{
\begin{tabular}{| c | c || c | c |}
\hline
\text{Symbols} & \text{Meaning} & \text{Symbols} & \text{Meaning} \\
\hline
$\M{I}$ & Identity Matrix & $\V{a}_i$ & $i$th column vector of $\M{A}$ \\
$\M{\Omega}$ & Gaussian Matrix & $l_i(\M{A})$ & the $i$th leverage score of $\M{A}$\\
$\otimes$ & Kronecker Product & $p_i$ & $i$th leverage score sampling probability \\
$\M{S}$ & Sketching Matrix & $[n]$ & index set 1 to $n$ \\
$[\M{A}]_+$ & Proj. to nonnegative orthant & $\mathscr{A}$ & A set, Euler script\\
$\M{X}$ & NMF Data Matrix & $\M{A}$ & A matrix, bold-uppercase\\
$\M{W}$ & Left NMF Factor & $\V{a}$ & A vector, bold-lowercase\\
$\M{H}$ & Right NMF Factor & $|\cdot|$ & Absolute value or cardinality\\
$\M{A}^\Tra$ & Transposition & $\mathbb{R}_{+}$ & Nonnegative Real Numbers \\
$\M{Q}_A$ & Orthonormal basis for range of $\M{A}$ & $\kappa(\M{A})$ & Condition Number of $\M{A}$ \\
$\| \cdot \|_F$ & Frobenius norm & $\|\cdot \|_2$ & 2-norm\\
$\sigma_i(\M{A})$ & $i$th largest singular value of $\M{A}$ & $\sigma_{max}(\M{A}),\sigma_{min}(\M{A})$ & largest and smallest singular value\\
\hline
\end{tabular}
} % scale box
\caption{Notation}
\label{tab:notation_table}
\end{table}

Standard NMF is formulated as
\begin{equation}
    \label{eq:nmf_obj}
\min_{\{\M{W},\M{H}\}\ge0}\| \M{X} - \M{W}\M{H}^\Tra \|_F
\end{equation}
where $\M{X} \in \mathbb{R}^{m \times n}$, $\M{W} \in \mathbb{R}_{+}^{m \times k}$, and $\M{H} \in \mathbb{R}_{+}^{n \times k}$.
The notation $\mathbb{R}_{+}$ denotes the nonnegative orthant and $k$ is the desired reduced rank, given as an input, and 
usually $k \ll \min(m,n)$.
While $\M{X}$ is often also nonnegative, it is not strictly required.
Since NMF is a non-convex optimization problem, optimization algorithms often find a local minimum at best.
The works \cite{gillis_book,BCD_park} give a comprehensive discussion of NMF.

When $\M{X}$ is symmetric, $\M{X} = \M{X}^\Tra$, 
it is often desired the two low-rank factors in \Cref{eq:nmf_obj} be the same \cite{DaKuang_SymNMF,gillis_symNMF,Dropping_Sym4SymNMF,symnmf_SDM}.
This problem is called Symmetric NMF (SymNMF) and its objective function is expressed as
\begin{equation}
    \label{eq:symnmf_obj}
\min_{\M{H}\ge0}\| \M{X} - \M{H}\M{H}^\Tra \|_F .
\end{equation}
SymNMF has found applications to (hyper)graph clustering, image segmentation, and community detection in social networks \cite{DaKuang_SymNMF,edvw_hayashi,hier_NMF_rundong}.
For a relationship between SymNMF and spectral clustering see \cite{JNMF_jogo}.
%Algorithms for SymNMF are often more complicated than those for NMF.
%Intuitively this is because of increased nonlinearity in $\M{H}$.
%because of the interactions among the variables of $\M{H}$ with themselves, as opposed to those between $\M{W}$ and $\M{H}$ as in standard NMF.

\subsection{Algorithms for SymNMF}
Many algorithms for solving SymNMF shown in \Cref{eq:symnmf_obj} have been proposed.
Broadly these methods can be put into two categories: Alternating Updating (AU) methods and all-at-once optimization methods.
AU methods alternate between updating a subset of the variables while holding others fixed, eventually iterating through the entire subset of variables.
All-at-once methods update all of the variables simultaneously. 

AU methods for SymNMF include symmetrically regularized Alternating Nonnegative Least Squares (ANLS) \cite{DaKuang_SymNMF}, symmetrically regularized Hierarchical Least Squares (HALS) \cite{Dropping_Sym4SymNMF}, Cyclic Coordinate Descent (CCD) \cite{gillis_symNMF}, and Progressive Hierarchical Alternating Least Squares (PHALS) \cite{phals_Hou}.
All-at-once methods include Projected Gradient Descent (PGD), Projected Newton-like update \cite{DaKuang_SymNMF}, and Projected Gauss-Newton with Conjugate Gradients (PGNCG) \cite{srinivas_SC20_GN}.
We focus on methods based on regularized ANLS and the PGNCG algorithm.
This is because the ANLS method is generally superior to both PGD and the Newton-like method \cite{DaKuang_SymNMF}, and
the CCD method tends to be unsuitable for large data sets as it sequentially iterates over the elements of $\M{H}$.
We now present the methods based on regularized ANLS and the PGNCG algorithm.

\subsubsection{Symmetrically Regularized Alternating Nonnegative Least \\ Squares for SymNMF}
\label{sec:sym_anls}
The regularized ANLS and HALS methods for SymNMF are based on solving a surrogate problem of the form
\begin{equation}
\label{eq:surrogate_symnmf_obj}
    \min_{ \{\M{W},\M{H}\} \ge 0}\|\M{X} - \M{W}\M{H}^\Tra \|_F^2 + \alpha\|\M{W} -\M{H}\|_F^2,
\end{equation}
where $\M{W}$ and $\M{H}$ are forced to be close to each other in Frobenius norm using a large value of $\alpha$, thus guiding the iterates towards a symmetric approximation \cite{DaKuang_SymNMF}. 
The authors of \cite{Dropping_Sym4SymNMF} show that under mild assumptions the critical points of \Cref{eq:symnmf_obj} and \Cref{eq:surrogate_symnmf_obj} are the same and argue that this makes \Cref{eq:surrogate_symnmf_obj} an appropriate surrogate for \Cref{eq:symnmf_obj}. \Cref{eq:surrogate_symnmf_obj} can be iteratively approximated by using the ANLS method for updating $\M{W}$ and  $\M{H}$ as in the following two equations :
\begin{align}
\label{eq:ANLS4symNMF}
    \min_{\M{W}\ge0}\Big\|\begin{bmatrix}
    \M{H}\\
    \sqrt{\alpha}\M{I}
    \end{bmatrix}
    \M{W}^\Tra - 
     \begin{bmatrix}
    \M{X}\\
    \sqrt{\alpha}\M{H}^\Tra
    \end{bmatrix} \Big\|_F \ \text{and} \ 
    \min_{\M{H}\ge0}\Big\|\begin{bmatrix}
    \M{W}\\
    \sqrt{\alpha}\M{I}
    \end{bmatrix}
    \M{H}^\Tra - 
     \begin{bmatrix}
    \M{X}\\
    \sqrt{\alpha}\M{W}^\Tra
    \end{bmatrix} \Big\|_F.
\end{align}
This approach enables using many of the tools for standard NMF for the SymNMF problem.
For example, the two equations shown in \Cref{eq:ANLS4symNMF} can be solved as Nonnegative Least Squares (NLS) problems.

The ANLS method is generally superior to PGD and the Newton-like method \cite{DaKuang_SymNMF}. 
This is mainly because PGD suffers from slow convergence and the Newton-like method is expensive for even moderately sized problems due to the need for approximate Hessian inversion. 
To solve the ANLS formulation we use the Block Principle Pivoting (BPP) solver from \cite{Jinug_BPP}.

\subsubsection{Hierarchical Alternating Least Squares for SymNMF}
\label{sec:hals_4symnmf}
The HALS method for NMF was developed in \cite{gillis_hals,hals_cichocki}.
A method that uses the HALS framework for optimizing \Cref{eq:surrogate_symnmf_obj} was proposed in \cite{Dropping_Sym4SymNMF}.
The update rules are given by the following two equations 
\begin{equation}
\label{eq:symhals_update1}
    \V{w}_i = \max\Big( \frac{(\M{R}_{i} + \alpha \M{I})\V{h}_i}{\|\V{h}_i\|_2^2 + \alpha},0 \Big) \ \text{and} \ \V{h}_i = \max\Big( \frac{(\M{R}_i + \alpha \M{I})\V{w}_i}{\|\V{w}_i\|_2^2 + \alpha},0 \Big).
\end{equation}
Here $\M{W} = \big[\V{w}_1, \cdots, \V{w}_k\big]$, $ \M{H} = \big[\V{h}_1, \cdots, \V{h}_k\big]$ are the columns of $\M{W}$ and $\M{H}$ and $ \M{R}_{ i} = \M{X} - \sum_{j \neq i, j=1}^k \V{w}_j \V{h}_j^\Tra.$
Following these rules, the columns of $\M{W}$ and $\M{H}$ are updated as pairs in sequence as $\V{w}_1, \V{h}_1,\cdots \V{w}_i, \V{h}_i, \cdots \V{w}_k, \V{h}_k$.
This method is less efficient as it relies on the matrix $\M{R}_i$, as in \Cref{eq:symhals_update1}.
%Additionally it is not straight forward to batch updates using the products $\M{X}\M{H}$ and $\M{X}^\Tra\M{W}$.
To make the updates efficient we modify the update rule following algorithms from \cite{gillis_hals} for more efficient HALS updates to
\begin{align}
\label{eq:symhals_update2}
        \V{w}_i \leftarrow \Big( \frac{(\M{X} - \M{W}\M{H}^\Tra + \alpha\M{I})\V{h}_i} {\|\V{h}_i\|_2^2 + \alpha} + \frac{\|\V{h}_{i}\|_2^2 }{\|\V{h}_i\|_2^2 + \alpha}\V{w}_{i}\Big)_+, \\
        \V{h}_i \leftarrow \Big( \frac{(\M{X} - \M{H}\M{W}^\Tra + \alpha\M{I})\V{w}_i} {\|\V{w}_i\|_2^2 + \alpha} + \frac{\|\V{w}_{i}\|_2^2 }{\|\V{w}_i\|_2^2 + \alpha}\V{h}_{i}\Big)_+
\end{align}
which are mathematically equivalent to those in \Cref{eq:symhals_update1}.
These rules can be seen as combining the update order for HALS from \cite{hals_cichocki} and regularization, such as used in \cite{cohen2018nonnegative}.
A derivation can be found in the \cref{app_efficient_HALS_update_derivation}.
This formulation allows for the updates of all the columns of $\M{W}$ then all the columns of $\M{H}$ or vice versa.
As discussed in \cite{gillis_hals}, choosing to update all $\V{w}_i$'s followed by all $\V{h}_i$'s allows for the products $\M{W}^\Tra\M{X}$ and $\M{X}\M{H}$ to be computed and reused through a single sweep of updates over all columns of $\M{W}$ and $\M{H}$ resulting in better computational efficiency in practice.
To illustrate this, consider the update for $\V{w}_i$.
The bulk of the computation is needed for computing $(\M{X} - \M{W}\M{H}^\Tra + \alpha\M{I})\V{h}_i$ which is the $i$th column of $(\M{X}\M{H} - \M{W}\M{H}^\Tra\M{H} + \alpha\M{H})$, where the product $\M{X}\M{H}$ will not change as each $\V{w}_i$ is updated.
The same applies for the columns of $\M{H}$.
Overall our proposed updates shown in \Cref{eq:symhals_update2} are more memory efficient and more computationally efficient by a factor of 2.

\subsubsection{Projected Gauss-Newton with Conjugate Gradients for SymNMF}
The algorithm PGNCG-SymNMF was proposed for efficiently computing SymNMF in highly parallel computing environments \cite{srinivas_SC20_GN}.
It is an all-at-once method and uses the Projected Gauss-Newton method to directly optimize the SymNMF objective, \Cref{eq:symnmf_obj}.
The main computational load lies in solving a least squares problem of the form 
\begin{equation*}
    \min_{\V{p}} \| \M{J}\V{p} + \V{r} \|_2
\end{equation*}
for a search direction $\V{p}$ at every iteration.
The matrix $\M{J}$ and the vector $\V{r}$ are the Jacobian and residual of \Cref{eq:symnmf_obj} respectively.
A solution to this LS problem is then approximated using the Conjugate Gradient (CG) method on the Normal Equations : 
$\M{J}^\Tra\M{J} \V{p} = - \M{J}^\Tra\V{r}$.
A core computational kernel of the CG method is computing matrix vector products with the matrix $\M{J}^\Tra\M{J}$.
Fortunately, the Jacobian $\M{J}$ has the Kronecker product form $\M{J} = -(\M{H} \otimes \M{I}_m) - \M{P}_{m,m}(\M{H} \otimes \M{I}_m)$, where $\M{P}_{m,m}$ is the perfect shuffle or ``vec" permutation, which can be used to efficiently apply $\M{J}$ to a vector.
Additionally the vector $\V{g} = \M{J}^\Tra\V{r}$ has the form $\V{g} = -2\text{vec}(\M{X}\M{H} - \M{H}(\M{H}^\Tra\M{H}))$, which is typically the main computational bottleneck and requires the matrix multiplications $\M{X}\M{H}$ and $\M{H}^\Tra\M{H}$. 
See \cite{srinivas_SC20_GN} for details.

The PGNCG method is competitive with the ANLS and CCD method for SymNMF \cite{srinivas_SC20_GN}.
The PGNCG method generally converges much faster than PGD as it approximates second-order derivatives and does not suffer from large computational complexity, as the Newton-like algorithm does, due to the exploitation of the Jacobian's structure for use in the CG iterations.

\subsection{Sketching in Numerical Linear Algebra}
Randomized Numerical Linear Algebra (RndNLA) is an important area of research with practical applications in finding fast approximate solutions to linear systems, least squares problems, eigenvalue problems, among others.
Surveys on this topic include \cite{halko_survey,tropp_RndNLA_Survey2}.
There are two main tools we will use from the RndNLA literature.
The first is the Randomized Range Finder (\ref{alg:QB_decomposition}) \cite{halko_survey} which has many applications in RandNLA such as computing approximate, truncated Singular Value Decompositions (SVD's) and Symmetric Eigenvalue Decompositions (EVD's).
The second is leverage score sampling for approximately solving least squares problems \cite{sketching4NLA_woodruff,rnd_algs4_mtrx_and_data_mahoney}.

\subsubsection{Randomized Range Finder}
The \ref{alg:QB_decomposition} is a method for finding an approximate orthonormal basis for the range space of a matrix.
It is the foundation for many randomized methods in RandNLA, such as computing an approximate, truncated SVD in a randomized way \cite{halko_survey}.
\begin{varalgorithm}{RRF}
 \caption{: Randomized Range Finder}
 \label{alg:QB_decomposition} 
\begin{algorithmic}[1]
\Require{: data matrix $\M{X} \in\mathbb{R}^{m \times n}$, target rank $r$, oversampling parameter $\rho$, and an exponent $q$}
\Ensure{$\M{Q}_Y \in \mathbb{R}^{m \times (r + \rho)}$, is an approximate orthonormal basis for the leading column span of $\M{X}$:}
\Function{$[\M{Q}_Y] =$ RRF}{$\M{X},r,\rho,q$}
 \State{$l := r + \rho$} \Comment{$l$ is the rank of the approximation being computed}
 \State{Draw a Gaussian Random matrix $\M{\Omega} \in \mathbb{R}^{n \times l}$}
 \State{Compute $\M{Y} := (\M{X}\M{X}^\Tra)^q\M{X}\M{\Omega} \in \mathbb{R}^{n \times l}$}
 \State{Compute $ \M{Y} = \M{Q}_Y\M{R}_Y$ a thin-QR Decomposition of $\M{Y}$ where $\M{Q}_Y \in \mathbb{R}^{n \times l}$ and $\M{R}_Y \in \mathbb{R}^{l \times l}$ }
 \EndFunction
\end{algorithmic}
\end{varalgorithm}
An algorithm outline for the \ref{alg:QB_decomposition} is given in \cref{alg:QB_decomposition}.
Parameters of the \ref{alg:QB_decomposition} are the target rank $r$, a column oversampling parameter $\rho$, and $q$, the number of power iterations to perform.
The computational complexity of the \ref{alg:QB_decomposition} is $O( q mnl + ml^2)$ where $l=r+\rho$.
The approximate output from the \ref{alg:QB_decomposition} is often used to compute a ``QB-Decomposition''.
If a matrix $\M{X}$ is input to the \ref{alg:QB_decomposition} and a matrix $\M{Q}_X$ is output, then $\M{Q}_X\M{Q}_X^\Tra\M{X} = \M{Q}_X\M{B}_X$, where $\M{B}_X = \M{Q}_X^\Tra\M{X}$, is called a QB-Decomposition of $\M{X}$.

\begin{varalgorithm}{Apx-EVD}
%\begin{algorithm}[tb]
 \caption{: Approximate Truncated Eigenvalue Decomposition of Symmetric Matrix}
 \label{alg:apx_evd} 
\begin{algorithmic}[1]
\Require{: symmetric matrix $\M{X} \in\mathbb{R}^{m \times m}$, target rank $r$, oversampling parameter $\rho$, and exponent $q$}
\Ensure{$\M{U}\in\mathbb{R}^{m\times (r+\rho)}$ and $\M{\Lambda} \in\mathbb{R}^{ (r+\rho)\times(r+\rho)}$ where $\M{X} \approx \M{U}\M{\Lambda} \M{U}^\Tra$ is an approximate EVD of $\M{X}$.}
\Function{$[\M{\M{U},\M{\Lambda}]} =$ Apx-EVD}{$\M{X},r,\rho,q$}
 \State{$l := r + \rho$} \Comment{$l$ is the rank of the approximation being computed}
 \State{$\M{Q}_X$ := RRF($\M{X},r,\rho,q$), where  $\M{Q}_X\in\mathbb{R}^{m\times l}$}
 \State{Compute $\M{T} := \M{Q}_X^\Tra\M{X}\M{Q}_X \in \mathbb{R}^{l \times l}$}
 \State{Compute $\M{T} =\M{Q}_T\M{\Lambda}\M{Q}_T^\Tra$ an eigenvalue decomposition of $\M{T}$ where $\M{Q}_T\in\mathbb{R}^{l\times l}$ and $\M{\Lambda}\in\mathbb{R}^{l\times l}$}
 \State{Compute $\M{U} := \M{Q}_X\M{Q}_T \in\mathbb{R}^{m\times l}$}
 \EndFunction
\end{algorithmic}
%\end{algorithm}
\end{varalgorithm}

For a symmetric input, an approximate eigenvalue decomposition can also be obtained by using the approximate basis from the \ref{alg:QB_decomposition}.
This procedure is shown in \cref{alg:apx_evd}, which stands for approximate eigenvalue decomposition.
More details and references can be found in \cite{halko_survey}.

\subsubsection{Leverage Score Sampling for Ordinary Least Squares}
Another approach that we will use from RandNLA is sketching for ordinary least squares (OLS) problems, specifically, leverage score sampling for OLS problems.
%OLS problems are ubiquitous in mathematically mature sciences.
%Randomized techniques for solving least squares problems have received much attention in the theoretical computer science community \cite{rnd_algs4_mtrx_and_data_mahoney,sketching4NLA_woodruff}.
A standard OLS or $l2-$regression problem is
\begin{equation}
    \label{eq:ls_def}
    \min_{\V{x} \in \mathbb{R}^k}\| \M{A}\V{x}- \V{b} \|_2
\end{equation}
where  $\M{A} \in \mathbb{R}^{m \times k}$, $\V{x}\in\mathbb{R}^{k}$, and $\V{b}\in\mathbb{R}^m$.
Our focus will be on overdetermined OLS problems, where $m \gg k$, and $\M{A}$ has full rank.

The \text{sketch and solve} paradigm \cite{tropp_RndNLA_Survey2} for OLS methods takes the form
\begin{equation}
    \label{eq:rnd_ls_def}
    \hat{\V{x}}_{ols} = 
    \argmin_{\V{x}}\| \M{S}\M{A}\V{x} - \M{S}\V{b} \|_2,
\end{equation}
where $\M{S} \in \mathbb{R}^{s\times m}$, with $s \ll m$, is called a \textit{sketching} matrix.
Computational savings come from the fact that one can now solve the smaller problem in \Cref{eq:rnd_ls_def}
as opposed to the full sized problem in \Cref{eq:ls_def}.

There are many ways to generate the sketching matrix $\M{S}$.
We focus on when $\M{S}$ is a row-sampling matrix generated according to the leverage score distribution of $\M{A}$.
Leverage score sampling is a well-studied method for sketching OLS problems \cite{rnd_algs4_mtrx_and_data_mahoney}.
In this method the leverage scores (see \Cref{eq:lev_score_def}) are used to define a probability distribution over the rows of the matrix $\M{A}$. 
That is, some number of rows, say $s$ rows, of the matrix $\M{A}$ are sampled with replacement with probability proportional to the value of their leverage scores.
The leverage score of the $i$th row of a matrix $\M{A}$ is defined as 
\begin{equation}
    \label{eq:lev_score_def}
    l_i(\M{A}) = \|\M{Q}_A[i,:] \|_2^2
\end{equation}
where the matrix $\M{Q}_A$ is any othonormal basis for the column space of $\M{A}$ and $\M{Q}_A[i,:]$ is the $i$th row of $\M{Q}_A$.
For example the matrix $\M{U}_A$, where $\M{U}_A\M{\Sigma}_A\M{V}^\Tra_A = \M{A}$ is a thin SVD of $\M{A}$, can be used to calculate the leverage scores.
These values are normalized into probabilities $p_i = \frac{l_i(\M{A})}{\| \M{Q}_A\|_F^2}$.
Using the $p_i$'s, $s$ samples are drawn with replacement and the matrix $\M{S} \in \mathbb{R}^{s \times m}$ is formed as
\begin{equation}
\label{eq:lev_score_S}
\M{S}_{ji} =
\begin{cases}
    \frac{1}{\sqrt{sp_i}}, \ \text{if row $i$ was drawn as the $j$th sample}\\
    0
\end{cases}.
\end{equation}
Due to the special form of $\M{S}$, $\M{S}\M{A}$ does not require matrix-matrix multiplication but only row selection and scaling. 
Computing the leverage scores of $\M{A}$ via a full matrix factorization, such as QR or SVD, costs $O(mk^2)$.
This makes solving the smaller problem in \Cref{eq:rnd_ls_def} just as expensive as the original LS problem in \Cref{eq:ls_def}, in the case of a single right hand side (RHS).
To deal with this, schemes for quickly computing approximate leverage scores have been proposed \cite{fast_apx_lev_scores}.
Additionally, sometimes special structures in the coefficient matrix, $\M{A}$, can be exploited to obtain fast leverage score estimates \cite{lev_scores4_CP,SPALS_neurips}.
%Unfortunately neither of these techniques are applicable to our problem. 
%In the case of \cite{fast_apx_lev_scores} the method requires that $k\ln(k) = o(m \ln(m))$ which is usually not the case for our applications.

When leverage score sampling is used, with an appropriate number of samples $s$, the solution to the sampled problem in \cref{eq:rnd_ls_def} satisfies the bound
\begin{equation}
\label{eq:ols_bnd}
    \|\hat{\V{x}}_{ols} - \V{x}_{ols}\|_2
    \leq \sqrt{\epsilon_r}\frac{\|\V{r}_{ols}\|_2}{\sigma_{min}(\M{A})}
\end{equation}
with probability $1-\delta$ where $\V{r}_{ols} = \M{A}\V{x}_{ols} - \V{b}$ and $\V{x}_{ols}$ is the minimizer of \cref{eq:ls_def}.
Additionally, $\delta$ and $\epsilon_r$ are values between 0 and 1 which control the failure probability and approximation error respectively. 
The number of samples $s$ is a function of these values and lower values of $\delta$ and $\epsilon_r$ incur a higher value of $s$.
That is for lower failure probability and/or smaller error one must take more samples.
Details can be found in \cite{larsen2022sketching}.

\subsubsection{Leverage Score Sampling for Nonnegative Least Squares}
\label{sec:lvs_nls_thm}
In order to apply leverage score sampling for Least Squares problems to NMF we extend results for Leverage Score Sampling for OLS problems to Nonnegative Least Squares Problems.
The Nonnegative Least Squares problem has the general form
\begin{equation}
    \label{eq:nls}
    \min_{\V{x} \ge 0 }\|\M{A}\V{x} - \V{b}\|_2.
\end{equation}
NLS problems are quite different from OLS problems as they do not yield a closed form solution.
Despite this we are able to prove an analogous bound for sampling NLS problems with leverage scores.
The statement of this result is given by the following theorem
%\footnote{\textcolor{red}{$\epsilon_r$ and $\delta$ can be thought of as user input parameters where $\epsilon_r$ controls the error and $\delta$ is the failure probability.}}
\begin{theorem}
\label{thm:nls_bnd}
Let  $\V{x}_{nls} = \argmin_{ \V{x} \ge 0} \|\M{A}\V{x} - \V{b} \|_2$ be a NLS solution where $\M{A} \in \mathbb{R}^{m \times k}$, $m > k$, and $\text{rank}(\M{A})=k$. Also let $\M{S} \in \mathbb{R}^{s \times m}$ be a leverage score sampling matrix for $\M{A}$ as in \cref{eq:lev_score_S} with $s$ samples satisfying
\begin{equation*}
    s \ge k \max(C \log(k/\delta),1/(\delta \epsilon_r)) \ \ \text{where} \ \ C = 144/(1-\sqrt{2})^2
\end{equation*}
for some $\epsilon_r,\delta \in (0,1)$.
Also let $\hat{\V{x}}_{nls} = \argmin_{\V{x} \ge 0} \|\M{S}\M{A}\V{x} - \M{S}\V{b} \|_2$ be the sampled NLS solution.
Then with probability $1-\delta$, the following holds : 
\begin{equation*}
\label{eq:cls_bnd_detailed}
    \|\hat{\V{x}}_{nls} - \V{x}_{nls}\|_2
    \leq \sqrt{\epsilon_r}\frac{\|\V{r}_{nls}\|_2}{\sigma_{min}(\M{A})}
\end{equation*}
where $\V{r}_{nls} = \M{A}\V{x}_{nls} - \V{b} $ and $\sigma_{min}(\M{A})$ is the minimum singular value of $\M{A}$\footnote{We note that \cref{thm:nls_bnd} does not include a leverage score ``missestimation factor'' ($\beta$ in \cite{sketching4NLA_woodruff}) which is often included in works concerning leverage score sketching for OLS problems \cite{larsen2022sketching,Mahoney16,sketching4NLA_woodruff}.
When inexact leverage scores are used for sampling, the missestimation factor gives a measure of how close the inexact leverage scores are to the true leverage scores.
We do not thoroughly discuss missestimation factors because we do not use the concept in this work.
However, \cref{thm:nls_bnd} can be easily generalized to incorporate such a factor.}.
\end{theorem}
This result immediately begets the idea to use randomized NLS methods to rapidly solve the NLS subproblems in Alternating Nonnegative Least Squares methods for NMF and SymNMF previously described in \Cref{sec:sym_anls}.
We described such a randomized algorithm in \Cref{sec:LVS_symNMF}.
A full proof of \Cref{thm:nls_bnd} is given in \Cref{sec:theory_nls}.
We note that \cref{thm:nls_bnd} holds for all convex Least Squares problems, this fact is made apparent in the proof given in \Cref{sec:theory_nls}.

\subsection{Related Work}
\label{sec:related_work}
There has been much work on speeding up and scaling algorithms for low-rank methods mostly focusing on parallel algorithms and, more recently, randomization.
We now give a brief review of related work with a focus on randomized methods.

\subsubsection{Randomized Methods for other Low-rank Approximations}
One of the first applications of leverage score sampling to matrix low-rank approximation was for computing a CUR decomposition \cite{CUR_Mahoney}. 
The CUR decomposition selects actual rows and columns of the matrix to produce a low-rank approximation. 
Since the CUR decomposition is composed of sampled rows and columns it is said to be interpretable in the ``space'' of the original data.

There has been a lot of work related to randomized algorithms for computing the Canonical Polyadic (CP) decomposition for tensors.
The CP Decomposition decomposes a $N$-way tensor into a sum of rank one tensors.
Various constraints can be imposed similarly to low-rank matrix approximations, for example the Nonnegative CP decomposition.
Two of the first methods for computing a randomized CP decomposition are by Battaglino et al. \cite{casey_rndCP} and Zhou 
et al. \cite{Zhou2014DecompositionOB}.
Erichson et al. \cite{Erichson_RndCP} proposed a randomized method for computing a CP decomposition based on using a tensor version of the \ref{alg:QB_decomposition}.
The main idea is to first compress the tensor, via the \ref{alg:QB_decomposition}, compute a CP of the compressed tensor, and then lift the compressed CP back to the uncompressed space.
Leverage scores have also been used in randomized CP algorithms \cite{lev_scores4_CP,SPALS_neurips,bharadwaj2023fast}.
There has also been work on scaling low-rank approximation methods to distributed computing environments \cite{planc_srinivas,mpi_faun_ramki,kanna_mpifaun2} and combining such methods with randomized techniques \cite{bharadwaj2023distributedmemory}.
These methods are applicable to sparse input tensors as they perform explicit sampling of tensor elements, thus preserving the sparsity pattern.

\subsubsection{Randomized NMF Algorithms}
Existing randomized algorithms for NMF focus on compressing the input matrix $\M{X}$.
Speed up is obtained from the fact that iteratively updating the low-rank factors for a compressed version of $\M{X}$ is cheaper.
The first proposed method we are aware of used a random Gaussian matrix to sketch, once from each side of $\M{X}$, resulting in two sketched matrices with smaller dimensions \cite{RP_NMF2010}.
Tepper and Sapiro \cite{Comp_NMF2016} proposed a similar method but used the \ref{alg:QB_decomposition} to compute an approximate basis for the row and column spans of $\M{X}$.
Their results showed that using the \ref{alg:QB_decomposition} significantly increased the accuracy of the randomized NMF results in terms of final residual.
Erichson et al. \cite{rnd_HALS} proposed a randomized HALS algorithm also based on the \ref{alg:QB_decomposition}.
This method computes a single \ref{alg:QB_decomposition} of the input matrix $\M{X}$ and fits NMF to the resulting QB-decomposition.
This method has the advantage of calling the \ref{alg:QB_decomposition} once and avoids sketching the factor matrices of NMF at each iteration.
However, due to the way the problem size is reduced, imposing nonnegativity to the factors of the original matrix becomes an issue. 

\section{NMF with Low-rank Approximate Input}
\label{sec:LAI_NMF}
Our first proposed algorithm is a method called Low-rank Approximate-Input NMF (LAI-SymNMF).
LAI-NMF computes an NMF of a low-rank approximation of the initial data matrix $\M{X}$.
The objective function for LAI-NMF is 
\begin{equation}
    \label{eq:alg_lai_nmf}
    \min_{ \{\M{W},\M{H} \} \ge0}\|\M{U}_{X}\M{V}_{X} -  \M{W}\M{H}^\Tra\|_F,
\end{equation}
where $\M{U}_{X} \in \mathbb{R}^{m\times l}$, $\M{V}_{X} \in \mathbb{R}^{l \times n}$ with $k \leq l \ll \min(m,n)$ and $\M{U}_{X}\M{V}_{X} \approx \M{X}$ is a low-rank approximation of $\M{X}$.
The primary idea is that an approximate solution to \Cref{eq:alg_lai_nmf} can be quickly computed by exploiting the product form of $\M{U}_{X}\M{V}_{X}$ to compute matrix vector products.
That is, $\M{U}_{X}(\M{V}_{X}\V{v})$ is cheaper to compute and approximates the product $\M{X}\V{v}$ (for an arbitrary vector $\V{v}$).
Computing matrix products with the data matrix $\M{X}$ is the main computational bottleneck for many NMF and SymNMF algorithms.
This idea has been explored before in \cite{Cichocki_LRANMF} where the authors used low-rank approximations such as the truncated SVD, and in \cite{rnd_HALS} where the QB-decomposition was used.

\subsection{SymNMF with Low-rank Approximate Input}
We now present \ref{alg:lai_symnmf}, an instantiation of the low-rank approximate input method.
Since the data matrix $\M{X}$ is symmetric we also require that our low-rank approximation be symmetric and not compress only one side of the matrix $\M{X}$ (which would destroy symmetry) as in \cite{rnd_HALS}.
This is accomplished by using the approximate EVD of a symmetric matrix, using \cref{alg:apx_evd}, which gives an approximate truncated EVD of $\M{X}$.
The formulation for \ref{alg:lai_symnmf} is
\begin{equation}
    \label{eq:alg_lai_symnmf}
    \min_{\M{H}\ge0}\|\M{U}_X\M{\Lambda}_X\M{U}_X^\Tra -  \M{H}\M{H}^\Tra\|_F,
\end{equation}
where $\M{U}_X\M{\Lambda}_X\M{U}_X^\Tra \approx \M{X}$ is an approximate truncated EVD of $\M{X}$.
%This allows for the data matrix $\M{U}_X\M{\Lambda}_X\M{U}_X^\Tra$ to be rapidly applied to vectors or matrices by exploiting the low-rank structure $\M{U}_X\M{\Lambda}_X(\M{U}_X^\Tra\M{Y})$ for some arbitrary matrix $\M{Y}$.

\begin{varalgorithm}{LAI-SymNMF}
 \caption{: SymNMF of a Low-Rank Approximated $\M{X}$} 
 \label{alg:lai_symnmf} 
\begin{algorithmic}[1]
\Require{: a symmetric matrix $\M{X} \in\mathbb{R}^{m \times m}$, target rank $k$, oversampling parameter $\rho$, exponent $q$, and regularization parameter $\alpha$}
\Ensure{$\M{H}$, the factor for an approximate rank-$k$ SymNMF of $\M{X}$.}
\Function{$[\M{H}] =$ LAI-SymNMF}{$\M{X},k,\rho,q$}
 \State{$l := k + \rho$}
 \State{[$\M{U}, \M{\Lambda}] := $ Apx-EVD($\M{X},k,\rho,q$)}, \Comment{Obtain approximate, truncated EVD of $\M{X}$, $\M{U}\in\mathbb{R}^{m \times l}$ and $\M{\Lambda}\in\mathbb{R}^{l \times l}$}
 \State{$\M{V} := \M{U}\M{\Lambda}$}
\State{Initialize $\M{H}\in \mathbb{R}^{n \times k}$}
\While{Convergence Crit. Not Met}
\State{$\M{Y}_H := (\M{H}^\Tra\M{V})\M{U}^\Tra + \alpha\M{H}^\Tra$} \Comment{This replaces $\M{H}^\Tra\M{X}$}
 \State{$\M{G}_H := \M{H}^\Tra\M{H} + \alpha \M{I}$}
    \State{$\M{W} := \text{Update}\Big(\M{G_H},\M{Y}_H\Big)$} \Comment{See \Cref{app:updateF} for description of Update()}
\State{$\M{Y}_W := (\M{W}^\Tra\M{U})\M{V}^\Tra + \alpha \M{W}^\Tra$} \Comment{This replaces $\M{W}^\Tra\M{X}$}
 \State{$\M{G}_W := \M{W}^\Tra\M{W} + \alpha \M{I}$}
    \State{$\M{H} := \text{Update}\Big(\M{G_W},\M{Y}_W\Big)$}
 \EndWhile
 \EndFunction
\end{algorithmic}
\end{varalgorithm}

%As an additional step, we suggest post processing the output, $\M{U}_X$ and $\M{\Lambda}_X$, of \ref{alg:apx_evd} by enforcing that it be positive-semidefinite (PSD).
%This is easily done by removing any diagonal entries of $\M{\Lambda}_X$ that satisfy $(\M{\Lambda}_X)_{ii} \le 0$ for $i = 1:l$ and the corresponding columns of $\M{U}_X$.
%Intuitively this makes sense as the objective of SymNMF seeks to approximate $\M{X}$ as $\M{H}\M{H}^\Tra$ which is always positive semidefinite even if $\M{X}$ is not.

\ref{alg:lai_symnmf} is flexible.
Overall if the method for computing SymNMF requires computing the products $\M{X}\M{H}$ and $\M{H}^\Tra\M{H}$, and potentially $\M{W}^\Tra\M{X}$ and $\M{W}^\Tra\M{W}$ as in \Cref{eq:surrogate_symnmf_obj}, then \ref{alg:lai_symnmf} can be efficient.
Not all methods for computing SymNMF rely on these products, such as Cyclic Coordinate Descent (CCD) \cite{gillis_symNMF}.
As previously mentioned, in \cite{DaKuang_SymNMF} it was shown that SymNMF via ANLS applied to \Cref{eq:surrogate_symnmf_obj} is superior to PGD and a Newton-like method.
Similarly, it was shown in \cite{srinivas_SC20_GN} that SymNMF via BPP was competitive with a Projected Gauss-Newton based method for SymNMF and CCD.
CCD is relatively inefficient for large problems as it iterates sequentially over every element of $\M{H}$.
Therefore we use methods based on \Cref{eq:surrogate_symnmf_obj} with BPP and HALS, and the PGNCG's method as our base line for comparison.

To simplify pseudocode and emphasize the flexibility of the method we introduce the $\text{Update}(\M{G},\M{Y})$ function which takes in the Gram matrix $\M{G} \in \mathbb{R}^{k \times k}$ and the product between $\M{X}$ and either $\M{W}$ or $\M{H}$ denoted as $\M{Y} \in \mathbb{R}^{k \times m}$, and performs an update using methods such as HALS or BPP.
For a more in-depth discussion of $\text{Update}()$ see \cref{app:updateF}.
We note that the Update() function abstraction is useful for Alternating Updating (AU) based methods.
However, one of the advantages of our LAI method is that it is applicable to more algorithms.
Existing randomized methods can be effectively used for the NMF update rules such as BPP, MU, or HALS but cannot or have not been used for all-at-once methods such as PGNCG.
The algorithm outline showing how LAI can be used in conjunction with the PGNCG method is shown in \cref{alg:lai_pgncg_symnmf} in \cref{sec:pgncg_code}.

\paragraph{Computational Complexity}
The major part of computational complexity of \cref{alg:lai_symnmf} is due to the \ref{alg:QB_decomposition} and iteratively updating the factors.
Again, the cost of the \ref{alg:QB_decomposition} is $O(q m^2 l)$.
Then computing $\M{V} = \M{X}^\Tra\M{U}$, where $\M{U}$ is the output of the \ref{alg:QB_decomposition}, $\M{V}$ and $\M{U}$ are $m \times l$ and $\M{X}=\M{X}^\Tra$ is $m\times m$, costs $O(m^2l)$.
Additionally, each iteration requires forming two Gram matrices costing $O(mk^2)$ each and applying the LAI to the factor matrices costing $O(mkl)$.
If the algorithm runs for $t$ iterations then the overall cost is $O(q m^2l + tmkl)$ and $l \ge k$.
So if $tk \ll q m$ then we expect that computing the low-rank approximate input via the \ref{alg:QB_decomposition} will dominate the run time.
Naturally, the choice of update function will determine the update cost.

%%%%%%%%%%%%%%%%%%%%%%%%%%%%%%%%%%%%%%%%%%%%%

\subsection{Approximation Errors for LAI-NMF}
\label{sec:apx_errors_4LAINMF}
The authors of \cite{Cichocki_LRANMF} presented a simple error bound applicable to LAI-NMF that we can use to reason about LAI-NMF's performance.
Proposition 1 from \cite{Cichocki_LRANMF} states the following:
\begin{proposition}
\label{lem:lra_nmf_bnd}
Given a matrix $\M{X} \in \mathbb{R}^{m\times n}$ and a low rank approximation $\M{X} \approx \M{U}_X\M{V}_X$, where $\M{U}_X \in \mathbb{R}^{m\times l}$ and $\M{V}_X \in \mathbb{R}^{l \times n}$, with error $\mu = \|\M{X} - \M{U}_{X}\M{V}_{X}\|_F$, define $\{\M{W}^*,\M{H}^*\}$ as the minimizers of \Cref{eq:alg_lai_nmf} and let $\upsilon^* = \min_{ \{\M{W},\M{H} \} \ge0}\|\M{X} -  \M{W}\M{H}^\Tra\|_F$ with low-rank parameter $k \leq l$, then
\begin{equation}
    \upsilon^* \le \|\M{X} - \M{W}^*(\M{H}^*)^\Tra\|_F \le 2\mu + \upsilon^*
\end{equation}
\end{proposition}

\begin{proof}
Define $\{ \M{W}_+,\M{H}_+\} = \argmin_{ \{\M{W},\M{H}\} \ge 0} \|\M{X}-\M{W}\M{H}^\Tra\|$
\begin{align*}
    \| \M{X} - \M{W}^*(\M{H}^*)^\Tra\|_F 
    = \| \M{X} - \M{U}_X\M{V}_X + \M{U}_X\M{V}_X - \M{W}^*(\M{H}^*)^\Tra\|_F \\
    \leq \| \M{X} - \M{U}_X\M{V}_X \|_F 
    + \| \M{U}_X\M{V}_X - \M{W}^*(\M{H}^*)^\Tra\|_F \leq \mu + \| \M{U}_X\M{V}_X - \M{W}_+(\M{H}_+)^\Tra\|_F \\ 
    \leq \mu + \| \M{U}_X\M{V}_X - \M{X}\|_F + \| \M{X} - \M{W}_+(\M{H}_+)^\Tra\|_F \\
    = 2\mu + \upsilon^*
\end{align*}
\end{proof}
\cref{lem:lra_nmf_bnd} allows us to reason about the achievable quality of approximation of LAI-NMF.
Choosing a larger $l$ can help decrease $\mu$ but will also result in higher computational complexity.
A natural choice for computing $\M{U}_{X}$ and $\M{V}_{X}$ is the truncated SVD (which would minimize $\mu$) or, as we use, an approximate truncated SVD or EVD computed using the \ref{alg:QB_decomposition}.
As an alternative, the intermediate inequality, from the proof of \cref{lem:lra_nmf_bnd},
\begin{equation}
\label{eq:lai_2nd_ineq}
    \| \M{X} - \M{W}^*(\M{H}^*)^\Tra\|_F \leq \mu + \| \M{U}_X\M{V}_X - \M{W}_+(\M{H}_+)^\Tra\|_F 
\end{equation}
provides an intuitive way to reason about LAI-NMF.
This inequality says that LAI-NMF residual depends on $\mu$, which measures the quality of the low-rank input, and the term $\| \M{U}_X\M{V}_X - \M{W}_+(\M{H}_+)^\Tra\|_F $.
This second term can be thought of as quantifying how much of the optimal NMF solution is captured in the low-rank input.

\Cref{lem:lra_nmf_bnd} can give an error bound for \cref{alg:lai_symnmf}.
For this, one needs an error bound for the QB-decomposition or the approximate EVD.
That is, given a decomposition from the \ref{alg:QB_decomposition} as $\M{Q}_X\M{B}_X$, we seek a value $\| \M{Q}_X\M{B}_X - \M{X}\|_F \leq \mu_{RRF}$.
\Cref{thm:gu_bnd}, from Gu \cite{Gu_rndsubspace}, provides such a bound.
\Cref{thm:gu_bnd} is a partial statement of Theorem 5.8 from \cite{Gu_rndsubspace} that we include for reference and completeness.
\begin{theorem}
    \label{thm:gu_bnd}
    Let $\M{Q}_X\M{B}_X$ be a low-rank approximation of $\M{X} \in \mathbb{R}^{m \times n}$, with $n \leq m$, obtained from the \ref{alg:QB_decomposition} with desired low-rank $r$, power iteration parameter $q$, column over sampling parameter $\rho=l-r$ and a parameter $0 < \delta << 1$. 
    Define
    \begin{equation*}
        C_\delta = \frac{e\sqrt{l}}{\rho+1} \left(\frac{2}{\delta}\right)^{\frac{1}{\rho+1}} \Big(\sqrt{n-l+\rho} + \sqrt{l} + \sqrt{2 \log(2/\delta)} \Big).
    \end{equation*}
    Then with probability $1-\delta$ the following holds:
    \begin{equation*}
        \|\M{Q}_X\M{B}_X - \M{X}\|_F
        \leq
        \sqrt{\Big(\sum_{j=r+1}^n \sigma_j^2(\M{X}) \Big) + r C_\delta^2 \sigma^2_{r+1}(\M{X}) \Big(\frac{\sigma_{r+1}(\M{X})}{\sigma_r(\M{X})}\Big)^{4q}}.
    \end{equation*}
\end{theorem}
This Theorem says that $\| \M{Q}_X\M{B}_X - \M{X}\|_F \leq \mu_{RRF}$ holds with some chosen probability $1-\delta$, where $\mu_{RRF}$ depends on $\delta$ and other parameters of the $\ref{alg:QB_decomposition}$ such as $l$ and $q$.

\Cref{prop:rndlia_nmf_bnd} explicitly combines \cref{lem:lra_nmf_bnd} and \cref{thm:gu_bnd} to give a probabilistic error bound for LAI-NMF with a randomized low-rank input from the \ref{alg:QB_decomposition}.
\begin{proposition}
\label{prop:rndlia_nmf_bnd}
Given a matrix $\M{X} \in \mathbb{R}^{m\times n}$ compute a low-rank approximation $\M{X} \approx \M{Q}_X\M{B}_X$ where $\M{Q}_X \in \mathbb{R}^{m\times l}$ and $\M{B}_X \in \mathbb{R}^{l \times n}$ from the \ref{alg:QB_decomposition}.
%Let $\sigma_i$ be the $i$th largest singular value of $\M{X}$.
Then by \cref{thm:gu_bnd}, for any $0 < \delta \ll 1$, we have that
\begin{equation*}
    \|\M{X} - \M{Q}_X\M{B}_X\|_F 
    \leq 
    \sqrt{\Big(\sum_{j=k+1}^n \sigma_j^2(\M{X}) \Big) + k C_\delta^2 \sigma^2_{k+1}(\M{X}) \Big(\frac{\sigma_{k+1}(\M{X})}{\sigma_k(\M{X})}\Big)^{4q}}
    = \mu_{RRF} 
\end{equation*}
holds with probability $1-\delta$.
Define $\{\M{W}^*,\M{H}^*\} = \argmin_{\M{W},\M{H} \ge 0}\|\M{Q}_X\M{B}_X - \M{W}\M{H}^\Tra\|_F$ and the optimal NMF error $\upsilon^* = \min_{ \{\M{W},\M{H} \} \ge0}\|\M{X} -  \M{W}\M{H}^\Tra\|_F$.
Then with probability $1-\delta$
\begin{equation}
    \upsilon^* \le \|\M{X} - \M{W}^*(\M{H}^*)^\Tra\|_F \le 2\mu_{RRF} + \upsilon^*
\end{equation}
as in \cref{lem:lra_nmf_bnd}.
\end{proposition}

In the case that $\M{X}$ is symmetric and the \ref{alg:QB_decomposition} is used to compute an approximate EVD it is simple to extend \cref{prop:rndlia_nmf_bnd}.
Given $\mu_{RRF}$, one can obtain an error bound for the approximate, truncated EVD produced by \cref{alg:apx_evd}.
We use a fact from \cite{halko_survey}.
Given a low-rank approximation $\M{X} \approx \M{Q}_X\M{B}_X$, where $\M{X}$ is symmetric, from \cref{alg:QB_decomposition} and defining $\M{P}_X = \M{Q}_X\M{Q}_X^\Tra$, observe that 
\begin{align*}
    \| \M{X} -\M{P}_X\M{X}\M{P}_X \|_F
    = \| \M{X} - \M{P}_X\M{X} + \M{P}_X\M{X} - \M{P}_X\M{X}\M{P}_X \|_F \\
    \leq \| \M{X} - \M{P}_X\M{X}\|_F + \|\M{P}_X\M{X} - \M{P}_X\M{X}\M{P}_X \|_F
    = \| \M{X} - \M{Q}_X\M{B}_X\|_F + \|\M{P}_X (\M{X} - \M{X}\M{P}_X) \|_F \\
    \leq \mu_{RRF} + \| \M{X} - \M{P}_X \M{X} \|_F \leq 2\mu_{RRF},
\end{align*}
where the last equality uses the symmetry of $\M{X}$ and $\M{P}_X$.
Therefore the low-rank approximation produced by \cref{alg:apx_evd} achieves a residual of no more than $2\mu_{RRF}$ if the \ref{alg:QB_decomposition} achieves an residual of no more than $\mu_{RRF}$ (with high probability).

%%%%%%%%%%%%%%%%%%%%%%%%%%%%%%%%%%%%%%%%%%%%%

\subsection{Practical Considerations for LAI-SymNMF}
\label{sec:refine_symnmf}
The quality of the SymNMF approximate solution found by \ref{alg:lai_symnmf} is dependant on the LAI.
In the proposed \cref{alg:lai_symnmf}, the LAI is a truncated EVD.
We propose two methods for ensuring that a high-quality factorization is produced by \ref{alg:lai_symnmf}.
Each one deals with a separate component of \Cref{eq:lai_2nd_ineq}.
The first is to post-process the output from \ref{alg:lai_symnmf} by running a few iterations of the full NMF method.
The second is to test and improve the quality of the approximate truncated EVD before starting the NMF iterations.
We now discuss these two methods in more detail.

\paragraph{Iterative Refinement}
Iterative Refinement (IR) runs some number of NMF iterations using the full matrix $\M{X}$ instead of the LAI.
That is, after the iterations of \ref{alg:lai_symnmf} are finished, the algorithm switches over to using the full input matrix $\M{X}$, therefore capturing information possibly lost in the low-rank approximation of $\M{X}$.
This helps in cases where the right side of \Cref{eq:lai_2nd_ineq} is large.
In practice our experimental results show that this method is effective in improving the SymNMF approximations attained by \ref{alg:lai_symnmf} while running faster than standard SymNMF methods.

\paragraph{Adaptive RRF}
The \ref{alg:QB_decomposition} has two main hyperpameters 1) column oversampling parameter $\rho$ and 
2) the power iteration parameter $q$.
There exists work on adaptive methods for selecting $\rho$ \cite{halko_survey}.
For our algorithms, where $k$ is usually considered a static input to NMF methods, we find that choosing $\rho$ in the range of $2k$ to $3k$ is satisfactory. 
Empirically we find that determining a good $q$ is more difficult.
Prior works recommend a choice of $q=2$ \cite{rnd_HALS,halko_survey}.
However we find that this choice can be inadequate and negatively impact performance. 
To remedy this we propose an Adaptive RRF algorithm that automatically chooses $q$.
This method checks the residual of the QB-Decomposition after each power iteration and stops once a certain stopping criteria is met (e.g. lack of reduction in residual), similar to NMF.
The residual check is cheap.
Checking the residual of the QB-Decomposition after each power iteration requires only one extra multiplication against $\M{X}$ when calling the \ref{alg:QB_decomposition} by use of a standard `trick' for computing the residual.
That is if $q$ power iterations are performed we only apply $\M{X}$, $q+1$ times.
The algorithm outline can be found in the Appendix in \cref{alg:adaRRF}.
This approach ensures we achieve a good value of $\mu_{RRF}$ as in \Cref{eq:lai_2nd_ineq}.

%%%%%%%%%%%%%%%%%%%%%%%%%%%%%%%%%%%%%%%%%%%%%

\subsection{Discussion of LAI-SymNMF}
Compared to existing randomized methods, such as those in \cite{rnd_HALS,Comp_NMF2016}, the LAI method is more general in that it can work for any NMF method that relies on matrix vector products $\M{X}\V{v}$, where $\V{v}$ is an arbitrary vector, for performance.
For example, the Compressed-NMF method from \cite{Comp_NMF2016}, which we compare against in \Cref{sec:experiments}, is only applicable for Alternating Updating methods.
The PGNCG method for SymNMF from \cite{srinivas_SC20_GN} can be used for \ref{alg:lai_symnmf} but not for Compressed-NMF.
Finally, unlike the randomized method in \cite{rnd_HALS}, the LAI method decouples the randomization from the NMF algorithm and accordingly preserves the convergence properties, such as convergence to a stationary point \cite{BCD_park}, of existing NMF algorithms applied to the low-rank input and can be reasoned about via \cref{prop:rndlia_nmf_bnd}.
A more detailed discussion of existing randomized NMF methods and LAI-NMF is given in \cref{app:lai_vs_comp}.

\section{SymNMF via Leverage Score Sampling}
\label{sec:LVS_symNMF}
This section presents the algorithm for randomized SymNMF based on using leverage scores sampling to sketch the NLS problems in \Cref{eq:ANLS4symNMF}.
In the context of low-rank approximations leverage score sampling has been successfully used for computing CP decompositions \cite{lev_scores4_CP,SPALS_neurips}, especially of large sparse tensors.
We propose this method as suitable for large, sparse data sets such as graph data.
Though we focus on SymNMF we expect that this method would work well for standard NMF as well.
Leverage score sampling preserves not only sparsity but nonnegativity as well. 

\subsection{Leverage Score Sampling for Multiple Right Hand Sides}
Unlike in the CP decomposition, the coefficient matrix in the LS problem for low-rank matrix approximation has, in general, no special structure we can exploit to obtain fast leverage score estimates.
%This means that we will need to obtain speed up in a different way.
%The answer to this problem comes again from the observation
However, for many methods the products $\M{W}^\Tra\M{X}$ and $\M{X}\M{H}$ are the most expensive part of an NMF iteration \cite{mpi_faun_ramki}.
By computing a thin $\M{Q}\M{R}$ factorization of the matrices $\M{W}$ and $\M{H}$ at each iteration we can obtain exact leverage scores for use in sampling and avoid the expensive full matrix products involving the data matrix $\M{X}$.
The algorithm outline is given in \cref{alg:lvs_nmf}.
%Most Leverage score sampling research has focused on the single right hand side part of the problem.

To formalize this idea, consider the NLS problem for updating $\M{H}$:
\begin{equation}
\label{eq:nls_update_H}
    \min_{\M{H} \ge 0}\|\M{W}\M{H}^\Tra - \M{X}\|_F,
\end{equation}
where the coefficient matrix $\M{W} \in \mathbb{R}^{m\times k}$ is much smaller than the right hand side matrix $\M{X}\in \mathbb{R}^{m\times m}$ if $k \ll n$.
Consider (approximately) solving the problem by the
%following procedure compute $\M{W}^\Tra\M{W}$ and $\M{W}^\Tra\M{X}$, then performing an
update as $\M{H} := \text{Update}(\M{W}^\Tra\M{W}, \M{W}^\Tra\M{X})$.
Recall Update() was introduced for \cref{alg:lai_symnmf} and its details can be found in \cref{app:updateF}.
Computing $\M{W}^\Tra\M{W}$ and $\M{W}^\Tra\M{X}$ costs $O(mk^2)$ and $O(m^2k)$ flops.
The cost of the Update() will be denoted as $O(T(m,k))$ and is dependent on the method used.
In light of this we suggest the following randomized approach:
\begin{enumerate}
    \item Compute a thin QR-decomposition of $\M{W} = \M{Q}_W\M{R}_W$ for $O(mk^2)$ flops.
    \item Compute the leverage scores exactly using $\M{Q}_W$ and generate the sampling matrix $\M{S}_W \in \mathbb{R}^{s \times m}$ as in \Cref{eq:lev_score_S} drawing $s$ samples.
    \item Perform an Update() for the reduced problem $\min_{\M{H}\ge 0}\|\M{S}_W\M{W}\M{H}^\Tra - \M{S}_W\M{X}\|_F^2$ as $\M{H} := \text{Update}(\M{W}^\Tra \M{S}_W^\Tra \M{S}_W\M{W}, \M{W}^\Tra\M{S}_W^\Tra \M{S}_W\M{X})$. \\ (As opposed to $\M{H} := \text{Update}(\M{W}^\Tra\M{W}, \M{W}^\Tra\M{X})$.)
\end{enumerate}
The conditions for this scheme to provide speed up are roughly that $s\ll m$ and that the cost $T(m,k)$ does not dominate the overall complexity.
The key observation here is that computing the thin QR-decomposition costs only $O(mk^2)$ flops and so when a large number of right hand side (RHS) vectors is present, computing the leverage scores is not the dominating cost.
Note that this observation is relevant for problems with a similar structure to \Cref{eq:nls_update_H}. 
For example one can approximately solve an OLS problem with many RHS vectors and small coefficient matrix using this scheme.
%Extrapolating this observation for OLS problems with multiple RHS to NLS problems with multiple RHS, by sampling the NLS problems in the ANLS for SymNMF method we can approximately solve each NLS problem and obtain speed up by avoiding the expensive products $\M{W}^\Tra\M{X}$ and $\M{X}\M{H}$.
%To our knowledge this algorithm has not be suggested previously, likely because it is only applicable to specific values of $m,n,$ and $k$.

Since the NLS problems given by \Cref{eq:ANLS4symNMF} are regularized, we propose the scheme given below for leverage score sampling:
\begin{align*}
    \Big\|
    \begin{bmatrix}
        \M{S} & 0 \\
        0 & \M{I}_k
    \end{bmatrix} 
    \Big(
    \begin{bmatrix}
    \M{H}\\
    \sqrt{\alpha}\M{I}_k
    \end{bmatrix}
    \M{W}^\Tra - 
     \begin{bmatrix}
    \M{X}\\
    \sqrt{\alpha}\M{H}^\Tra
    \end{bmatrix} \Big) \Big\|_F^2
    = 
    \Big\|
    \begin{bmatrix}
    \M{S}\M{H}\\
    \sqrt{\alpha}\M{I}_k
    \end{bmatrix}
    \M{W}^\Tra - 
     \begin{bmatrix}
    \M{S}\M{X}\\
    \sqrt{\alpha}\M{H}^\Tra
    \end{bmatrix} \Big\|_F^2 \\
    = \Big\|\M{S}\M{H}\M{W}^\Tra - \M{S}\M{X}
     \|_F^2 + \alpha \| \M{W} - \M{H} \Big\|_F^2,
\end{align*}
where the leverage score sampling matrix denoted by $\M{S}$ samples only rows of $\M{H}$ and the regularization portion is deterministically included.
A similar technique is used for the sampling of $\M{W}$ when $\M{H}$ is being updated.

\subsubsection{Complexity}
At a high level the main computational kernels of SymNMF via regularization include $\M{X}\M{H}$, $\M{W}^\Tra\M{X}$, $\M{W}^\Tra\M{W}$, and $\M{H}^\Tra\M{H}$, which cost $O(m^2k)$ flops for the products with $\M{X}$ and $O(mk^2)$ for the Gramians.
Once this is done these matrix products are used to perform an update, e.g. via (approximately) solving the NLS problem.

Leverage score sampling replaces these products with $\M{X}\M{S}_H^\Tra\M{S}_H\M{H}$, $\M{W}^\Tra\M{S}_W^\Tra\M{S}_W\M{X}$, $\M{W}^\Tra\M{S}_W^\Tra\M{S}_W\M{W}$, and $\M{H}^\Tra\M{S}_H^\Tra\M{S}_H\M{H}$,
which cost $O(msk)$ and $O(sk^2)$.
The number of samples $s$ will be discussed in more detail later in \Cref{sec:theory_nls}.
Additionally, computing the thin QR-Decomposition to obtain the leverage scores costs $O(mk^2)$.
The discrepancy between asymptotic flop costs of the deterministic method and the leverage score based method comes primarily from the difference between $m$ and $s$.

As previously stated, sampling does not generally affect the cost of the Update Rule which costs $O(T(m,k))$.
The two update rules we use are the HALS and BPP methods.
For a discussion of these rules and their properties see \cite{BCD_park,gillis_book}.
%BPP is active set-like method and technically requires $O(k!(k^3 + mk))$ flops to perform an update in the worst case, though it takes much less in practice due to caching matrix factorizations and good initialization.
%The HALS method takes $O(mk^2)$ flops per update.
%It is important to note that per iteration flops costs provided a limited view of the efficiency of these methods as they have different convergence properties.

\begin{varalgorithm}{LvS-SymNMF}
 \caption{Randomized SymNMF via Leverage Score Sampling} 
 \label{alg:lvs_nmf} 
\begin{algorithmic}[1]
\Require{Symmetric data matrix $\M{X} \in\mathbb{R}^{m \times m}$, target rank $k$, threshold $\tau$}
 \Ensure{$\{\M{H}\}$ as the factors for an approximate rank-$k$ SymNMF of $\M{X}$.}
 \Function{$[\M{H}] = $ LvS-SymNMF}{$\M{X},k,s,\tau$}
\State{Randomly initialize $\M{H}\in \mathbb{R}^{m \times k}$}
 \While{Convergence Crit. Not Met}
 \State{$\M{R}_H :=$ chol$(\M{H}^\Tra\M{H})$} \Comment{Compute upper triangular Cholesky factor $\M{R}$}
 \State{Solve $\M{H}\M{R}_H^{-1} = \M{Q}_H$} \Comment{Triangular solve for $\M{Q}_H$}
 \State{Compute $\V{p}_H[i] = \|\M{Q}_H[i,:]\|_2^2$ for $i=1:m$} 
 \Comment{Compute leverage scores of $\M{H}$}
 \State{Construct $\M{S}_H$ using $\V{p}_H$ according to Eqn. (\ref{eq:lev_score_S})}
\State{$\M{Y}_H := \M{H}^\Tra\M{S}_H^\Tra\M{S}_H\M{X}^\Tra + \alpha \M{H}^\Tra$}
 \State{$\M{G}_H := \M{H}^\Tra\M{S}_H^\Tra \M{S}_H \M{H} + \alpha\M{I}$}
    \State{$\M{W} := \text{Update}\Big(\M{G_H},\M{Y}_H\Big)$} \Comment{See \Cref{app:updateF} for description of Update()}
\State{$\M{R}_W := $ chol$(\M{W}^\Tra\M{W})$} \Comment{Compute upper triangular Cholesky factor $\M{R}_W$}
\State{Solve $\M{W}\M{R}_W^{-1} = \M{Q}_W$} \Comment{Triangular solve for $\M{Q}_W$}
\State{Compute $\V{p}_W[i] = \|\M{Q}_W[i,:]\|_2^2$ for $i=1:m$} 
\State{Construct $\M{S}_W$ using $\V{p}_W$ according to Eqn. (\ref{eq:lev_score_S})}
\State{$\M{Y}_W := \M{W}^\Tra\M{S}_W^\Tra\M{S}_W\M{X} + \alpha \M{W}^\Tra$}
 \State{$\M{G}_W := \M{W}^\Tra\M{S}_W^\Tra\M{S}_W\M{W} + \alpha \M{I}$}
    \State{$\M{H} := \text{Update}\Big(\M{G_W},\M{Y}_W\Big)$}
\EndWhile
\EndFunction
\end{algorithmic}
\end{varalgorithm}

%%%%%%%%%%%%%%%%%%%%%%%%%%%%%%%%%%%%%
\subsection{Practical Considerations for SymNMF via Leverage Score Sampling}
This section describes the implementation of \cref{alg:lvs_nmf}.
We discuss how the leverage scores are computed and the use of a hybrid leverage score sampling scheme which was introduced in \cite{lev_scores4_CP}. 
A theoretical analysis of Hybrid Sampling is given later in \Cref{sec:hybrid_samp_theory}.

%\subsubsection{Computing the thin $\M{Q}\M{R}$ factorization}
For computing the thin QR-decomposition of a full rank matrix $\M{F} \in \mathbb{R}^{m \times k}$ where $\M{F} = \M{Q}\M{R}$, $\M{Q} \in \mathbb{R}^{m \times k}$ and $\M{R}\in \mathbb{R}^{k \times k}$, we use the CholeskyQR algorithm.
CholeskyQR computes $\M{F}^\Tra\M{F}$, then computes the Cholesky Decomposition $\M{F}^\Tra\M{F} = \M{R}^\Tra\M{R}$ where $\M{R}$ is $k \times k$ and upper triangular, and lastly solves the triangular linear system $\M{F} = \M{Q}\M{R}$ to obtain $\M{Q}$.
%For our experimental set up, in MATLAB computing a thin $\M{Q}\M{R}$ decomposition via Cholesky-QR of a matrix $\M{A}$ which is approximately $24$million $\times$ $25$ is $\approx 2.7\times$ faster than calling $\text{qr}(\M{A},0)$.
CholeskyQR is numerically less stable than Householder QR but faster and empirically we find that it works well for computing leverage scores.

%\subsubsection{Hybrid Leverage Score Sampling}
%\label{sec:hybrid_samp_theory}
Hybrid leverage score sampling samples a subset of rows deterministically and then randomly samples from the remaining rows.
We find that Hybrid Sampling is crucial for good performance in our empirical results and we offer a rigorous analysis of its theoretical performance in \Cref{sec:hybrid_samp_theory}.
Hybrid Sampling was proposed and shown to be effective for computing CP decompositions of sparse tensors in \cite{lev_scores4_CP}.
A similar method has been used and theoretically analyzed before for the column subset selection problem \cite{deterministic_lvs_sampling}.

In Hybrid Sampling, a threshold $\tau \in [0,1]$ is used as a hyperparameter.
When sampling according to the leverage score distribution, all rows that satisfy $p_i \ge \tau$ are deterministically selected.
Let the full set of row indices be $\mathscr{I}$, the set that is deterministically included be $\mathscr{I}_D$ and $s_D = |\mathscr{I}_D|$, and the rest be $\mathscr{I}_R = \mathscr{I} \setminus \mathscr{I}_D$ with $s_R = |\mathscr{I}_R|$, the remaining indices from which random samples are drawn.

Let $\M{Q}_A$ be a $m \times k$ matrix with othonormal columns that is being sampled.
Assume without loss of generality that the rows of $\M{Q}_A$ permuted conformally to the sets $\mathscr{I}_D$ and $\mathscr{I}_R$ then the hybrid sampling matrix takes the form
\begin{equation} \label{eq:SH}
    \M{S}_{H}  
    = \begin{bmatrix}
    \M{S}_D & \M{0}_{s_D \times (m-s_D)} \\
    \M{0}_{(s_R) \times s_D} & \M{S}_R
    \end{bmatrix}
    \in \mathbb{R}^{s_{H} \times m}
\end{equation}
where $\M{S}_R \in \mathbb{R}^{s_R \times (m-s_D)}$ and $\M{S}_D \in \mathbb{R}^{s_D \times s_D}$ is a permutation matrix for the deterministically included portion defined as
\begin{equation}
    (\M{S}_D)_{ji} = 
    \begin{cases}
        1, \ \text{if row $i$ is the $j$th deterministic sample} \\
        0, \ \text{otherwise}.
    \end{cases}
\end{equation}
%Note this means that $\M{S}_D$ is a permutation matrix.
$\M{S}_D$ is often included in notation to make the deterministic inclusion aspect of the equations explicit.
The submatrix $\M{S}_{R}$ is a leverage score sampling matrix as defined in \Cref{eq:lev_score_S} but just of the indices in $\mathscr{I}_R$.
When sampling for $\M{S}_R$, rows that were sampled during the deterministic phase are not considered and the leverage scores are renormalized appropriately.
The new leverage scores probabilities are $\tilde{p}_i = \frac{l_i}{k - \theta}$ where $\theta = \sum_{i\in\mathscr{I}_D} l_i(\M{A})$.

%%%%%%%%%%%%%%%%%%%%%%%%%%%%%%%%%%%%%%%%%%%%%%

\subsection{Analysis of LvS-SymNMF}
\label{sec:theory_lvsSymNMF}
We now explore some theoretical questions relevant to \cref{alg:lvs_nmf}.
Specifically we seek to answer two questions: 1) Can results for leverage score sampling for OLS problems be extended to NLS problems? and 2) What is the sample complexity of the Hybrid Sampling method from \cite{lev_scores4_CP}?

%%%%%%%%%%%%%%%%%%%% START CLS %%%%%%%%%%%%%%%%%%%%%%
\subsubsection{Leverage Score Sampling for NLS Problems}
\label{sec:theory_nls}
Error bounds and corresponding sampling complexities for sketching the ordinary least squares (OLS) problem have been previously shown in a number of works \cite{sarlos_FLSA,sarlos_og}.
Larsen and Kolda gave a bound and proof structure in their work on computing a randomized CP decomposition \cite{lev_scores4_CP}.
Boutsidis and Drineas \cite{drineas_Rnd_NNLS} considered using the randomized Hadamard transform for solving the NLS problem.
Our proof structure generally follows that in \cite{sarlos_FLSA} for OLS problems.
%More recently Woodruff \cite{sketching4NLA_woodruff} claims that \textit{sparse subspace embeddings} can be used for sketching and solving constrained regression problems, such as NLS problems.
%However the proof is not given and our proof setup is different than the one presented for the OLS case, which is based on the original work by Sarlos \cite{sarlos_og}.

This section provides a proof of \Cref{thm:nls_bnd} which provides generalized error bounds and sampling complexities for the Nonnegative Least Squares problem (NLS).
We note that our results hold for all constrained least squares (CLS) problems as far as the problem remains convex.
The NLS problem in \Cref{eq:nls} is a convex optimization problem but unlike the case of OLS, does not yield a closed form solution.
We are concerned only with the case where $\M{A}$ is full rank and overdetermined. 
The NLS error bound and sample complexity can be derived based on results by Daniels for the perturbation of Convex Quadratic Programs \cite{qp_pert1} and leverage score sampling for the OLS problem \cite{sarlos_FLSA,Mahoney16,larsen2022sketching,sketching4NLA_woodruff}.

%\subsubsection{Proof Outline}
%Our proof outline follows that for the OLS problem as presented by Larsen and Kolda \cite{lev_scores4_CP}.
The proof uses two \textit{Structural Conditions} (SC's) such that if both are true then the error bound in \cref{thm:nls_bnd} for the NLS problem holds.
We first discuss the sampling complexity and probability conditions under which these SC's hold.
%Next the conditions, e.g. sampling complexity and probability, are discussed under which these SC's hold.

Let $\M{A} = \M{U}_A\M{\Sigma}_A\M{V}_A^\Tra$ be the thin SVD of the coefficient matrix in \Cref{eq:nls} where $\M{U}_A \in \mathbb{R}^{m \times n}$, $\M{\Sigma}_A \in \mathbb{R}^{n \times n}$, and $\M{V}_A \in \mathbb{R}^{n \times n}$.
Let $\M{S}$ be a leverage score sampling matrix for $\M{A}$.
The first Structural Condition (SC1) is
\begin{equation}
    \label{eq:sc1}
    1-\epsilon_s \le \sigma_{i}(\M{U}_A^\Tra\M{S}^\Tra\M{S}\M{U}_A) \le 1+\epsilon_s
\end{equation}
for all $i \in [k]$, some $\epsilon_s \in (0,1)$, and where $\sigma_{i}(\M{A})$ is the $i$th singular value of $\M{A}$.
%The leverage score sketching matrix $\M{S}$ satisfies SC1 with high probability (given sufficiently many samples) \cite{sketching4NLA_woodruff}.
The second SC (SC2) is
\begin{equation}
    \label{eq:sc2}
    \|\M{U}_A^\Tra\V{r}_{nls} 
    - \M{U}_A^\Tra\M{S}^\Tra \M{S}\V{r}_{nls}\|_2 \le \sqrt{\epsilon_r} \frac{\|\V{r}_{nls}\|_2}{\sqrt{2}}
\end{equation}
for some $\epsilon_r \in (0,1)$.
The leverage score sketching matrix $\M{S}$ satisfies SC1 and SC2 with high probability (given sufficiently many samples).
For SC1 this is shown in \cite{sketching4NLA_woodruff} and for SC2 it can be shown by using \cref{thm:rnd_matmul} on the product $\M{U}_A^\Tra\V{r}_{nls}$ \cite{Mahoney16}.
%For now we need only that these two conditions be true for some matrix $\M{S}$.
We include the associated theorems for these statements in \cref{app:sc_theorems}.

The second result we make use of is a bound on the perturbation of convex Quadratic Programs (QP's).
Convex QP's have the general form
\begin{align}
\label{eq:QP_general}
    \min_{ \V{x} \in \mathscr{C}}Q(\V{x}) 
    = \min_{ \V{x} \in \mathscr{C} }\frac{1}{2} \V{x}^\Tra \M{K} \V{x} - \V{x}^\Tra \V{d}
\end{align}
where $\M{K}$ is a square positive semi-definite matrix, $\V{d}$ is a vector, and $\mathscr{C}$ is a convex set.
Consider the NLS problems in \Cref{eq:nls} and
\begin{equation}
    \label{eq:sampled_nls}
    \min_{\V{x} \ge 0}\|\M{S}\M{A}\V{x} - \M{S}\V{b} \|_2
\end{equation}
and the equivalent QP's, respectively,
\begin{align}
\label{eq:QP_def}
    \min_{ \V{x} \ge 0 }Q(\V{x}) 
    = \min_{ \V{x} \ge 0 }\frac{1}{2} \V{x}^\Tra\M{A}^\Tra\M{A}\V{x} - \V{x}^\Tra\M{A}^\Tra\V{b},
\end{align}
\begin{align}
    \label{eq:QP_pert_def}
    \min_{\V{x} \ge 0}\hat{Q}(\V{x}) 
    = \min_{ \V{x} \ge 0 }\frac{1}{2} \V{x}^\Tra\M{A}^\Tra\M{S}^\Tra\M{S}\M{A}\V{x} - \V{x}^\Tra\M{A}^\Tra\M{S}^\Tra\V{b}.
\end{align}
The objective function give by \ref{eq:QP_pert_def} can be interpreted as a perturbed version of the objective function given by \ref{eq:QP_def} with perturbed parameters $\hat{\M{K}} = \M{A}^\Tra\M{S}^\Tra\M{S}\M{A}$ and $\hat{\V{d}} = \M{A}^\Tra\M{S}^\Tra\V{b}$.
We make use of the following, which is a straightforward consequence of the fact that $\nabla f(\V{x}^{*})^\Tra(\V{x} - \V{x}^{*}) \ge 0$ for all $\V{x} \in \mathscr{C}$, which is the characterization of a minima (at $\V{x}^{*}$) of a differentiable convex function $f(\V{x})$, 
\begin{lemma}[{\cite[Equation 2.4]{qp_pert1}}] 
    \label{lem:daneils_qp_ineq}
    Convex Quadratic Program Inequality:
    $Q(\V{x})$ and $\hat{Q}(\V{x})$ from \Cref{eq:QP_def} and \Cref{eq:QP_pert_def} satisfy
    \begin{equation}
\label{eq:qp_ineq1}
    (\hat{\V{x}}_{nls} - \V{x}_{nls})^\Tra \Big[\nabla \hat{Q}(\hat{\V{x}}_{nls}) - \nabla\hat{Q}(\V{x}_{nls}) \Big]
   \leq
    (\hat{\V{x}}_{nls} - \V{x}_{nls})^\Tra \Big[ \nabla Q(\V{x}_{nls}) - \nabla\hat{Q}(\V{x}_{nls})\Big] 
\end{equation}
where $\nabla Q(\V{x}) = \M{K}\V{x} - \V{d}$ is the gradient of $Q(\V{x})$ at $\V{x}$, $\nabla \hat{Q}(\V{x}) = \M[\hat]{K}\V{x} - \V[\hat]{d}$ is the gradient of $\hat{Q}(\V{x})$ at $\V{x}$, and $\V{x}_{nls}$ and $\hat{\V{x}}_{nls}$ are the minimizers of \cref{eq:QP_def} and \cref{eq:QP_pert_def} respectively.
\end{lemma}
All the tools needed for the proof of \cref{thm:nls_bnd} have now been established. 
%\subsubsection{Proof of \cref{thm:nls_bnd}}
%(Structural Conditions and Perturbation of Convex QPs).
\begin{proof}[Proof of Theorem~\ref{thm:nls_bnd}]
Substituting the QP formulations of the original and sampled NLS problems into \Cref{eq:qp_ineq1}, denoting $\hat{\V{x}}_{nls} - \V{x}_{nls} = \V{y}$ and parameterizing in terms of the matrix $\M{U}_A$ by writing $\M{Z} = \M{\Sigma}_A\M{V}^\Tra_A$ and $\M{Z}\V{y} = \V{z}$, the right hand side becomes 
\begin{align*}
    \V{y}^\Tra \Big[\nabla \hat{Q}(\hat{\V{x}}_{nls}) - \nabla\hat{Q}(\V{x}_{nls}) \Big]
    =
    \V{y}^\Tra\Big[\hat{\M{K}}\hat{\V{x}}_{nls} - \hat{\V{d}} - \hat{\M{K}}\V{x}_{nls} + \hat{\V{d}}\Big]
    = \V{y}^\Tra\hat{\M{K}}\V{y}
    = \\ \V{y}^\Tra\M{Z}^\Tra\M{U}_A^\Tra\M{S}^\Tra\M{S}\M{U}_A\M{Z}\V{y} 
    = \V{z}^\Tra(\M{U}_A^\Tra\M{S}^\Tra\M{S}\M{U}_A)\V{z}
    = \|\M{S}\M{U}_A\V{z}\|_2^2
    \ge  \sigma_{min}(\M{S}\M{U}_A)^{2} \|\V{z}\|_2^2.
\end{align*}
The left hand side of \Cref{eq:qp_ineq1} yields
\begin{align*}
    \V{y}^\Tra \Big[ \nabla Q(\V{x}_{nls}) - \nabla\hat{Q}(\V{x}_{nls})\Big]
    = \V{y}^\Tra \Big[\M{K}\V{x}_{nls} - \V{d} - \hat{\M{K}}\V{x}_{nls} + \hat{\V{d}}\Big] \\
    = \V{z}^\Tra \Big[\M{U}_A^\Tra\M{U}_A\M{Z}\V{x}_{nls} - (\M{U}_A^\Tra\M{S}^\Tra)(\M{S}\M{U}_A\M{Z})\V{x}_{nls} - \M{U}_A^\Tra\V{b}
    + (\M{U}_A^\Tra\M{S}^\Tra\M{S})\V{b}\Big] \\
    = \V{z}^\Tra \Big[ \M{U}_A^\Tra(\M{A}\V{x}_{nls} 
    -\V{b}) 
    - \M{U}_A^\Tra\M{S}^\Tra \M{S} (\M{A}\V{x}_{nls} 
    - \V{b}) \Big] \\
    = \V{z}^\Tra\Big[\M{U}_A^\Tra\V{r}_{nls} 
    - \M{U}_A^\Tra\M{S}^\Tra \M{S}\V{r}_{nls}\Big].
\end{align*}
Recall that $\V{r}_{nls}  = \M{A}\V{x}_{nls} - \V{b}$ is the NLS residual vector.
Combining the previous two equations back into \cref{lem:daneils_qp_ineq} we have 
\begin{align*}
    \V{z}^\Tra \Big[\M{U}_A^\Tra\V{r}_{nls} 
    - \M{U}_A^\Tra\M{S}^\Tra \M{S}\V{r}_{nls}\Big] 
    &\ge \sigma_{min}(\M{S}\M{U}_A)^2\|\V{z}\|_2^2,
    \end{align*}
    which implies
\begin{equation}
\label{eq:intermediate_ineq1}
    \|\V{z}\|_2 \Big\| \M{U}_A^\Tra\V{r}_{nls} 
    - \M{U}_A^\Tra\M{S}^\Tra \M{S}\V{r}_{nls}\Big\|_2 
    \ge \sigma_{min}(\M{S}\M{U}_A)^2\|\M{\Sigma}_A\M{V}_A(\hat{\V{x}}_{nls} - \V{x}_{nls})\|_2^2.  
\end{equation}
We now invoke the SC's with sufficient samples
\begin{equation*}
    s \ge k \max(C \log(k/\delta),1/(\delta \epsilon_r)) \ \ \text{where} \ \ C = 144/(1-\sqrt{2})^2
\end{equation*}
so that both SC1, as in \cref{thm:su_bound} with $\epsilon_s = 1-1/\sqrt{2}$ meaning that $\sigma_{min}(\M{S}\M{U}_A)^2 \ge 1/\sqrt{2}$, and SC2, as in \Cref{eq:sc2} as in \cref{thm:fro_bound} with $\M{U}_A^\Tra\V{r}_{nls}$ being the matrix product to approximate, both hold with high probability.
With these parameters each SC holds with at least probability $1-\delta/2$ so they both hold with at least probability $1-\delta$.
Thus, from \Cref{eq:intermediate_ineq1} we have that if both SC's hold, then
\begin{align*}
    \frac{1}{\sqrt{2}}\sigma_{min}(\M{A})\|(\hat{\V{x}}_{nls} - \V{x}_{nls})\|_2
    \leq
    \frac{\sqrt{\epsilon_r}\|\V{r}_{nls}\|_2}{\sqrt{2}}, 
\end{align*}
%     \rightarrow 
%     \|\hat{\V{x}}_{nls} - \V{x}_{nls}\|_2
%     \leq \frac{\sqrt{\epsilon_r}\|\V{r}_{nls}\|_2}{\sigma_{min}(\M{A})}.
% \end{align*}
from which the claimed upper bound immediately follows.
% Where the notation $\bm{\rightarrow_{SC's}}$ is used to emphasize the usage of the two SC's.
% This concludes the proof of \cref{thm:nls_bnd}.
\end{proof}

\paragraph{Discussion}
As stated previously \Cref{thm:nls_bnd} is applicable to all problems of the form
\begin{equation}
    \label{eq:cls}
    \min_{\V{x} \in \mathscr{C} }\|\M{A}\V{x} - \V{b}\|_2.
\end{equation}
with overdetermined, full rank $\M{A}$, i.e., all convex LS problems.
This is due to the fact that \cref{eq:qp_ineq1} is a result of convexity.
Additionally, there are other types of sketches which can be used.
For example, the Subsampled Randomized Hadamard Transform discussed in \cite{sarlos_FLSA} and the `sparse embedding matrices' from \cite{sketching4NLA_woodruff} can be applied.
This is because these sketches can satisfy the two SC's with an appropriate number of samples.

%%%%%%%%%%%%%%%%%%%% END CLS  %%%%%%%%%%%%%%%%%%%%%%
%%%%%%%%%%%%%%%%%%%% START Hybrid %%%%%%%%%%%%%%%%%%%%%%
\subsubsection{Analysis of Hybrid Sampling}
\label{sec:hybrid_samp_theory}
This section presents our theoretical results for Hybrid Sampling.
According to the proof of \cref{thm:nls_bnd} to show that Hybrid Sampling works for OLS and NLS problems we need only show that the Hybrid Sampling matrix satisfies the two SC's in \Cref{eq:sc1} and \Cref{eq:sc2}.
Applying these results in the proof structure used for \cref{thm:nls_bnd} will yield sampling complexities, theoretical guarantees, and algorithms for Hybrid Sampling and solving NLS problems.

The first Structural Condition for Hybrid Sampling is given by 
\cref{thm:hybrid_su_bound}.
\begin{lemma}
\label{thm:hybrid_su_bound} 
    Given $\M{A} \in \mathbb{R}^{m\times k}$ consider its thin SVD $\M{U}_A\M{\Sigma}_A\M{V}_A^\Tra$ and its row leverage scores $l_i(\M{A})$ for each row $i \in [m]$, where $[m]$ denotes the set of integers from 1 to $m$.
    %Let $0< \tau \le 1$ be the deterministic sampling parameter such that if $\tau \leq l_i(\M{A})/k$ then $i \in \mathscr{I}_D$ and $\mathscr{I}_R = [m]\setminus \mathscr{I}_D$.
    Let $\mathscr{I}_D$ be the set of deterministically included rows and define $\mathscr{I}_R = [m]\setminus \mathscr{I}_D$ as the set of rows to be sampled from.
    Let $\M{S}_{H} \in \mathbb{R}^{s \times m}$ be a row sampling and rescaling matrix contructed via sampling with replacement on $\mathscr{I}_R$ with $s_R$ samples drawn according to the renormalized leverage scores $l_i(\M{A})$, $s_D$ deterministic samples taken from $\mathscr{I}_D$ where $|\mathscr{I}_D| = s_D$, and $s = s_D +s_R$ .
    Define $\theta = \sum_{i \in \mathscr{I}_D} l_i(\M{A})$ and $\xi = k - \theta$.
    If $\frac{144\xi\log(2k/\delta)}{\epsilon_s^2} < s_R$
    then the following equation holds with probability at least $1-\delta$
    \begin{equation*}
        1-\epsilon_s \le \sigma^2_i(\M{S}_H\M{U}_A) \le 1+ \epsilon_s
    \end{equation*}
    for all $i \in [m]$ and $\epsilon_s,\delta \in (0,1)$.
\end{lemma}

This proof is a modification of the proofs in \cite{sketching4NLA_woodruff,Ismail_row_samp} for proving the analogous statement for standard leverage score sampling, which apply a Matrix Chernoff Bound.
A statement of the Matrix Chernoff Bound is available in the Appendix as \cref{thm:matrix_chernoff}.
This theorem allows us to reason about the quantity $Pr\big[\|\M{W}\|_2 > \epsilon \big]$ where $\M{W} = \frac{1}{s_R}\sum_{j=1}^{s_R}\M{X}_j$ and $\M{X}_j$ are draws of an independent, symmetric random matrix $\M{X} \in \mathbb{R}^{k \times k}$.
Additionally let $\mathscr{S}_R$ be the set of samples or draws of $\M{X}$ such that $|\mathscr{S}_R| = s_R$.
To prove \cref{thm:hybrid_su_bound} we set $\M{W} = \M{I} - \M{U}_A^\Tra\M{S}_{H}^\Tra\M{S}_{H}\M{U}_A$, where $\M{S}_{H}$ is given by \cref{eq:SH}.

\begin{proof}
We assume an appropriate row permutation conformal with \cref{eq:SH} such that $\M{U}_A = 
\begin{bmatrix} 
\M{U}_D \\
\M{U}_R
\end{bmatrix}$
where $\M{U}_D \in \mathbb{R}^{(s_D \times k)}$ and $\M{U}_R \in\mathbb{R}^{(m-s_D) \times k}$.
Define $\M{G}_D = (\M{S}_D\M{U}_D)^\Tra\M{S}_D\M{U}_D = \M{U}_D^\Tra\M{U}_D$ and $\M{G}_R = \M{U}_R^\Tra\M{U}_R$.
Then we have $(\M{S}_{H}\M{U}_A)^\Tra(\M{S}_{H}\M{U}_A) = \M{G}_D + (\M{S}_{R}\M{U}_R)^\Tra(\M{S}_{R}\M{U}_R)$.
Let $\V{q}$ be a random row sample drawn from $\M{U}_R$ according to the distribution given by $\tilde{p}_i = \| \V{u}_i \|_2^2 / \| \M{U}_R \|_F^2 = \| \V{u}_i \|_2^2 / \xi$, the corresponding leverage score probability for $i \in \mathscr{I}_R$, and let $\tilde{\pi}$ be $\V{q}$'s leverage score.
Define $\V{q}_j$ for $j \in [s_R]$ to be a sequence of independent draws of $\V{q}$, so that $\V{q}_j=\V{u}_i$ and $\tilde{\pi}_j=\tilde{p}_i$ if the $i$th row of $\M{U}$ is the $j$th sample.
Let $\M{X}_j = \M{I}_k - \M{G}_D - \frac{1}{\tilde{\pi}_j}\V{q}_j\V{q}_j^\Tra$, which is an independent, symmetric random matrix drawn from $\M{X} = \M{I}_k - \M{G}_D - \frac{1}{\tilde{\pi}}\V{q}\V{q}^\Tra$.
Note that for any $i$ we have $\frac{\|\V{u}_i\|_2^2}{\tilde{p}_i} = \frac{\|\V{u}_i\|_2^2}{\|\V{u}_i\|_2^2}\xi = \xi$.

Then
\begin{align*}
    \M{W} = 
    \frac{1}{s_R}\sum_{j=1}^{s_R} \M{X}_j = \frac{1}{s_R}\sum_{j=1}^{s_R}  \Big(\M{I}_k - \M{G}_D - \frac{1}{\tilde{\pi}_j}\V{q}_j\V{q}_j^\Tra \Big) 
    = \M{I}_k - \M{G}_D - \sum_{j=1}^{s_R} \frac{\V{q}_j}{\sqrt{s_R\tilde{\pi}_j}}  \frac{\V{q}_j^\Tra}{\sqrt{s_R\tilde{\pi}_j}}  \\
    %= \Big(\M{I}_k - \M{G}_D - (\M{S}_R\M{U}_R)^\Tra\M{S}_R\M{U}_R\Big)
    = \M{I}_k - \M{G}_D - (\M{S}_R\M{U}_R)^\Tra\M{S}_R\M{U}_R 
    = \M{I}_k - (\M{S}_{H}\M{U}_A)^\Tra\M{S}_{H}\M{U}_A.
\end{align*}
So $\M{I}_k - (\M{S}_{H}\M{U}_A)^\Tra\M{S}_{H}\M{U}_A$ is a sum of symmetric, independent random matrices.
Next, to apply \cref{thm:matrix_chernoff} we need to verify  three conditions: $\E[\M{X}] = 0$, $\|\M{X}\|_2 \le \gamma$, and $\|\E[\M{X}^\Tra\M{X}]\|_2 \le \nu^2$, where $\gamma$ and $\nu^2$ are bounds used in \cref{thm:matrix_chernoff} to be derived.

First, we prove $\E[\M{X}] = \M{0}$.
Observe that $\E[\frac{1}{\tilde{\pi}}\V{q}\V{q}^\Tra] = \M{G}_R$, then
\begin{align*}
    \E[\M{X}] 
    = \E[\M{I}_k - (\M{S}_D\M{U}_D)^\Tra(\M{S}_D\M{U}_D) - \frac{1}{\tilde{\pi}}\V{q}\V{q}^\Tra]
    = \M{G}_R - \E[\frac{1}{\tilde{\pi}} \V{q}\V{q}^\Tra]
    = \M{0}_{k\times k}
\end{align*}

Second, bound $\|\M{X}\|_2 \le \gamma$.
Observe $\frac{1}{\tilde{\pi}}\V{q}\V{q}^\Tra$ is a rank-1 matrix with its 2-norm bounded by $\max_{i\in {\cal I}_R} \frac{\|\V{u}_i\|_2^2}{\tilde{p}_i} = \xi$. 
Then 
\begin{align*}
    \|\M{X}\|_2
    \le \|\M{I}_k - (\M{S}_D\M{U}_D)^\Tra(\M{S}_D\M{U}_D) \|_2 + \| \frac{1}{\tilde{\pi}}\V{q}\V{q}^\Tra\|_2 
    %\leq  1 + \frac{1}{\tilde{p}_i}\|\V{q}_i\V{q}_i^\Tra\|_2
    \leq 1 + \xi.
\end{align*}
So $\gamma = 1 + \xi$ gives the needed bound.

Finally, we show $\|\E[\M{X}^\Tra\M{X}]\| \le \nu^2$.
By expanding $\E[\M{X}^\Tra\M{X}]$ and using the fact that
\begin{align*}
    \E\big[(\V{q}\V{q}^\Tra)(\V{q}\V{q}^\Tra)/\tilde{\pi}^2 \big] 
    = \sum_{i \in \mathscr{I}_R} \frac{1}{\tilde{p}_i}(\V{u}_i\V{u}_i)^\Tra (\V{u}_i \V{u}_i)^\Tra
    = \xi\sum_{i \in \mathscr{I}_R} \V{u}_i\V{u}_i^\Tra
    = \xi \M{U}_R^\Tra\M{U}_R,
\end{align*}
we obtain
\begin{align*}
    \|\E[\M{X}^\Tra\M{X}]\|_2 
    = \| \xi \M{U}_R^\Tra\M{U}_R - \M{G}_R\M{G}_R \|_2
    = \|\M{U}_R^\Tra( \xi\M{I} - \M{U}_R\M{U}_R^\Tra)\M{U}_R\|_2 \\ 
    \le \| \xi\M{I} - \M{U}_R\M{U}_R^\Tra\|_2 
    \le |\xi -\sigma_{max}(\M{U}_R\M{U}_R^\Tra)| \le |\xi - 1|.
\end{align*}
So $\nu^2 = |\xi - 1|$ gives the desired bound.

Substituting these into \Cref{eq:mcb_w_bound} leads to
\begin{align}
  Pr[\|\M{I} - \M{U}_A^\Tra\M{S}_{H}^\Tra\M{S}_{H}\M{U}_A \|_2^2 \ge \epsilon_s] 
  =  \delta \le 2k \exp\Big[-s \epsilon_s^2/\Big(2(|\xi -1|) + 2(1 + \xi)\epsilon_s/3 \Big)\Big]
\end{align}
for some desired probability $\delta$ and desired error tolerance $\epsilon_s$.
By choosing \\$\frac{144\xi\log(2k/\delta)}{\epsilon_s^2} \leq s_R$ we have with probability at least $1-\delta$
\begin{equation*}
    \|\M{I} - \M{U}_A^\Tra\M{S}_{H}^\Tra\M{S}_{H}\M{U}_A\|_2^2 \le \epsilon_s
\end{equation*}
which proves \cref{thm:hybrid_su_bound}.
\end{proof}

The second Structural Condition for Hybrid Sampling is given by \cref{thm:hybrid_fro_bound}.
\begin{lemma}
    \label{thm:hybrid_fro_bound}
    Let the thin SVD of $\M{A} \in \mathbb{R}^{m\times k}$ be $\M{U}_A\M{\Sigma}_A\M{V}_A^\Tra$ and define the nonnegative least squares problem $\V{x}_{nls} = \argmin_{\V{x}\ge 0 }\|\M{A}\V{X}-\V{b}\|_2^2$
    with $\V{b}\in\mathbb{R}^{m\times n}$.
    %, let $0< \tau \le 1$ be a threshold parameter with corresponding $\theta $, $0 \leq \theta \leq k$, 
    Let $\theta = \sum_{i\in\mathscr{I}_D}l_i(\M{A})$ be the sum of the leverage scores corresponding to rows in the deterministically included set.
    Let $\M{S}_{H}$ be the Hybrid Sampling matrix for $\M{A}$ with $\frac{2\xi}{\delta\epsilon_r} \leq s_R$.
    Then the following holds with probability at least $1-\delta$
    \begin{equation*}
        \|\M{U}_A\V{r}_{nls} -\M{U}_A\M{S}^\Tra_{H}\M{S}_{H}\V{r}_{nls}\|_2^2 \le \epsilon_r \mathscr{R}_{nls}^2/2
    \end{equation*}
    where $\mathscr{R}_{nls}^2 = \|\V{r}_{nls}\|_2^2$, $\V{r}_{nls} = \M{A}\V{x}_{nls} - \V{b}$, and $\epsilon_r,\delta \in (0,1)$.
\end{lemma}
\begin{proof}
Define $\M{S}_{\bar{D}}$ to be the $(m-s_{H})\times m$ matrix that picks out the rows of $\M{U}_A$ that are not deterministically sampled by $\M{S}_{D}$.
Row-partition the vector $\V{r}_{nls}$ conformally to $\M{U}_A$ as $\V{r}_{nls} = \begin{bmatrix}
    \V{r}_D \\
    \V{r}_R
\end{bmatrix}$.
Then we have
\begin{align*}
    \|\M{U}_A\V{r}_{nls} - \M{U}_A^\Tra \M{S}_{H}^\Tra\M{S}_{H}\V{r}_{nls}\|_F^2
    = \|\M{U}_A\V{r}_{nls} - \M{U}_D^\Tra \M{S}_D^\Tra\M{S}_D\V{r}_D - \M{U}_R^\Tra \M{S}_R^\Tra\M{S}_R\V{r}_R\|_F^2 \\
    = \|\M{U}_D^\Tra \M{S}_D^\Tra\M{S}_D\V{r}_D + \M{U}_R^\Tra \M{S}_{\bar{D}}^\Tra\M{S}_{\bar{D}}\V{r}_R - \M{U}_D^\Tra \M{S}_D^\Tra\M{S}_D\V{r}_D - \M{U}_R^\Tra \M{S}_R^\Tra\M{S}_R\V{r}_R\|_F^2 \\
    = \|\M{U}_R^\Tra \M{S}_{\bar{D}}^\Tra\M{S}_{\bar{D}}\V{r}_R - \M{U}_R^\Tra \M{S}_R^\Tra\M{S}_R\V{r}_R\|_F^2.
\end{align*}
From this form we can apply \cref{thm:fro_bound} to obtain
\begin{align*}
    \E[\|\M{U}_A\V{r}_{nls} - \M{U}_A^\Tra \M{S}_{H}^\Tra\M{S}_{H}\V{r}_{nls}\|_F^2] 
    =
    \E[\|\M{U}_R^\Tra \V{r}_R - \M{U}_R^\Tra \M{S}_R^\Tra\M{S}_R\V{r}_R\|_F^2 ] \\
    \le 
    \frac{1}{s_R} \|\M{U}_R\|_F^2\|\V{r}_R\|_F^2 
    \le
    \frac{\xi}{s_R} \|\V{r}_{nls}\|_F^2 \le \frac{\xi}{s_R} \mathscr{R}^2_{nls}.
\end{align*}
Finally using Markov's Inequality we have
\begin{align*}
    Pr\Big[ \|\M{U}^\Tra \M{S}_{H}^\Tra\M{S}_{H}\V{r}_{nls}\|_F^2
    \ge \frac{\epsilon_r \|\V{r}_{nls}\|_F^2}{2} \Big] 
    \le 2\frac{\E[\|\M{U}^\Tra \M{S}_{H}^\Tra\M{S}_{H}\V{r}_{nls}\|_F^2]}{\epsilon_r\|\V{r}_{nls}\|_F^2} 
    \le 2\frac{\frac{\xi}{s_R} \mathscr{R}^2_{nls}}{\epsilon_r\|\V{r}_{nls}\|_F^2}
    = \frac{2\xi}{\epsilon_r s_R}
\end{align*}
To succeed with probability at least $\delta$, we need $\frac{2\xi}{\epsilon_r \delta} \leq s_R$.
\end{proof}

\paragraph{Discussion of \Cref{thm:hybrid_su_bound} and \Cref{thm:hybrid_fro_bound}}
\cref{thm:hybrid_su_bound} and \cref{thm:hybrid_fro_bound} tell us that if $\theta = \sum_{i \in \mathscr{I}_D} l_i(\M{A})$ proportion of the $k$ total leverage score `mass' is taken deterministically with $s_D$ rows, one needs only to take an additional $\xi\phi$ random samples, where $\xi = k - \theta$, instead of $k\phi$ to achieve the same theoretical guarantees for NLS or OLS problems, 
where $\phi = \max(C \log(k/\delta),1/(\delta \epsilon_r)) $ and $C = 144/(1-\sqrt{2})^2$.
Overall hybrid sampling requires $s_D + \xi\phi$ samples whereas standard leverage score sampling requires $k \phi$ samples.
For hybrid sampling to result in a lower sample complexity we must have $s_D < (k-\xi)\phi = \theta\phi$.
This relies on the deterministic samples accounting for a sufficiently large $\theta$.
This helps to explain why the hybrid method typically outperforms the deterministic method when the same number of samples are taken and supports the experimental results for the CP decomposition of sparse tensors in \cite{lev_scores4_CP}.

To determine the deterministic inclusion set $\mathscr{I}_D$ we use the thresholding technique from \cite{lev_scores4_CP}.
This method chooses a parameter $0< \tau \le 1$ and then the deterministic set includes all rows such that $\tau \leq l_i(\M{A})/k$.
This means that all rows with a leverage score value greater or equal to the threshold $\tau$ are deterministically included.
\cref{thm:hybrid_su_bound} and \cref{thm:hybrid_fro_bound} do not strictly rely on using a thresholding technique to determine which rows should be deterministically included.
Both results easily hold in the case that a different method is used to determine the set $\mathscr{I}_D$.
Any such technique for determining $\mathscr{I}_D$, like the thresholding one, will likely want to keep $|s_D|$ small while making $\theta$ as large as possible.

%%%%%%%%%%%%%%%%%%%% END Hybrid %%%%%%%%%%%%%%%%%%%%%%

%%%%%%%%%%%%%%%%%%%%%%%%%%%%%%%%%%%%%

%\input{Theory_Section}
\section{Experimental Results}
\label{sec:experiments}
This section presents empirical findings on the proposed SymNMF algorithms.
The section focuses on two data sets, one dense and one sparse.
Both data sets are represented as graphs and the proposed SymNMF algorithms are used to cluster the vertices of the graphs following approaches given in previous work  \cite{edvw_hayashi,DaKuang_SymNMF}. Namely, clustering via \cite{edvw_hayashi} proceeds by constructing a hypergraph random walk from the data, which is then used to construct a symmetric adjacency matrix that serves as the input to SymNMF clustering.
For all experiments we use the initialization strategy from \cite{DaKuang_SymNMF}, though we acknowledge other initialization techniques such as those based on the SVD \cite{BOUTSIDIS20081350}.
The strategy in \cite{DaKuang_SymNMF} proceeds by generating a initial matrix for $\M{H} \in \mathbb{R}_{+}^{m\times k}$ with elements drawn uniformly from the interval on 0 to 1, and scaling them by the value $2*\sqrt{\zeta/k}$ where $\zeta$ is the average of all elements of $\M{X}$.
Intuitively this ensures that the norm of the initial factor matrices is not too large or small in comparison to $\M{X}$.
Given an output factor $\M{H}$ the cluster membership for the $i$th vertex $v$ is determined by finding the column index with the maximum value in the $i$th row of $\M{H}$ \cite{DaKuang_SymNMF}.

For checking convergence we use the normalized residual and the norm of the projected gradient.
These are two standard metrics for checking the convergence of NMF and its variants \cite{BCD_park,rnd_HALS}.
Definitions and some more discussion of these two metrics can be found in \cref{app:stopping_crit}.

%%=========
% WoS data set
%%=========
\subsection{Web of Science Text Data}
For the dense, symmetric case we utilize the Web of Science (WoS) Text\footnote{https://data.mendeley.com/datasets/9rw3vkcfy4/6} data set.
%This data set is documents by terms and we use a term-frequency inverse-document-frequency (TF-IDF), bag of words model for its matrix representation.
The collection contains 46985 documents and 58120 terms.
The number of term-document relationships is roughly $0.0013\%$ of all possible connections.
Each document is assigned one of 7 categories based on its content which we use as ground truth labels.

We cluster the data set using SymNMF via the Hypergraph with Edge Dependent Vertex Weights (EDVW) methodology \cite{edvw_hayashi}.
This turns the matrix into a symmetric adjacency matrix where documents are vertices and words are considered as hyperedges.
The matrix is likely dense as each hyperedge is expanded into a weighted clique in the obtained adjacency matrix.
To assess clustering quality we compute the Adjusted Rand Index (ARI) for the clusterings produced by each SymNMF algorithm.
To cluster with SymNMF we follow the methodology in \cite{JNMF_jogo}.

%For this set of experiments we use reduction in the normalized residual norm to determine when to stop iterating.
Once the algorithms are unable to reduce the normalized residual by more than \texttt{1e-4} for four consecutive iterations the methods stop.
If Iterative Refinement (IR), described in \Cref{sec:refine_symnmf}, is being used the method will switch over to using the full input data matrix and apply the same stopping criteria to determine when to stop fitting to the full input matrix.
All methods use \ref{alg:adaRRF} to determine how many power iterations to perform.
\ref{alg:adaRRF} iterates until the normalized residual ceases to decrease by \texttt{1e-3} per power iteration.
Each method is run 10 times.

We use a labeling system to name the algorithms.
The algorithms we consider here vary in terms of Update rules among 1) HALS, 2) BPP, 3) PGNCG,
if they use LAI, Compression (Comp) as in \cite{Comp_NMF2016}, or neither,
and finally if Iterative Refinement is used, denoted with IR.
A combination of these labels indicates the method used.
For example LAI-BPP-IR is the \cref{alg:lai_symnmf} with BPP as the update rule and iterative refinement run at the end.
LAI-BPP indicates the same technique but without IR at the end.
Finally BPP indicates standard SymNMF method with BPP as the update rule.

The regularized version of SymNMF in \Cref{eq:surrogate_symnmf_obj} requires an input hyperparameter $\alpha$ to balance the two objectives.
The compressed problem also needs this hyperparameter.
In practice we find that using the value  $\alpha = \max(\M{X})$ recommended for the uncompressed problem works well \cite{DaKuang_SymNMF}.

\subsubsection{Results}
The convergence results for these experiments are shown in \cref{fig:WoS_NR_plts}.
These show that the \ref{alg:lai_symnmf} method results in significant speed up. \ref{alg:lai_symnmf} without Iterative Refinement (IR) achieves about a 4$\times$ speed up over standard SymNMF with BPP.
For \ref{alg:lai_symnmf} with HALS we observe a speed up of 7.5$\times$.
For both BPP and HALS, IR does not appear to be needed when \ref{alg:adaRRF} finds a good starting low-rank approximation.
For comparison we include the SymNMF variant of the ``Compressed NMF'' method, from \cite{Comp_NMF2016}, and see that it performs nearly identically to \ref{alg:lai_symnmf}.
See \cref{app:lai_vs_comp} for a discussion of the similarity of the two methods.
In terms of the PGNCG method in \cref{fig:WoS_nr_GNCG_p30} and \cref{tab:WoS_symnmf_exp} we see that randomization greatly benefits PGNCG resulting in about a $6\times$ speed up.
Additionally, IR in conjunction with PGNCG achieves the lowest normalized residual norm.
Overall we see that randomized methods for SymNMF result in significant speed up while preserving, or even improving, solution quality.
The run time of the randomized methods included the time needed to compute the LAI.
This is why all the randomized methods appear to `start' later than the non-randomized methods.

Note that we do not report results for \ref{alg:lvs_nmf} on this data set as the code took too long to execute.
This is due to the fact that copying large portions (sampled rows) of a large dense data matrix $\M{X}$ at each iteration takes a long time.

Additional experimental data can be found in \cref{sec:additional_wos_exps}.
\cref{fig:WoS_NR_plts_vp} and \cref{tab:WoS_exp_additional_p40,tab:WoS_exp_additional_p80} show results for different values of the oversampling parameter $\rho$.
We find that varying $\rho$ does not have much effect on the solution quality but can negatively impact run time.
Finally, \cref{sec:additional_wos_exps} includes plots and tables showing statistics for using a static choice of $q=2$ with IR.
Comparing these results with those from using \cref{alg:adaRRF} shows the superiority of \cref{alg:adaRRF} versus a static choice of $q$.

We briefly compare against Spectral Clustering as described in \cite{Ng_spectral_clust,edvw_hayashi}
to validate SymNMF-based clustering results.
Spectral clustering achieves an average ARI of $0.293$ over 20 runs.
This is a worse ARI score than all of the SymNMF methods we try, see \cref{tab:WoS_symnmf_exp}.
Spectral clustering using the MATLAB functions \textbf{eigs()} and \textbf{kmeans()} takes about $35$ seconds to execute.
We also note that the normalized residual achieved by a rank $=7$ SVD is $0.9340$.

%\begin{figure}
%     \centering
%     \begin{subfigure}[b]{1.0\textwidth}
%         \centering
%         \includegraphics[width=\textwidth]{Figs/WoS_exp20_NR.png}
%         \caption{Residual Plots for the WoS data set.}
%         \label{fig:WoS_residual_plt2}
%     \end{subfigure}
%     \hfill     
%     \begin{subfigure}[b]{1.0\textwidth}
%         \centering
%         \includegraphics[width=\textwidth]{Figs/WoS_exp20_NR_zoom.png}
%         \caption{Zoomed in version of Figure \ref{fig:WoS_residual_plt2}}
%         \label{fig:WoS_residual_plt_Z}
%     \end{subfigure}
%     \hfill
%        \caption{Normalize Residual error value for various SymNMF Algorithms on the WoS data set using a EDVW Hypergraph Representation.}
%        \label{fig:three graphs}
%\end{figure}

\begin{figure}
     \centering
     \begin{subfigure}[b]{0.48\textwidth}
         \centering
         \includegraphics[width=\textwidth]{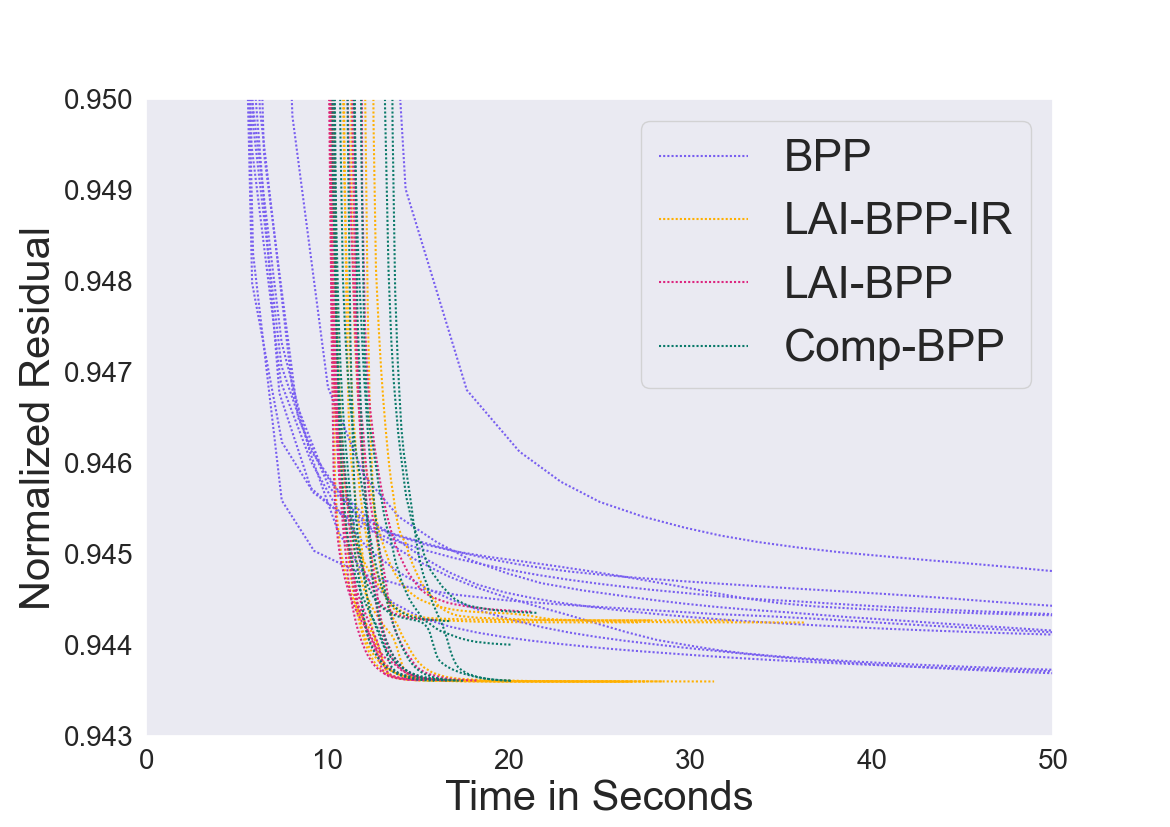}
         \caption{BPP}
         \label{fig:WoS_residual_BPP}
     \end{subfigure}
     %\hfill     
     \begin{subfigure}[b]{0.48\textwidth}
         \centering
         \includegraphics[width=\textwidth]{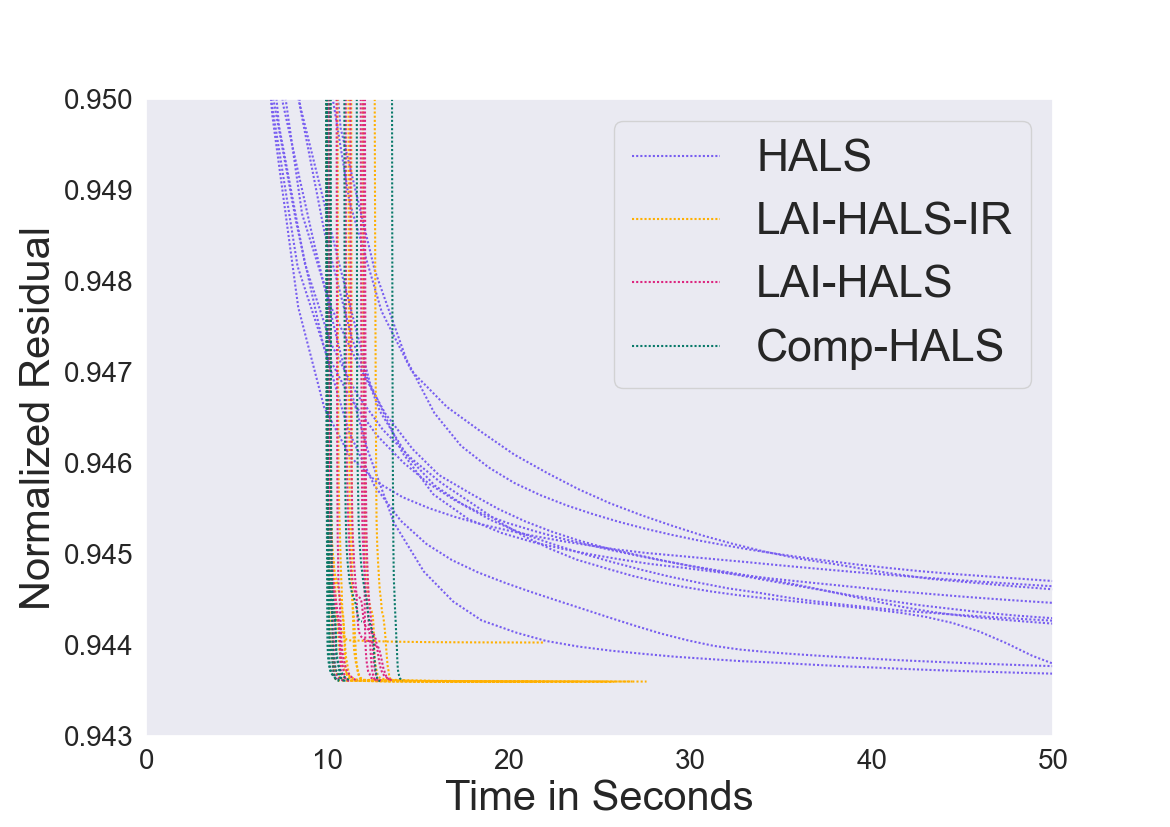}
         \caption{HALS}
         \label{fig:WoS_residual_HALS}
     \end{subfigure}

     \begin{subfigure}[b]{0.48\textwidth}
         \centering
         \includegraphics[width=\textwidth]{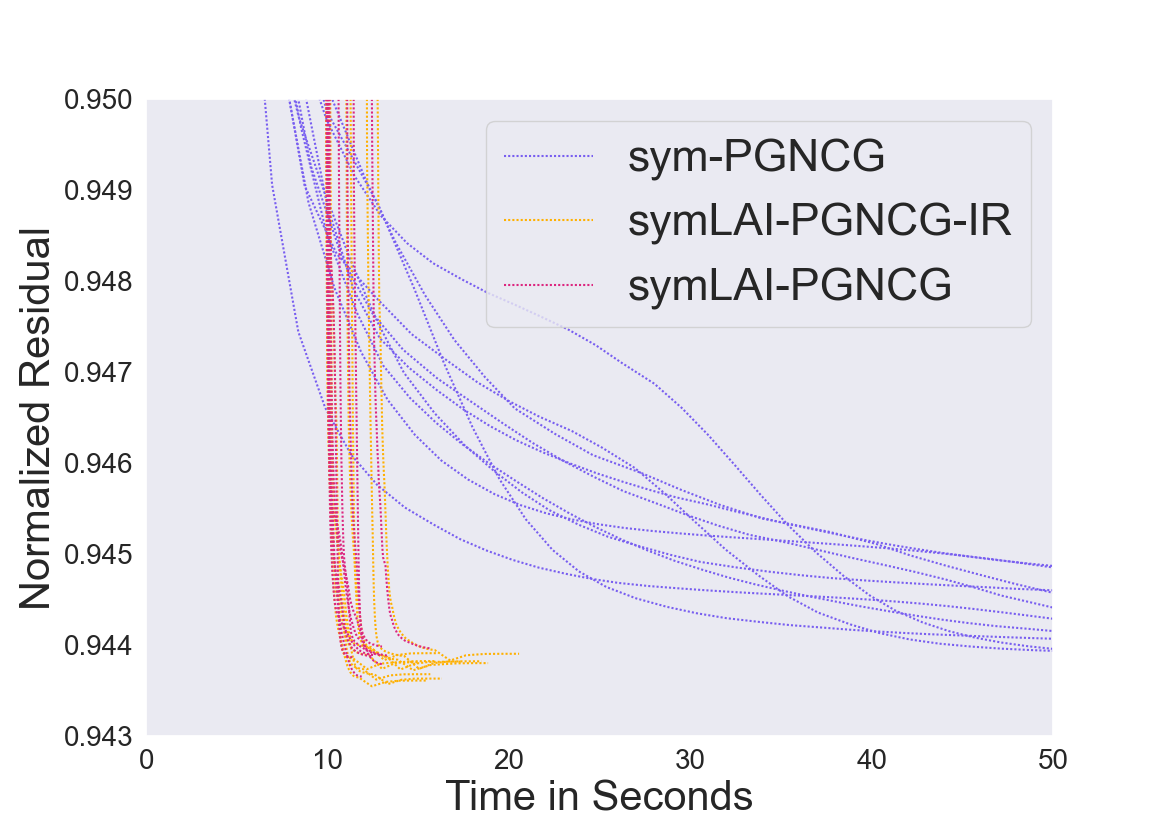}
         \caption{PGNCG}
         \label{fig:WoS_residual_GNCG}
     \end{subfigure}  
\caption{Normalized Residual Norm Plots for the WoS data set. Three different update rules are shown: BPP, HALS, and PGNCG.}
\label{fig:WoS_NR_plts}
\end{figure}

\begin{table}
\resizebox{\columnwidth}{!}{
\begin{tabular}{|l||c|c|c|c|c|}%
\centering
\bfseries Alg. & \bfseries Iters & \bfseries Time & \bfseries Avg. Min-Res & \bfseries Min-Res & \bfseries Mean-ARI \\ \hline % specify table head
%\csvreader[head to column names]{CSV_Data/WoS_ALL_exp29_T.csv}{}% use head of csv as column names
%{\\\hline \Alg & \AvgIters & \Time & \AvgMinRes & \MinRes & \MeanARI}% specify your coloumns here
%Alg & AvgIters & Time & AvgMinRes & MinRes & MeanARI & MeanNcut
PGNCG & 50.8 & 80.311 & 0.944 & 0.9437 & 0.3063 \\
LAI-PGNCG & 92.1 & 12.834 & 0.9439 & 0.9437 & 0.3107 \\
PGNCG-IR & 93.8 & 16.984 & 0.9437 & 0.9435 & 0.3161 \\
\hline
BPP & 38.1 & 66.95 & 0.9439 & 0.9436 & 0.3224 \\
LAI-BPP & 71.5 & 16.734 & 0.9438 & 0.9436 & 0.3264 \\
LAI-BPP-IR & 79.0 & 29.062 & 0.9439 & 0.9436 & 0.3148 \\
Comp-BPP & 79.3 & 18.266 & 0.9438 & 0.9436 & 0.3207 \\
\hline
HALS & 58.1 & 92.461 & 0.9437 & 0.9436 & 0.3201 \\
LAI-HALS & 97.0 & 12.124 & 0.9436 & 0.9436 & 0.3269 \\
LAI-HALS-IR & 90.5 & 23.799 & 0.9436 & 0.9436 & 0.3205 \\
Comp-HALS & 75.0 & 11.454 & 0.9437 & 0.9436 & 0.3237 
\end{tabular}
}
\caption{Metrics for the various run shown in \cref{fig:WoS_NR_plts}. Iters is mean number of iterations taken, Time is mean run time in seconds, Avg. Min-Res is average minimum achieved residual, Min-res is overall lowest achieved residual, and Mean-ARI is the mean ARI score. Averages are taken over 20 runs.}
\label{tab:WoS_symnmf_exp}
\end{table}

%%%%%%%%%%%%%%%%%%%%%%%%%%%%%%%%%%%%%%%%%%%%%
%% Sparse Experiments
%%%%%%%%%%%%%%%%%%%%%%%%%%%%%%%%%
\subsection{Microsoft Open Academic Graph}
We run our methods for SymNMF on the Microsoft Open Academic Graph (OAG) \cite{zhang2019oag}.
This is a data set that combines the Microsoft Academic Graph (MAG) \cite{microsoft_graph_sinha} and the Arnet-Miner (Aminer) academic graph \cite{arnetminer_tang}.
From the OAG\footnote{https://github.com/ramkikannan/planc} we take a subset of 37,732,477 papers and use their citation information to form a sparse graph with 966,206,008 non-zeros as in \cite{ics_paper}.
This adjacency matrix is symmetrically normalized and the diagonal is zeroed out following the methodology in \cite{DaKuang_SymNMF}.
The rank is set to $k=16$ for all experiments. 
The parameter $s$ is set to $\lceil 0.05*m\rceil$ where $m$ is the dimension of the square input matrix.

The regularized version of SymNMF in \Cref{eq:surrogate_symnmf_obj} requires an input hyperparameter $\alpha$ to balance the two objectives.
The sampled problem will also need this hyperparameter.
We propose that using the same values of $\alpha$ for the leverage score sampling method and the deterministic method is reasonable since $\mathbb{E} [\|\M{S}_H\M{H}\|_F^2] = \|\M{H}\|_F^2$ and $\mathbb{E} [\|\M{S}_W\M{W}\|_F^2] = \|\M{W}\|_F^2$.

\cref{alg:lvs_nmf} is run with BPP and HALS as the update rules.
The normalized residual norms and the projected gradient values are shown in \cref{fig:OAG_plts}.
See \Cref{sec:prj_grads} for details on projected gradient values.
First we note that the LAI method with BPP update rule is generally unsuccessful at reducing the error.
We hypothesize that this is because the \ref{alg:QB_decomposition} produces a LAI with a large value of $\| \M{U}_X\M{V}_X - \M{W}_+(\M{H}_+)^\Tra\|_F$ as in \Cref{eq:lai_2nd_ineq}. 

Next we note that using the Hybrid Sampling method is important to achieve good speed-up.
In the case of the BPP update rule, in \cref{fig:aminer_residual_BPP} one can observe that the purely random method ($\tau = 1$) does not provide any speed up over the standard BPP method.
However when the hybrid method with $\tau=1/s$ is used, one can see that the leverage score sampling algorithm becomes competitive.
When the BPP update rule is used, a significant amount of the time per iteration goes into solving the $2m$ NLS problems.
So while leverage score sampling effectively eliminates the cost of multiplying by $\M{X}$, only about a $50\%$ speed up is achieved as leverage score sampling has no effect on the cost of the update function. In the case of HALS, we observe significantly better speed up when leverage score sampling is used.
\cref{fig:aminer_residual_BPP} clearly shows that hybrid leverage score sampling outperforms both leverage score sampling and the standard HALS method.
The hybrid leverage score sampling method in this case is able to achieve about a $5.5\times$ speed up per iteration over standard HALS. 
We note that the use of the modified HALS update rules, as in \Cref{eq:symhals_update1,eq:symhals_update2}, are crucial for obtaining good performance.

\cref{fig:aminer_time_bd} shows the time break down for 3 algorithms: leverage score sampling with HALS and BPP as the update rule and a standard NMF method with HALS as the update rule.
The timings are divided into 3 categories: 1) ``Matrix Multiplication'' (MM) for computing the four main matrix products, 2) ``Solve'' for the time spent in the update rule (e.g. time spent in the BPP solver), and 3) ``Sampling'' time spent performing leverage score sampling which includes computing the thin QR decomposition.
The sampling time also includes time to form the needed subsampled matrices.
The OAG data set is stored using MATLAB's sparse matrix representation which allows for fast row/column slicing (because the matrix is symmetric) via the Compressed Sparse Column or (CSC) format.
This is in juxtaposition to the hypergraph representation of the WoS data set which is dense and so does not benefit from sparse matrix formats.
Additionally, the MM time is greatly reduced by using leverage score sampling while incurring an acceptable overhead in terms of sampling time.
Lastly we see that this data confirms that the BPP solver has limited potential computational gains from leverage score sampling.

\Cref{fig:OAG_plts} also shows the projected gradient values for the OAG runs.
As discussed in the next section, the solutions found by the randomized and non-randomized methods are quite different.
The fact that the randomized methods does a good job of reducing both the Residual and the Projected Gradient helps increase confidence that \cref{alg:lvs_nmf} is performing well.
See \cref{sec:prj_grads} for details on the Projected Gradient.

\begin{figure}
     \centering
     \begin{subfigure}[b]{0.48\textwidth}
         \centering
         \includegraphics[width=\textwidth]{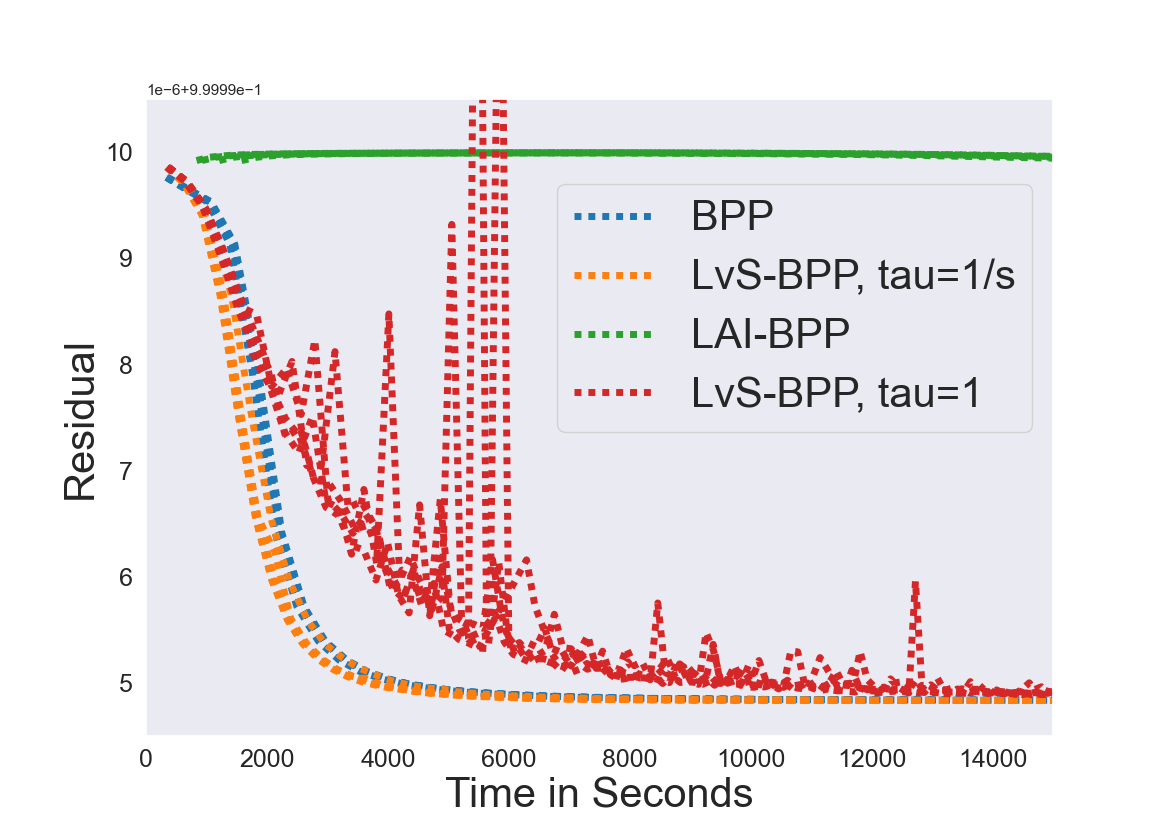}
         \caption{BPP, Normalized Residual Value}
         \label{fig:aminer_residual_BPP}
     \end{subfigure}
     \hfill     
     \begin{subfigure}[b]{0.48\textwidth}
         \centering
         \includegraphics[width=\textwidth]{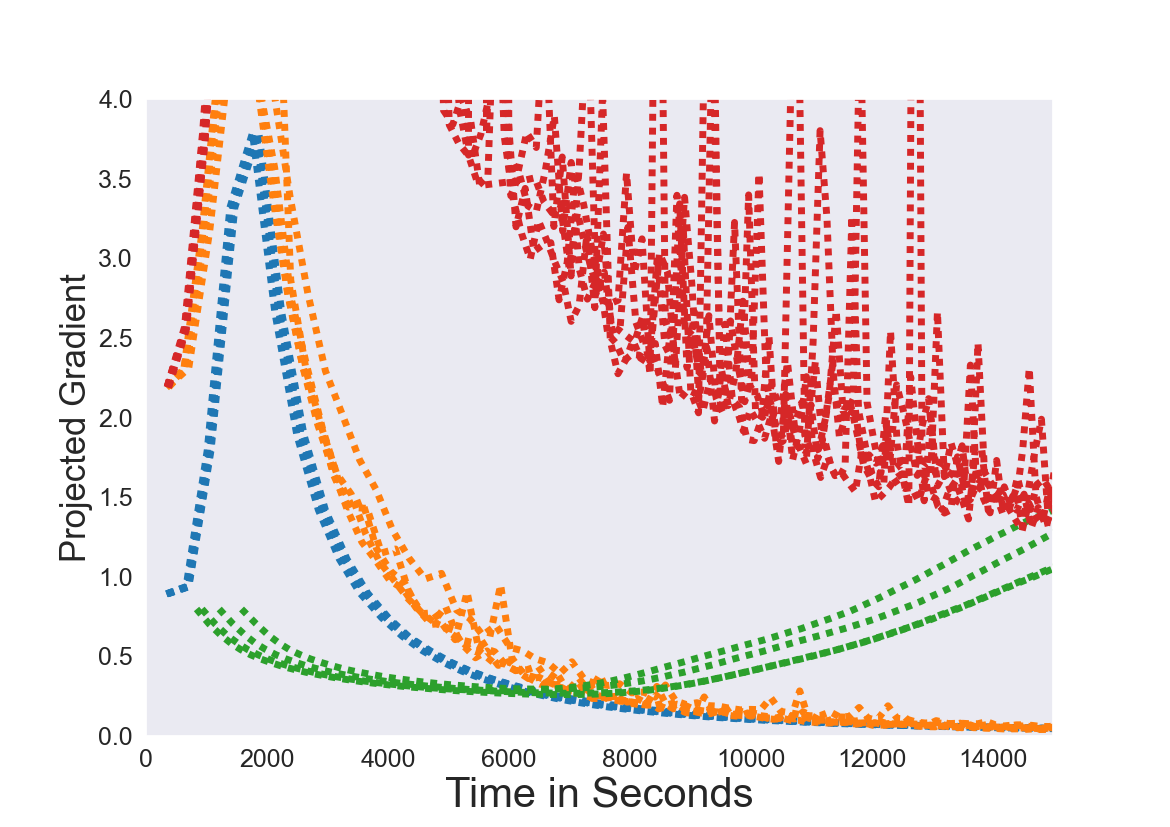}
         \caption{BPP, Projected Gradient Value}
         \label{fig:aminer_pg_BPP}
     \end{subfigure}
%\end{figure}

%\begin{figure}
     \centering
     \begin{subfigure}[b]{0.48\textwidth}
         \centering
         \includegraphics[width=\textwidth]{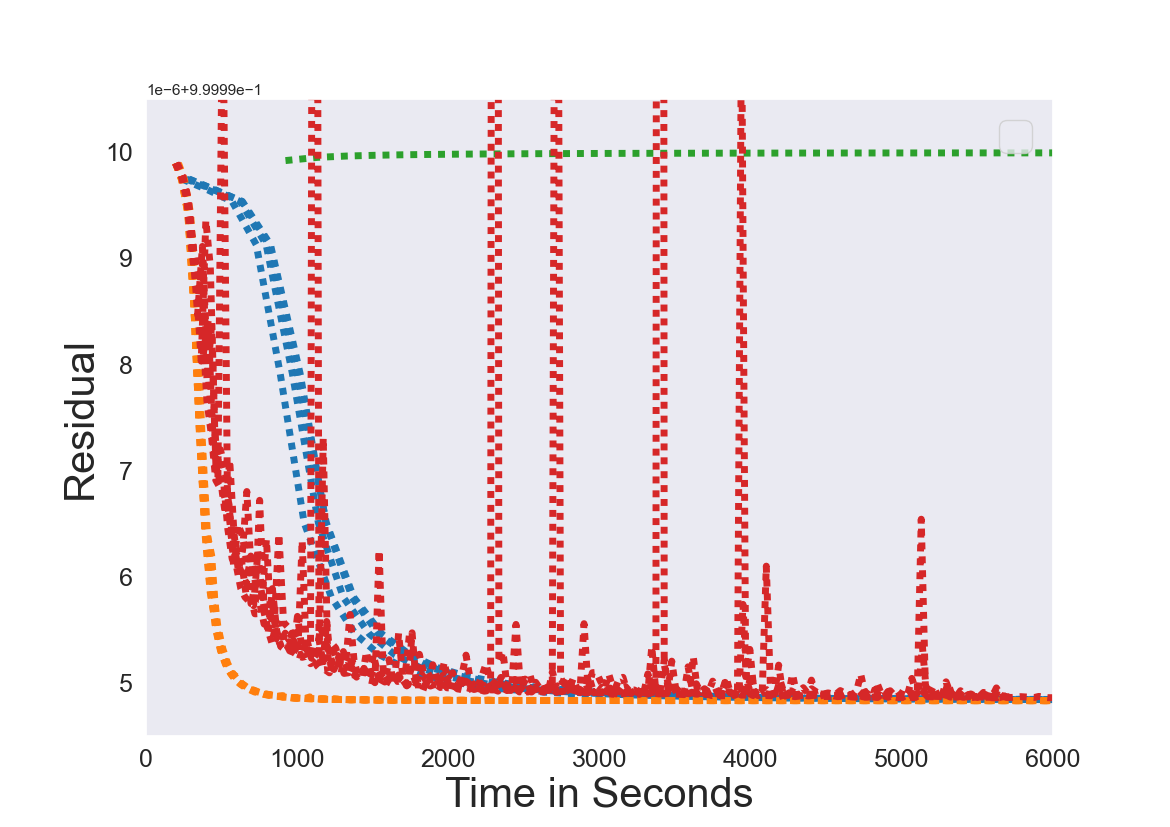}
         \caption{HALS, Normalized Residual Value}
         \label{fig:aminer_residual_HALS}
     \end{subfigure}
     \hfill     
     \begin{subfigure}[b]{0.48\textwidth}
         \centering
         \includegraphics[width=\textwidth]{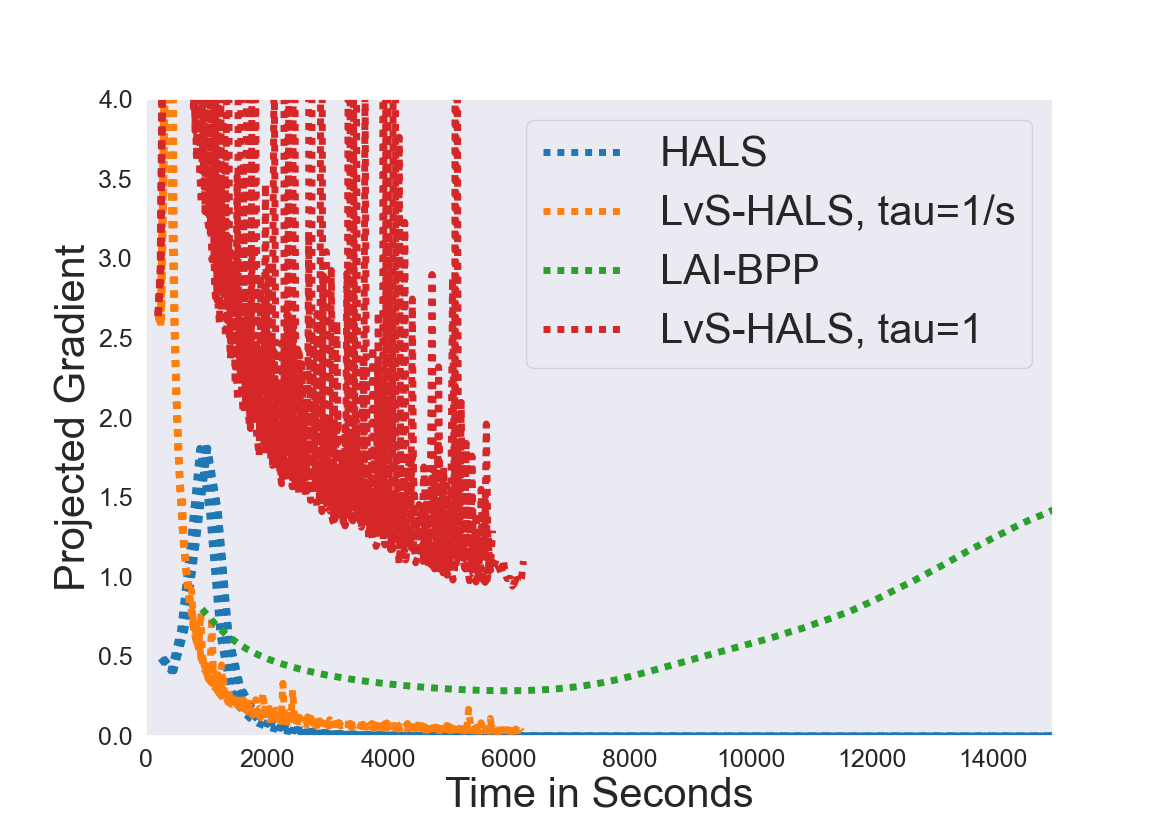}
         \caption{HALS, Projected Gradient Value}
         \label{fig:aminer_pg_HALS}
     \end{subfigure}
    \caption{Normalized residual and projected gradient values for OAG using HALS and BPP update rules. $k=16$, $\tau=1/s$ or $\tau=1$ as indicated in the legend.}
    \label{fig:OAG_plts}
\end{figure}

\begin{figure}
     \centering  
     \includegraphics[width=0.7\textwidth]{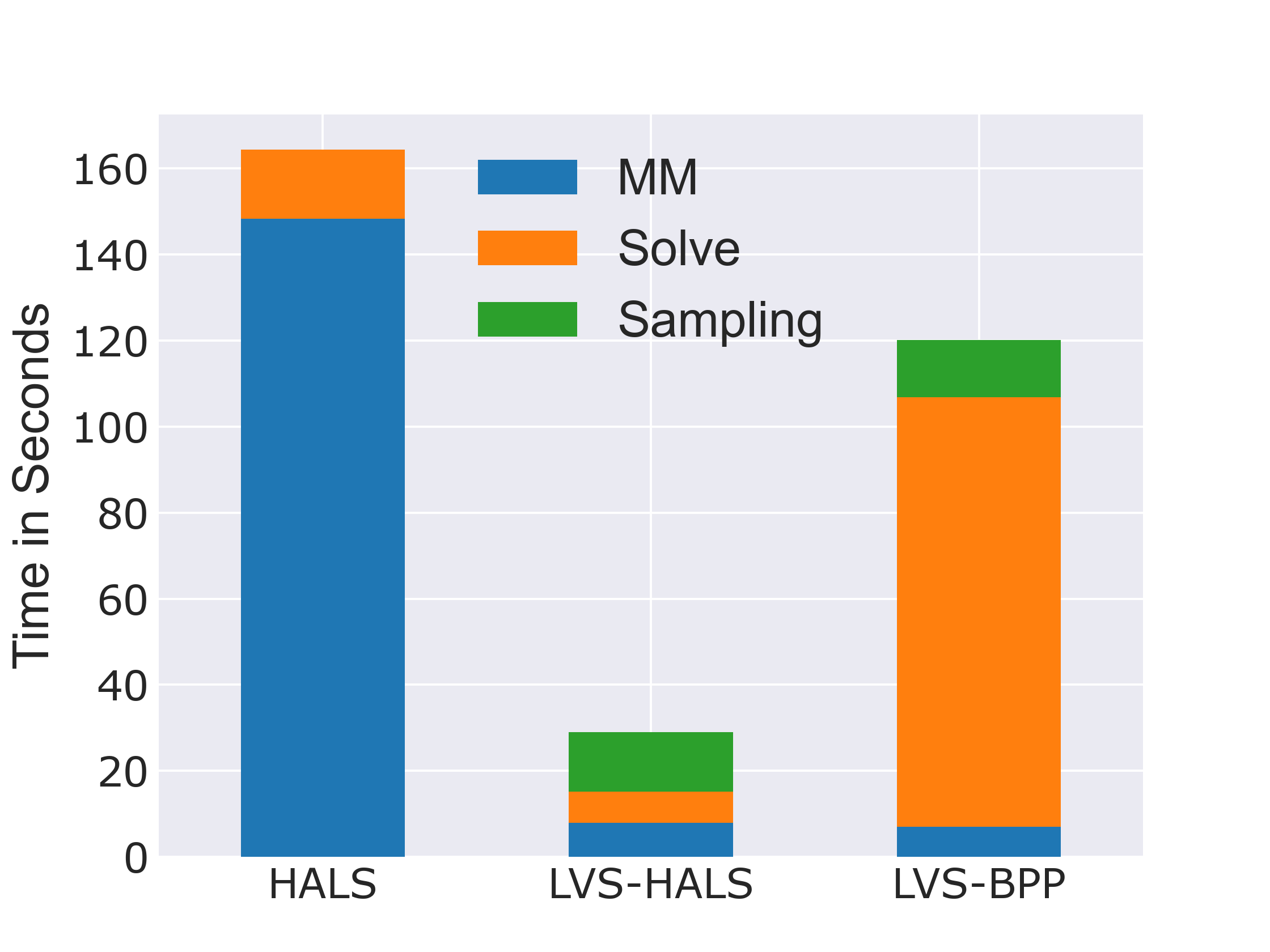}
    \caption{Time Breakdown per iteration for three algorithms. Standard SymNMF with HALS update rule (HALS), SymNMF with leverage score sampling and HALS (LvS-HALS), and SymNMF with leverage score sampling and BPP (LvS-BPP).}
    \label{fig:aminer_time_bd}
\end{figure}

\subsubsection{Results for the OAG}
We now analyze the output of running the \ref{alg:lvs_nmf} algorithms on the OAG.
The factors with $k=16$ are used to form $k$ clusters, as in the WoS hypergraph, following the methodology in \cite{DaKuang_SymNMF}.
Since this data set does not come with ground truth, assessing the quality of the \ref{alg:lvs_nmf} outputs is more difficult.
The residuals output by both the HALS and the  LvS-HALS method  differ by a tiny amount, on the order of \texttt{1e-12}, with the leverage score method achieving the lower residual.
%So strict residual comparison is unsatisfactory.

We use a similar labeling scheme as for the WoS experiments.
LvS represents the use of leverage scores sampling and the considered update rules are HALS and BPP.
So LvS-BPP represents \Cref{alg:lvs_nmf} with BPP.
Additionally for leverage score sampling methods the value of $\tau$ used is also indicated.

The two sets of clusters produced by the methods are different, as can be immediately deduced from the difference in cluster sizes.
Standard HALS produces one huge cluster containing all but around $100,000$ of the vertices.
These $100,000$ vertices are then split into an additional 15 clusters whose sizes range from around 5000 to 8000 vertices.
The LvS-HALS method finds 3 large clusters of sizes about $20$ million, $17$ million, and $150,000$.
The remaining 13 clusters have between one to six thousand vertices each.
%In \cref{fig:aminer_heatmaps}  the number of edges between the 16 different clusters is shown via a heat map.
Cluster-0 is the largest cluster for both methods.
Cluster-15 and Cluster-10 are the $17$ million and $150$ thousand sized cluster for LVS-HALS.
The smaller clusters are all somewhat connected to Cluster-0 and sparsely, if at all, connected to each other.
%The left portion of \cref{fig:aminer_heatmaps} shows that the standard HALS algorithm finds many clusters that are well separated.
%The right portion of the graph shows that the LvS-HALS algorithm is still able to find cluster structure but the clusters do not appear to be as cleanly separated as the HALS output.
%For LvS-HALS the intra-cluster connections are much denser than the inter-cluster connections.

To analyze these clusters more rigorously we compute the Silhouette Scores for each vertex \cite{sil_scores}.
For a set of clusters $\{C_j\}$ for $j\in [k]$ we define the following quantities for a vertex $v$ in cluster $C_l$
\begin{align*}
    a(v) = \frac{1}{|C_l|-1}\sum_{j \in C_l,j\neq v} \M{A}_{vj} \ \ \text{and} \ \ b(v) = \max_{t, t \neq l} \frac{1}{|C_t|}\sum_{j \in C_t} \M{A}_{vj}.
\end{align*}
Here $a(v)$ measures a vertex's similarity to its own cluster, $b(v)$ measures a vertex's similarity to the next most similar cluster and $\M{A}$ is the graph adjacency matrix input to the algorithm SymNMF.
The Silhouette Score for $v$ is then given by
\begin{equation*}
    s(v) = \frac{a(v) - b(v)}{\max(a(v),b(v))}.
\end{equation*}
Silhouette Scores range from $1$ to $-1$, with $1$ meaning a vertex is very similar to its own cluster and $-1$ meaning it should be moved to its next closest cluster.
Individual vertex scores are then averaged over their cluster memberships to obtain cluster level Silhouette Scores.
Note that this definition of Silhouette Score is for similarity metrics and so is different that the typical Silhouette Score equation which is defined for dissimilarity metrics. 

The Silhouette Scores are computed for each cluster.
The Silhouette Scores for the HALS methods are perfect with the exception of Cluster-0, which still achieves a high score of 0.78.
The Silhouette Scores for the LVS-HALS algorithm are not as high as the ones for HALS but are still well above 0.
The exception is Cluster-0 (from LVS-HALS), which has a near 0 Silhouette Score of $-0.04$, meaning that it does not form a good cluster but also that its vertices do not necessarily belong in one of the other clusters.
Other than Cluster-0 the LVS-HALS method finds clusters with mean Silhouette Score $0.92 \pm 0.1$ ($\pm$ Standard Deviation).
This indicates that the LVS-HALS method is able to find well structured clusters in terms of Silhouette Score.
To further validate that LVS-HALS is finding meaningful clusters we show the top 10 words associated with Clusters 10 to 15 in \cref{tab:t10words_lvshals}.
More data about this experiment can be found in \cref{sec:moag_appendix}.

%In all of our experiments the computed residual values are quite large, meaning they are close to $1$.
%Some unfamiliar users might find this fact worrisome as this means our low-rank model is approximating a very small fraction of the data set.
%However, this is not a concern as our methods are finding `global' structure of the data set.
%See Chapter 9 in \cite{gillis_book} for a discussion of this topic.

\begin{table}
\resizebox{\columnwidth}{!}{
\begin{tabular}{|c|c|c|c|c|c|c|c|c|c|}%
\hline
\centering
\bfseries T10 & \bfseries T11 & \bfseries T12 & \bfseries T13 & \bfseries T14 & \bfseries T15 % specify table head
\begin{filecontents*}{CSV_Data/pdT_lvs_twords.csv}
Label,topicA,topicB,topicC,topicD,topicE,topicF,topicG,topicH,topicI,topicJ,topicK,topicL,topicM,topicN,topicO,topicP
twordsA,zone,teaching,thermal,transform,beam,beam,equation,fast,forms,summary,matter,optical,determination,flow,equivalent,signaling
twordsB,teaching,industrial,limits,signals,storage,project,stability,teaching,fundamental,minimal,observations,voltage,organic,devices,constant,transcription
twordsC,construction,service,equilibrium,machine,hand,teaching,functional,summary,machine,hand,emission,gain,summary,fluid,weight,kinase
twordsD,forward,supply,summary,voltage,particle,existence,prove,taking,teaching,methyl,cluster,electron,block,integrated,forms,receptor
twordsE,situation,actual,geometry,power,operation,mathematical,physics,supply,mathematical,industry,ray,strain,absorption,driven,coefficient,genes
twordsF,law,summary,regional,parallel,modes,prove,beam,form,form,operation,infrared,layers,iii,electric,half,receptors
twordsG,education,industry,heat,signal,coupled,proved,fixed,situation,measuring,maintenance,mass,laser,compounds,film,proved,pathways
twordsH,chinese,calculation,hand,feature,calculation,tool,space,finite,social,technical,evolution,films,gel,device,plane,proteins
twordsI,china,international,fast,classification,integration,modern,weak,article,past,integration,physics,crystal,ions,voltage,establish,mediated
twordsJ,industry,simulation,units,space,china,plane,equations,international,operating,project,simulations,layer,chemistry,liquid,form,pathway
\end{filecontents*}
\csvreader[head to column names]{CSV_Data/pdT_lvs_twords.csv}{}% use head of csv as column names
{\\\hline \topicK & \topicL & \topicM & \topicN & \topicO & \topicP}% specify your coloumns here
\\ \hline
\end{tabular}
}
\caption{Top key words for Leverage Score HALS output. Topics 10-15 from the Microsoft Open Academic Graph run using HALS as the update rule with Leverage Score sampling ($\tau=1/s$) and $k=16$. The 10 top key words in terms of tf-idf association are shown in the table. We can see many of the top key words per topic seem to form a coherent subject matter.}
\label{tab:t10words_lvshals}
\end{table}

\section{Conclusion}
We have presented two novel algorithms for computing SymNMF in randomized ways.
These methods are suitable for large problems as demonstrated by our findings on two data sets.
Our proposed methods achieve significant speed ups of 5$\times$ or more and produce quality solutions on downstream clustering tasks.
Additionally our techniques are applicable to standard NMF formulations as well. 
Finally we presented a number of theoretical results that justify why and when our proposed algorithms perform well in practice.
These results include an analysis of leverage score sampling for approximately solving Nonnegative Least Squares problems, sample complexity analysis for the Hybrid Sampling scheme from \cite{larsen2022sketching}, and some guarantees for \cref{alg:lai_symnmf}.

Comparing these two methods we believe that the \cref{alg:lai_symnmf} method is appropriate when a high quality LAI can be computed quickly and captures the underlying NMF signal.
This is highlighted by \Cref{eq:lai_2nd_ineq} which bounds the potential quality of the \cref{alg:lai_symnmf} in terms of the quality of the Low-rank Approximate Input (LAI).
In other cases \cref{alg:lvs_nmf} can be used, particularly in the case of sparse input data, as it does not compute an upfront LAI but instead samples new subsets of rows at every iteration.
Comparing these methods computationally, \cref{alg:lai_symnmf} costs roughly $O(qm^2l)$ for computing the LAI and then $O(mkl)$ per iteration, where $l = O(k)$ is the rank of the intermediate approximation, and \cref{alg:lvs_nmf} costs roughly $O(mks) = O(mk^2 \, \max(\log(\frac{k}{\delta}),\frac{1}{\delta\epsilon_r}))$ per iteration, where $s$ is the number of row samples.
The upfront cost of \cref{alg:lai_symnmf} is expensive but its per iteration cost is lower than that of \cref{alg:lvs_nmf}.
This is assuming one stays in the regime of $k\ll m$.
Additionally it is worth noting that both methods still require operations costing $O(mk^2)$ for the products $\M{H}^\Tra\M{H}$ and $\M{W}^\Tra\M{W}$ or computing a thin-QR of $\M{W}$ and $\M{H}$.
This again emphasizes the reliance of $k\ll m$ for computational efficiency.
It is also important to note that the computational complexity costs of each algorithm do not tell the whole story.
For example on large dense problems \Cref{alg:lvs_nmf}'s need to copy rows of the data matrix can be expensive as mentioned in \Cref{sec:experiments}.

%\section*{Acknowledgments}
%Information Release PNNL-SA-193926.
%\lipsum[2-3]

\appendix
\newpage

\newpage
\section{Derivation of Symmetrically Regularized HALS}
\label{app_efficient_HALS_update_derivation}
This section presents the derivation of the regularized HALS method for SymNMF and its efficient updating.
For this derivation $\M{X} = \M{X}^\Tra$ is symmetric throughout.
Define
\begin{equation*}
    \M{R}_{\neq i} = \M{A} 
    - \M{W}_{\neq i} \M{H}^\Tra_{\neq i} 
    = \M{A} - \sum_{j \neq i, j=1}^k \V{w}_j \V{h}_j^\Tra
\end{equation*}
Where $\M{W}_{\neq i}$ is $\M{W}$ without the $i$th column and similar for $\M{H}_{\neq i}$.
Then the rank one update equation can be posed as
\begin{align*}
    \|\begin{bmatrix}
        \M{H} \\
        \sqrt{\alpha}\M{I}_k
    \end{bmatrix}
    \M{W}^\Tra
    - 
    \begin{bmatrix}
        \M{X} \\
        \sqrt{\alpha}\M{H}^\Tra
    \end{bmatrix}\|_F^2
    \rightarrow 
    \|\begin{bmatrix}
        \V{h}_{i} \\
    \sqrt{\alpha}\V{e}_i
    \end{bmatrix}\V{w}_i^\Tra
    +
    \begin{bmatrix}
        \M{H}_{\neq i} \\
    \sqrt{\alpha}\M{I}_{\neq i}
    \end{bmatrix}
    \M{W}_{\neq i}^\Tra
    - 
    \begin{bmatrix}
        \M{A} \\
        \sqrt{\alpha}\M{H}^\Tra
    \end{bmatrix}\|_F^2 \\
    =  
    \|\begin{bmatrix}
        \V{h}_{i} \\
    \sqrt{\alpha}\V{e}_i
    \end{bmatrix}\V{w}_i^\Tra
    - 
    \begin{bmatrix}
        \M{R}_{\neq i}^\Tra \\
        \sqrt{\alpha}(\M{I}_{\neq i}\M{W}_{\neq i}^\Tra - \M{H}^\Tra)
    \end{bmatrix}\|_F^2.
\end{align*}
Next denote the matrix $\tilde{\M{R}}_{\neq i}$ and the vector $\V{u}_i$ as (Note that we have flipped sign here)
\begin{equation*}
    \tilde{\M{R}}_{\neq i} =     \begin{bmatrix}
        \M{R}_{\neq i}^\Tra \\
        \sqrt{\alpha}(\M{H}^\Tra - \M{I}_{\neq i}\M{W}_{\neq i}^\Tra)
    \end{bmatrix},
    \V{u}_i = \begin{bmatrix}
        \V{h}_{i} \\
    \sqrt{\alpha}\V{e}_i
    \end{bmatrix}
\end{equation*}
Now apply the standard HALS update
\begin{align*}
    \V{w}_i^\Tra = \Big[\frac{\tilde{\M{R}}_{\neq i}^\Tra\V{u}_i}{\|\V{u}_i\|^2_2} \Big]_+
    = \Big[\Big(\begin{bmatrix}
        \M{R}_{\neq i}^\Tra \\
        \sqrt{\alpha}(\M{H}^\Tra - \M{I}_{\neq i}\M{W}_{\neq i}^\Tra)
    \end{bmatrix}^\Tra
    \begin{bmatrix}
        \V{h}_{i} \\
    \sqrt{\alpha}\V{e}_i
    \end{bmatrix}\Big)/ (\|\V{h}_i\|_2^2 + \alpha)\Big]_+ \\
    = \frac{\M{R}_{\neq i}\V{h}_i + \alpha\V{h}_i}{\|\V{h}_i\|_2^2 + \alpha}
    = \frac{(\M{R}_{\neq i} + \alpha\M{I})\V{h}_i}{\|\V{h}_i\|_2^2 + \alpha}
\end{align*}
Finally the efficient updating method is derived as
\begin{align*}
    \frac{(\M{R}_{\neq i} + \alpha\M{I})\V{h}_i}{\|\V{h}_i\|_2^2 + \alpha} 
    = 
    \frac{(\M{A} - \M{W}_{\neq i}\M{H}_{\neq i}^\Tra + \alpha\M{I})\V{h}_i}{\|\V{h}_i\|_2^2 + \alpha} \\
    = 
    \frac{(\M{A} - \M{W}_{\neq i}\M{H}_{\neq i}^\Tra + \V{w}_{i}\V{h}_{i}^\Tra  - \V{w}_{i}\V{h}_{i}^\Tra + \alpha\M{I})\V{h}_i}{\|\V{h}_i\|_2^2 + \alpha}
    = \frac{(\M{A} - \M{W}\M{H}^\Tra  + \V{w}_{i}\V{h}_{i}^\Tra + \alpha\M{I})\V{h}_i}{\|\V{h}_i\|_2^2 + \alpha} \\
    = \frac{(\M{A} - \M{W}\M{H}^\Tra + \alpha\M{I})\V{h}_i + (\V{w}_{i}\V{h}_{i}^\Tra)\V{h}_i} {\|\V{h}_i\|_2^2 + \alpha} 
    = \frac{(\M{A} - \M{W}\M{H}^\Tra + \alpha\M{I})\V{h}_i + \|\V{h}_{i}\|_2^2\V{w}_{i}} {\|\V{h}_i\|_2^2 + \alpha} 
\end{align*}
The final update rule is written as
\begin{equation*}
    \V{w}_i \leftarrow \Big( \frac{(\M{A} - \M{W}\M{H}^\Tra + \alpha\M{I})\V{h}_i} {\|\V{h}_i\|_2^2 + \alpha} + \frac{\|\V{h}_{i}\|_2^2 }{\|\V{h}_i\|_2^2 + \alpha}\V{w}_{i}\Big)_+
\end{equation*}
for the columns of $\M{W}$ and this can be done similarly for the columns of $\M{H}$.
%%%%%%%%%%%%%%%%%%%%%%%%%%%%%%%%%%%%%%%%%%

\section{Additional Material for LAI-SymNMF}

\subsection{Comparison between LAI-NMF and Compressed-NMF} 
\label{app:lai_vs_comp}
\ref{alg:lai_symnmf} is similar to the randomized NMF algorithms proposed in \cite{rnd_HALS,Comp_NMF2016}.
The ``Randomized HALS'' method proposed in \cite{rnd_HALS} which destroys symmetry is inherently unsymmetric, as it compresses only one dimension of the matrix, and so we do not consider comparing against it.
In fact \ref{alg:lai_symnmf} could be viewed as a natural extension of this ``Randomized HALS'' method to randomized SymNMF.
The so called ``Compressed-NMF'' algorithm proposed in \cite{Comp_NMF2016} can be straightforwardly extended to the SymNMF problem.
We do so and compare against it later.
Generally, Compressed-NMF operates as follows: 1) compute two approximate orthonormal bases for $\M{X}$ and $\M{X}^\Tra$ as $\M{Q}_{\M{X}} = \Call{RRF}{\M{X},k,\rho,\omega}$ and $\M{Q}_{\M{X}^\Tra} = \Call{RRF}{\M{X}^\Tra,k,\rho,\omega}$ respectively and 
2) alternatingly solve the two problems
    \begin{align*}
        \min_{\M{H}\ge 0}\| \M{Q}_{\M{X}}^\Tra \big( \M{W}\M{H}^\Tra - \M{X} \big) \|_F \ \ \text{and} \ \
        \min_{\M{W}\ge 0}\| \M{Q}_{\M{X}^\Tra}^\Tra \big( \M{H}\M{W}^\Tra - \M{X}^\Tra \big) \|_F
    \end{align*}
The products $\M{Q}_{\M{X}}^\Tra\M{X}$ and $\M{Q}_{\M{X}^\Tra}^\Tra\M{X}^\Tra$ can be computed once but $\M{Q}_{\M{X}}^\Tra\M{H}$ and $\M{Q}_{\M{X}}^\Tra\M{W}$ need be recomputed each iteration.
In this way a smaller problem is solved at each iteration resulting in computational savings.
This method can be used for SymNMF by applying it to \cref{eq:surrogate_symnmf_obj}.
Additionally, in the symmetric case one need call the \ref{alg:QB_decomposition} only once.

We relate Compressed-NMF to LAI-NMF in the following way.
Consider a QB-decomposition computed using the \ref{alg:QB_decomposition} where $\M{Q}_{\M{X}} = \Call{RRF}{\M{X},k,\rho,\omega}$ and $\M{B}_{\M{X}} = \M{Q}_{\M{X}}^\Tra\M{X}$.
%Then $\M{Q}_{\M{X}}\M{B}_{\M{X}} = (\M{Q}_{\M{X}}\M{Q}_{\M{X}}^\Tra) \M{X} \approx \M{X}$ is the $\M{Q}\M{B}$ decomposition.
Using the $\M{Q}\M{B}$ decomposition for LAI-NMF in \cref{eq:alg_lai_nmf}, the update to $\M{H}$ and its Quadratic Program (QP) form are
\begin{equation*}
    \min_{ \M{H} \ge 0}\| \M{W}\M{H}^\Tra - \M{Q}_{\M{X}}\M{B}_{\M{X}} \|_F 
    \ \leftrightarrow \ \min_{\M{H}\ge 0} tr\Big(\frac{1}{2} \M{H}\M{W}^\Tra\M{W}\M{H}^\Tra - \M{H}\M{W}^\Tra\M{Q}_{\M{X}}\M{B}_{\M{X}} \Big).
\end{equation*}
Taking the QP formulation for the Compressed NMF update for $\M{H}$, \\ $\min_{\M{H}\ge 0}\| \M{Q}_{\M{X}}^\Tra \big( \M{W}\M{H}^\Tra - \M{X} \big) \|_F^2$, yields 
\begin{align*}
\min_{\M{H}\ge0} tr\Big(\frac{1}{2} \M{H}\M{W}^\Tra\M{Q}_{\M{X}}\M{Q}_{\M{X}}^\Tra\M{W}\M{H}^\Tra - \M{H}\M{W}^\Tra\M{Q}_{\M{X}}\M{Q}_{\M{X}}^\Tra\M{X} \Big) \\
= \min_{\M{H}\ge0} tr\Big(\frac{1}{2}\M{H}\M{W}^\Tra\M{Q}_{\M{X}}\M{Q}_{\M{X}}^\Tra\M{W}\M{H}^\Tra - \M{H}\M{W}^\Tra\M{Q}_{\M{X}}\M{B}_{\M{X}}\Big). 
\end{align*}
The term on the right, $\M{H}\M{W}^\Tra\M{Q}_{\M{X}}\M{B}_{\M{X}}$, is the same for both problems.
%This formulation reveals that while Compressed NMF is presented as a Least Squares sketching method it fits to a QB-decomposition.
The difference between the two methods comes from the fact that the Gram matrix in LAI-NMF is $\M{W}^\Tra\M{W}$ but the Gram in Compressed-NMF contains the projection $\M{Q}_{\M{X}}\M{Q}_{\M{X}}^\Tra$.
Empirically, for SymNMF, we find that these methods perform nearly identically.
%We suggest that this is because both methods are really fitting NMF to the approximation $\M{Q}_{\M{X}}\M{B}_{\M{X}} \approx \M{X}$.
%This fact is clear in the LAI method but is obscured in Compressed NMF as the authors attempt to follow the sketch and solve paradigm for OLS.
%However the matrix $\M{Q}_{\M{X}}$, computed using the \ref{alg:QB_decomposition}, is not used as a sketch for LS problems in prior work nor is there theoretical justification for it to be used as such.
From this perspective we argue the effect of $\M{Q}_{\M{X}}\M{Q}_{\M{X}}^\Tra$ in the Gram matrix is minimal.
Overall we propose that the LAI approach is easier to reason about, as it simply fits an NMF model of the low-rank approximation of $\M{X}$.
As previously stated, using the basis $\M{Q}_{\M{X}}$ as a sketch for NLS problems is not justified theoretically.

\subsection{Low-rank Approximate Input Projected Gauss-Newton with Conjugate Gradients for SymNMF}
\label{sec:pgncg_code}
This section contains the pseudo code for the PGNCG method with low-rank approximate inputs.
The original algorithm was developed for high performance distributed computing environments \cite{srinivas_SC20_GN}.
The Pseudocode is given in \cref{alg:lai_pgncg_symnmf}.
The only alterations to this pseudo code are lines 3 and 7.
This highlights how simple the idea of LAI-SymNMF is in practice.

\begin{varalgorithm}{LAI-PGNCG-SymNMF}[H]
 \caption{: SymNMF of a Low-Rank Approximation} 
 \label{alg:lai_pgncg_symnmf} 
\begin{algorithmic}[1]
\Require{\textbf{input}: a symmetric matrix $\M{X} \in\mathbb{R}^{m \times m}$, target rank $k$, column oversampling parameter $\rho$, power iteration exponent $q$, and max number of CG iterations $s_{max}$}
\Ensure{$\{\M{H}\}$ as the factors for an approximate rank-$k$ SymNMF of $\M{X}$.}
\Function{$[\M{H}] =$ LAI-SymNMF}{$\M{X},k,\rho,\omega$}
 \State{$l := k + \rho$}
 \State{[$\M{U}, \M{\Lambda}] := $ Apx-EVD($\M{X},k,\rho,q$)}, \% Obtain approximate basis for range of $\M{X}$, $\M{U}\in\mathbb{R}^{m \times l}$ and $\M{T}\in\mathbb{R}^{l \times l}$
\State{Initialize $\M{H}\in \mathbb{R}^{n \times k}$}
\While{Convergence Crit. Not Met}
    \State{$\M{Z} = \M{0}_{n\times k}$}
    \State{$\M{R} = -2(\M{U}\M{\Lambda}(\M{U}^\Tra\M{H}) - \M{H}(\M{H}^\Tra\M{H}))$} \Comment{This replaces $\M{X}\M{H} - \M{H}(\M{H}^\Tra\M{H})$}
    \State{$\M{P} = \M{R}$}
    \State{$e_{old} = \|\M{R}\|^2_F$}
    \For{$s = 1:s_{max}$}
        \State{$\M{Y} = 2(\M{P}\M{H}^\Tra\M{H} + \M{H}\M{P}^\Tra\M{H})$}
        \State{$\alpha = e_{old}/(\sum_{i,j}\M{P}_{ij}\M{Y}_{ij})$}
        \State{$\M{Z} = \M{Z} + \alpha\M{P}$}
        \State{$\M{R} = \M{R} - \alpha\M{Y}$}
        \State{$e_{new} = \|\M{R}\|^2_F$}
        \State{$e_{old} = e_{new}$}
    \EndFor
    \State{$\M{H} = \max(\M{H} - \M{Z},0)$}
 \EndWhile
 \EndFunction
\end{algorithmic}
\end{varalgorithm}

%%%%%%%%%%%%%%%%%%%%%%%%%%%%%%%%%%%%%%%%%%%%%%%%%%%%%%%%%%%%%%
%%%%%%%%%%%%%%%%%%%%%%%%%%%%%%%%%%%%%%%%%%
\section{Stopping Criteria}
\label{app:stopping_crit}
Having a stopping criteria is often useful for NMF algorithms. 
In this work we use two measures the Normalized Residual and the Projected Gradient to measure convergence. 
\subsection{Residual Checks}
The NMF algorithms often require computation of the residual $\|\M{X} - \M{W}\M{H}^\Tra\|_F$ or the normalized residual 
\begin{equation}
    \frac{\|\M{X} -\M{W}\M{H}^\Tra \|_F}{\|\M{X}\|_F}.
\end{equation}
However, checking the residual working with the full matrix $\M{X}$ can be computationally expensive. 
Since checking the residual requires working with the full data matrix $\M{X}$ it can potentially dominate the run time of a randomized algorithm.
Therefore, it is important to have an appropriate method for estimating the residual at each iteration \cite{lev_scores4_CP,casey_rndCP}.

\paragraph{Residual Computation for LAI-NMF}
For \ref{alg:lai_symnmf} the issue is easily remedied as we can simply check the residual against the factored form of $\M{X} \approx \M{U}\M{V}^\Tra$ as
\begin{equation}
    \frac{\|\M{U}\M{V}^\Tra -\M{W}\M{H}^\Tra \|_F}{\|\M{U}\M{V}^\Tra\|_F}.
\end{equation}
This allows for rapid evaluation of an approximate residual using techniques similar to that used for NMF.
Our experiments show this is often a reasonable approach to use.

\subsection{Fast Residual Evaluation in NMF}
The NMF residual is given by $r^2 = \|\M{X} - \M{W}\M{H}^\Tra\|_F^2$.
This quantity can be computed cheaply by reusing already computed quantities from the previous update.
To see this observe that
\begin{align*}
    \|\M{X} - \M{W}\M{H}^\Tra\|_F^2 = tr\Big( (\M{X} - \M{W}\M{H}^\Tra)^\Tra (\M{X} - \M{W}\M{H}^\Tra) \Big) \\
    = tr(\M{X}^\Tra\M{X}) + tr(\M{\M{H}\M{W}^\Tra\M{W}\M{H}^\Tra}) - 2tr(\M{W}^\Tra\M{X}\M{H}).
\end{align*}
The $tr(\M{X}^\Tra\M{X})$ need only be computed one time, computing $tr(\M{\M{H}\M{W}^\Tra\M{W}\M{H}^\Tra})$ is cheap as long as $k << \min(m,n)$ and can reuse the most recent Gram matrix (either $\M{W}^\Tra\M{W}$ or $\M{H}^\Tra\M{H}$), and $2tr(\M{W}^\Tra\M{X}\M{H})$ can be computed cheaply by reusing the more recently computed product with $\M{X}$ depending on the update order of $\M{W}$ and $\M{H}$.
This method can easily be used with \ref{alg:lai_symnmf} to compute $\|\M{U}_X\M{\Lambda}_X\M{U}_X^\Tra -  \M{H}\M{H}^\Tra\|_F^2$ at each iteration.
Therefore we are able to evaluate the residual for almost free at each iteration of the algorithm.
This is important for determining when to stop iterating.

\subsection{Projected Gradients}
\label{sec:prj_grads}
The norm of the projected gradient is often used as a stopping criteria for NMF algorithms in place of the residual.
The idea is that when the NMF objective is at a stationary point, according to the Karush-Kuhn-Tucker (KKT) conditions, the projected gradients with respect to $\M{W}$ and $\M{H}$ will be 0.
Therefore we assume that if the projected gradients are small then we are close to a stationary point.

The projected gradient norm is defined as 
\begin{equation}
    \label{eq:prj_grad_4nmf}
    \sqrt{\|\nabla f_W\|_F^2 + \|\nabla f_H\|_F^2}
\end{equation}
where the partial gradients are
\begin{align}
    \nabla f_W = 2(\M{W}\M{H}^\Tra\M{H} - \M{A}\M{H}) \\
    \nabla f_H = 2(\M{H}\M{W}^\Tra\M{W} - \M{A}^\Tra\M{W})
\end{align}
and the projected gradient is
\begin{align}
    (\nabla^p f_W)_{ij} = 
    \begin{cases}
    (\nabla f_W)_{ij} \ \text{if} \ (\nabla f_W)_{ij}  < 0 \ \text{or} \ \M{W}_{ij} > 0 \\
    0
    \end{cases}
\end{align}
similar for the partial gradient with respect to $\M{H}$.
One issue with using the projected gradient as a comparison measure between algorithms is that different diagonal scalings of $\M{W}$ and $\M{H}$ can give different projected gradient values \cite{BCD_park,Jinug_BPP}.
However this is not an issue with the SymNMF objective. 
%\kh{Need to make sure this is accurate}.
The gradient of \cref{eq:symnmf_obj} is given by
\begin{equation}
    \nabla f_H = 4(\M{H}\M{H}^\Tra - \M{X})\M{H}
\end{equation}

%%%%%%%%%%%%%%%%%%%%%%%%%%%%%%%%%%%%
% AUTO QB
%%%%%%%%%%%%%%%%%%%%%%%%%%%%%%%%%%%%
\section{Adaptive Randomized Range Finder}
Using the standard ``trick'' for efficiently checking the residual we can derive the following formula which we use in \cref{alg:adaRRF} : 
\begin{align*}
    \|\M{Q}\M{B} - \M{X}\|_F^2 = \|\M{X}\|_F^2 - 2tr(\M{Q}^\Tra\M{X}\M{B}^\Tra) + tr(\M{B}^\Tra\M{Q}^\Tra\M{Q}\M{B}) \\
    = \|\M{X}\|_F^2 - 2tr(\M{B}\M{B}^\Tra) + tr(\M{B}^\Tra\M{B}) = \|\M{X}\|_F^2 - tr(\M{B}\M{B}^\Tra)
\end{align*}
To the check the residual of our LAI we need only compute the matrix $\M{B}$.
This $\M{B}$ can then be used in the next power iteration if needed.
\begin{varalgorithm}{Ada-RRF}
 \caption{Adaptive Randomized Range Finder}
 \label{alg:adaRRF} 
\begin{algorithmic}[1]
\Require{\textbf{input}: data matrix $\M{X} \in\mathbb{R}^{m \times n}$, target rank $r$, oversampling parameter $\rho$, maximum number of power iterations $q_{max}$}
\Ensure{$\M{Q}_Y \in \mathbb{R}^{m \times l}$}
\Function{$[\M{\M{U},\M{\Lambda}]} =$ Apx-EVD}{$X,r,\rho,\omega$}
 \State{$l := r + \rho$}
 \State{Draw a Gaussian Random matrix $\M{\Omega} \in \mathbb{R}^{n \times l}$}
 \State{Compute $\M{Q} = qr(\M{X}\M{\Omega}) \in \mathbb{R}^{m \times l}$}
 \While{$j \le q_{max}$}
 \State{$\M{B}^\Tra = \M{X}^\Tra\M{Q}\in \mathbb{R}^{n \times l}$}
 \State{Evaluate $\|\M{Q}\M{B} - \M{X}\|_F^2 = \|\M{X}\|_F^2 - tr(\M{B}\M{B}^\Tra)$}
 \State{$\M{Q} = qr(\M{B}^\Tra) \in \mathbb{R}^{n \times l}$}
 \If{Stopping Criteria satisfied}
    \State{break}
 \EndIf
 \State{$\M{Y} = \M{X}\M{Q}\in \mathbb{R}^{m \times l}$}
 \State{$\M{Q} = qr(\M{Y},0)$}
 \EndWhile
 \State{Return $\M{Q}$ }
 \EndFunction
\end{algorithmic}
\end{varalgorithm}
%%%%%%%%%%%%%%%%%%%%%%%%%%%%%%%%%%%%
% AUTO QB
%%%%%%%%%%%%%%%%%%%%%%%%%%%%%%%%%%%%
\section{Update() Function}
\label{app:updateF}
In this section we thoroughly explain the Update($\M{G},\M{Y}$) function we use to simplify our pseudocode.
This function takes in two matrices $\M{G}$ which is $k\times k$ and $\M{Y}$ which is $k \times m$ or $k \times n$.
The Framework for Alternating Updating NMF (FAUN) was proposed for the design of a massively parallel NMF code \cite{mpi_faun_ramki}.
In the FAUN the matrix products $\M{X}^\Tra\M{W}$, $\M{X}\M{H}$, $\M{W}^\Tra\M{W}$, and $\M{H}^\Tra\M{H}$ are computed and used to perform updates.
Many of the most effective NMF update rules require the computation of these 4 matrices as they appear in the gradient of the NMF NLS subproblems 
\begin{equation}
    \label{eq:updateW}
\min_{\M{W}\ge0}\| \M{X} - \M{W}\M{H}^\Tra \|_F
\end{equation}
and
\begin{equation}
    \label{eq:updateH}
\min_{\M{H}\ge0}\| \M{X} - \M{W}\M{H}^\Tra \|_F.
\end{equation}

Many methods can be implemented using the FAUN.
We now briefly discuss a few of the more popular methods.
\paragraph{Multiplicative Updates} is one of the most popular methods for performing NMF updates. Proposed in \cite{mu_4nmf}, it is guaranteed to non increase the objective function and uses the rules
\begin{align*}
    \M{W}_{ij} \leftarrow \M{W}_{ij} \frac{(\M{X}\M{H})_{ij}}{(\M{W}\M{H}^\Tra\M{H})_{ij}} \ \ \text{and} \ \
    \M{H}_{ij} \leftarrow \M{H}_{ij} \frac{(\M{X}^\Tra\M{W})_{ij}}{(\M{H}\M{W}^\Tra\M{W})_{ij}}.
\end{align*}

\paragraph{Hierarchical Least Squares} (HALS) uses the following update rules
\begin{align*}
    \V{w}_i \leftarrow \Big[ \V{w}_i + \frac{(\M{X}\M{H})_i - (\M{W}\M{H}^\Tra\M{H})_i}{(\M{H}^\Tra\M{H})_{ii}}\Big]_+ \ \ \text{and} \ \
    \V{h}_i \leftarrow \Big[ \V{h}_i + \frac{(\M{X}^\Tra\M{W})_i - (\M{H}\M{W}^\Tra\M{W})_i}{(\M{W}^\Tra\M{W})_{ii}}\Big]_+.
\end{align*}
All the columns of $\M{W}$ and $\M{H}$ are updated in sequence.
This method is popular for its good convergence properties and its simplicity to implement. 

\paragraph{Alternating Nonnegative Least Squares} ANLS updates the full $\M{W}$ and $\M{H}$ matrices in an alternating fashion by computing the optimal solutions for the NLS problems Eqns. (\ref{eq:updateH}) and (\ref{eq:updateW}).
These NLS problems are equivalent to solving the following $m+n$ quadratic programs (QPs)
\begin{align*}
\min_{\hat{\V{w}}_j \ge 0} \Big(\hat{\V{w}}_j^\Tra\M{H}^\Tra\M{H}\hat{\V{w}}_j - \hat{\V{w}}_j^\Tra\M{H}^\Tra\hat{\V{x}}_j \Big) \ \text{for} \ j = 1:m \\
\min_{\hat{\V{h}}_t \ge0} \Big(\hat{\V{h}}_t^\Tra\M{W}^\Tra\M{W}\hat{\V{h}}_t - \hat{\V{h}}_t^\Tra\M{W}^\Tra\V{x}_t \Big) \ \text{for} \ t = 1:n
\end{align*}
for every row vector $\hat{\V{w}}_j$ and $\hat{\V{h}}_t$ of $\M{W}$ and $\M{H}$, respectively.

It is simple to see that all of these rules rely on the aforementioned four matrix products.
Our proposed algorithms \ref{alg:lai_symnmf} and \ref{alg:lvs_nmf} can use any of these updates rules, and more e.g. projected gradient methods for solving NLS problems \cite{proj_grad4nmf}.

%%%%%%%%%%%%%%%%%%%%%%%%%%%%%%%%%%%%
% Misc. Analysis of methods
%%%%%%%%%%%%%%%%%%%%%%%%%%%%%%%%%%%%
\section{Theorems for Randomized Numerical Linear Algebra}
\label{app:rndNLA_thms}
This section contains a number of theorems that we use from the Randomized NLA literature.

\subsection{Structural Condition Theorems}
\label{app:sc_theorems}
We include statements of the two Structural Conditions taken from the work by Larsen and Kolda \cite{larsen2022sketching}.
Prior versions and similar statements of this result are originally from \cite{sketching4NLA_woodruff,Ismail_row_samp}
These theorems are included for completeness and reference.
\begin{theorem}
\label{thm:su_bound} 
    Given $\M{A} \in \mathbb{R}^{m\times k}$ consider its SVD $\M{U}\M{\Sigma}\M{V}^\Tra$ and its row leverage scores $l_i(\M{A})$ for each row $i \in [m]$, where $[m]$ denotes the set of integers form 1 to $m$.
    Let $\hat{l}_i(\M{A})$ be an overestimate of the leverage score such that for some $\beta \le 1$ it holds that $p_i(\hat{l}_i(\M{A}))\ge\beta*p_i(l_i(\M{A}))$ for all $i \in [m]$, where $p_i(l_i(\M{A}))$ denotes the probability corresponding to the $i$th leverage score.
    Construct a row sampling and rescaling matrix $\M{S} \in \mathbb{R}^{s \times m}$ via importance sampling according to the leverage score overestimates $\hat{l}_i(\M{A})$.
    If $s > \frac{144k\log(2k/\delta)}{\beta\epsilon^2}$ then the following equation holds with at least probability $1-\delta$
    \begin{equation*}
        1-\epsilon \le \sigma^2_i(\M{S}\M{U}) \le 1+ \epsilon
    \end{equation*}
    for all $i \in [m]$ and $\epsilon,\delta \in (0,1)$.
\end{theorem}
This theorem tells us that all the singular values of $\M{S}\M{U}$ are close to 1 if we take enough samples $s$.
Note that this implies
\begin{equation}
\label{eq:i_ussu}
     \|\M{I} - \M{U}^\Tra\M{S}^\Tra\M{S}\M{U}\|_2 < \epsilon.
\end{equation}

\subsection{Randomized Matrix Multiply}
For reference we state a result originally from \cite{apx_matmul_drineas}.
\begin{theorem}
\label{thm:rnd_matmul}
    Consider two matrices $\M{A} \in \mathbb{R}^{k\times m}$ and $\M{B} \in \mathbb{R}^{k\times n}$ and their product $\M{A}^\Tra\M{B} = \M{C}$.
    Construct a sampling matrix $\M{S}$ with $s$ rows that are chosen according to the probability distribution $\V{p} \in [0,1]^n$ with a $\beta > 0$ such that $\V{p}_i \ge \beta\|\M{A}[i,:]\|_2^2/\|\M{A}\|_F^2$ for all $i \in [m]$.
    Define $\V{t}$ as vector of length $s$ of sampled indices such that $\V{t}_i$ is the index of the $i$th sampled row.
    Consider the approximate matrix product
    \begin{equation*}
        \frac{1}{s}\sum_{i}^s \M{A}[\V{t}_i,:]^\Tra\M{B}[\V{t}_i,:] = \M{A}^\Tra\M{S}^\Tra\M{S}\M{B}
    \end{equation*}
    Then the approximate matrix product satisfies
    \begin{equation*}
        \E \Big[\|\M{A}^\Tra\M{B} - \M{A}^\Tra\M{S}^\Tra\M{S}\M{B}\|_F^2 \Big] \le \frac{1}{\beta s} \|\M{A}\|_F^2 \|\M{B}\|_F^2
    \end{equation*}
\end{theorem}
Applying Theorem (\ref{thm:rnd_matmul}) in conjunction with Markov's Inequality we obtain the following Lemma
\begin{lemma}
\label{thm:fro_bound}
    Taking $s \ge \frac{2\|\M{A}\|_F^2}{\beta\delta\epsilon_r}$ samples where $\delta,\epsilon_r \in (0,1)$ results in 
    \begin{equation*}
        \|\M{A}^\Tra\M{B} - \M{A}^\Tra\M{S}^\Tra\M{S}\M{B}\|_F^2  \ge \frac{\epsilon_r\|\M{B}\|}{2}
    \end{equation*}
    with probability $\delta$.
\end{lemma}
\begin{proof}
    Applying Markov's Inequality we have 
    \begin{align*}
        Pr\Big[ \|\M{A}^\Tra\M{B} - \M{A}^\Tra\M{S}^\Tra\M{S}\M{B}\|_F^2 \ge \frac{\epsilon_r\|\M{B}\|_F^2}{2}  \Big] \le \frac{2\E[\|\M{A}^\Tra\M{B} - \M{A}^\Tra\M{S}^\Tra\M{S}\M{B}\|_F^2]}{\epsilon_r\|\M{B}\|_F^2} \\
        \le \frac{\|\M{A}\|_F^2 \|\M{B}\|_F^2}{\beta s \epsilon_r\|\M{B}\|_F^2}
        = \frac{\|\M{A}\|_F^2}{\beta s \epsilon_r}
    \end{align*}
    For the bound to hold with probability $\delta$ we must have $s \ge \frac{2\|\M{A}\|_F^2}{\beta\delta\epsilon_r}$.
\end{proof}

\subsection{Matrix Chernoff Bounds}
This is a Matrix Chernoff Bound taken from \cite{sketching4NLA_woodruff} which is used to show the validity of the Hybrid Sampling leverage scores matrix.
\begin{theorem}
\label{thm:matrix_chernoff}
Matrix Chernoff Bound:
Let $\M{X}_1,\cdots,\M{X}_s$ be independent copies of a symmetric random matrix $\M{X} \in \mathbb{R}^{k \times k}$ with $\E[\M{X}] = 0$, $\|\M{X}\|_2 \le \gamma$, and $\|\E[\M{X}^\Tra\M{X}]\|_2 \le \nu^2$.
Let $\M{W} = \frac{1}{s}\sum_{i=1}^s\M{X}_i$.
Then for any $\epsilon > 0$ we have
\begin{equation}
\label{eq:mcb_w_bound}
    Pr[\|\M{W}\|_2 > \epsilon ] \le 2k\exp(-s\epsilon^2/(2\nu^2+2\gamma \epsilon/3))
\end{equation}
\end{theorem}

\section{Additional Experimental Data}
This section contains some additional data and experiments related to the experiments on the WoS and OAG data sets.
\subsection{World of Science Data Set}
\label{sec:additional_wos_exps}
This section contains some additional experiments on the World of Science data set. The first set of experiments are concerned with varying $\rho$ or the column over sampling parameter in the \ref{alg:QB_decomposition}.
The observation here is that increasing $\rho$ does not seem to have a beneficial impact on clustering quality or final residual.
\cref{fig:WoS_NR_plts} shows the runs for the various algorithms with $\rho=40$ and $\rho=80$.
\cref{tab:WoS_exp_additional_p40} and  \cref{tab:WoS_exp_additional_p80} give various metrics associated with these runs.
Each algorithm was run 10 times.
``Iters'' is average number of iterations, ``Time'' is the average time in seconds, ``Avg. Min-Res'' is the average minimum residual achieved, ``Min-Res'' is the overall minimum residual achieved, and ``Mean-ARI'' is the mean of the ARI scores for each algorithm.
\begin{figure}
     \centering
     \begin{subfigure}[b]{0.475\textwidth}
         \centering
         \includegraphics[width=\textwidth]{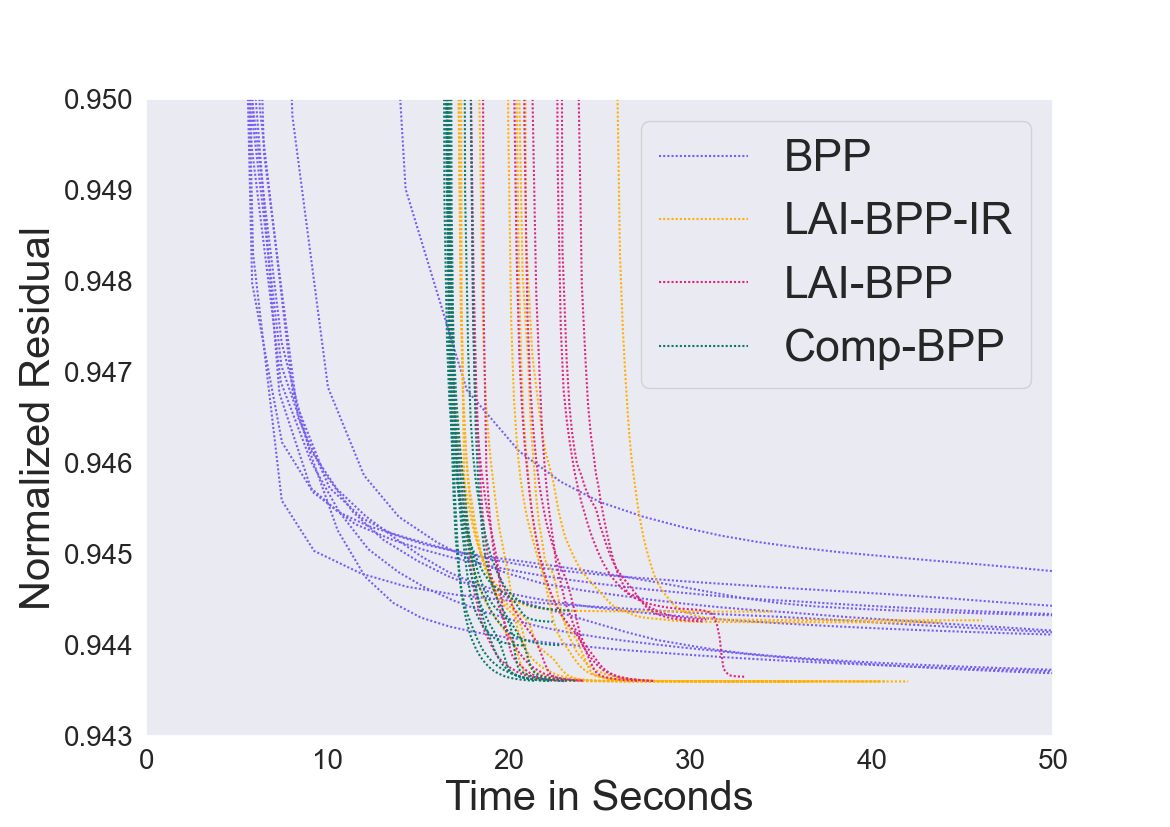}
         \caption{BPP, p=40}
         \label{fig:WoS_nr_BPP_p30}
     \end{subfigure}
     \hfill     
     \begin{subfigure}[b]{0.475\textwidth}
         \centering
         \includegraphics[width=\textwidth]{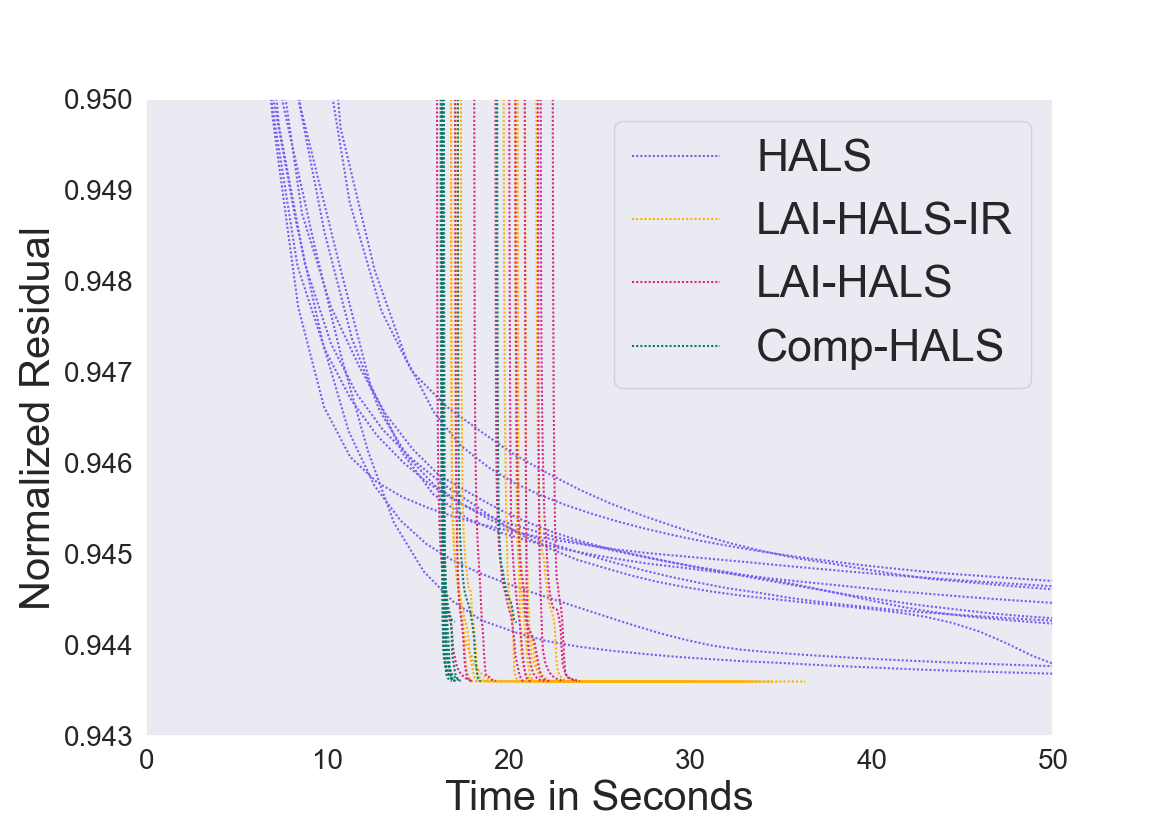}
         \caption{HALS, p=40}
         \label{fig:WoS_nr_HALS_p30}
     \end{subfigure}
     \begin{subfigure}[b]{0.475\textwidth}
         \centering
         \includegraphics[width=\textwidth]{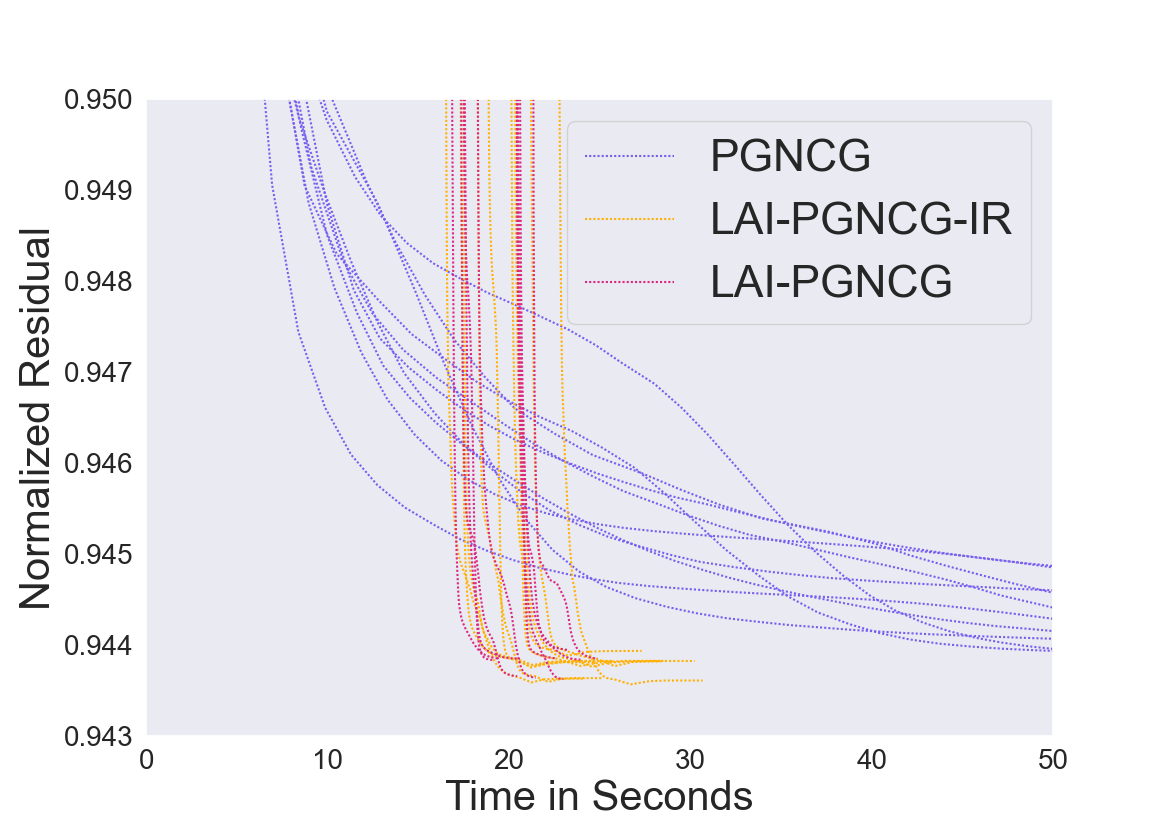}
         \caption{GNCG, p=40}
         \label{fig:WoS_nr_GNCG_p30}
     \end{subfigure}  
     \hfill
     \begin{subfigure}[b]{0.475\textwidth}
         \centering
         \includegraphics[width=\textwidth]{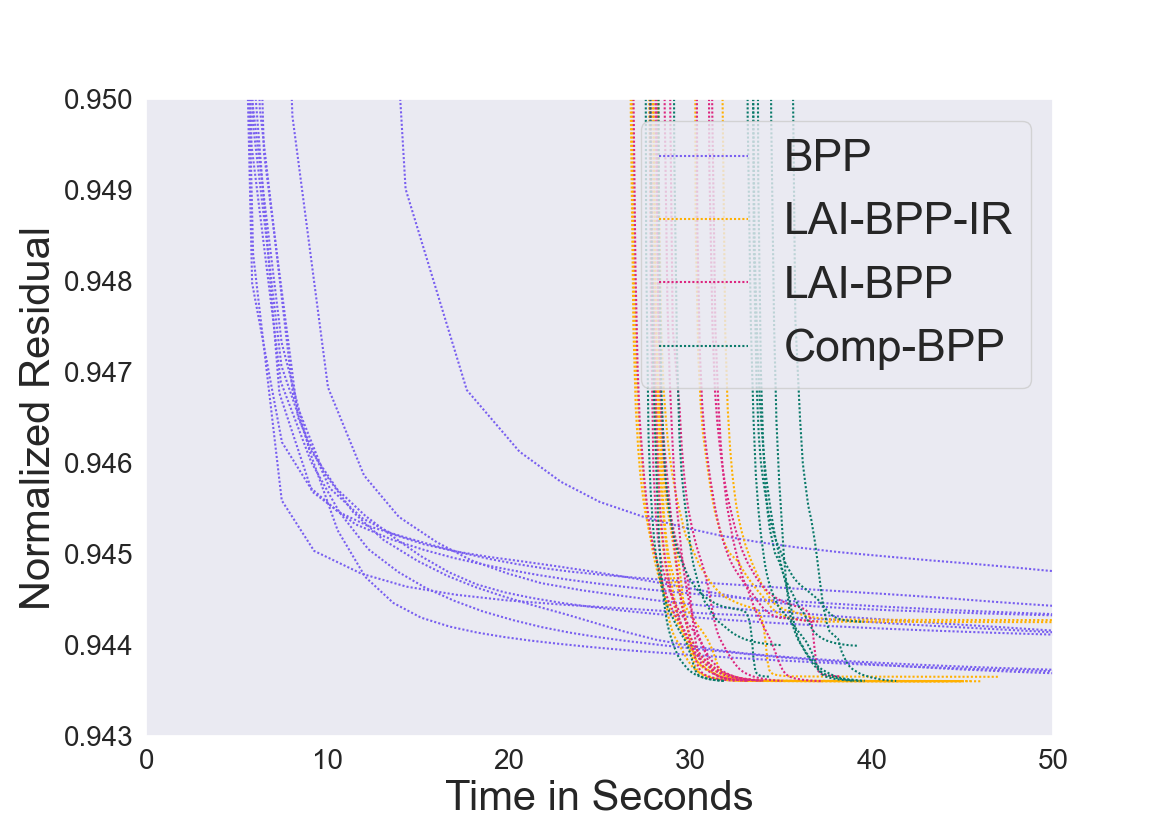}
         \caption{BPP, p=80}
         \label{fig:fig:WoS_nr_bpp_p40}
     \end{subfigure}
          \begin{subfigure}[b]{0.475\textwidth}
         \centering
         \includegraphics[width=\textwidth]{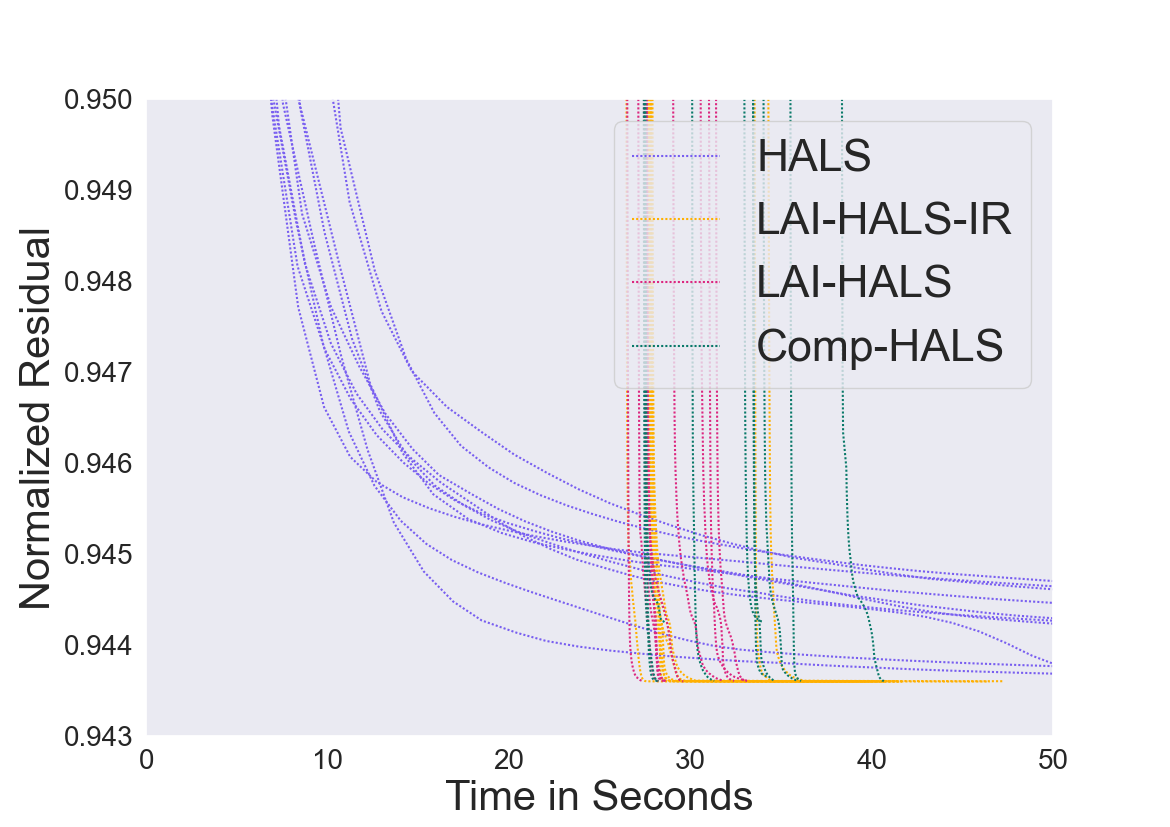}
         \caption{HALS, p=80}
         \label{fig:WoS_nr_HALS_p40}
     \end{subfigure}  
     \hfill
     \begin{subfigure}[b]{0.475\textwidth}
         \centering
         \includegraphics[width=\textwidth]{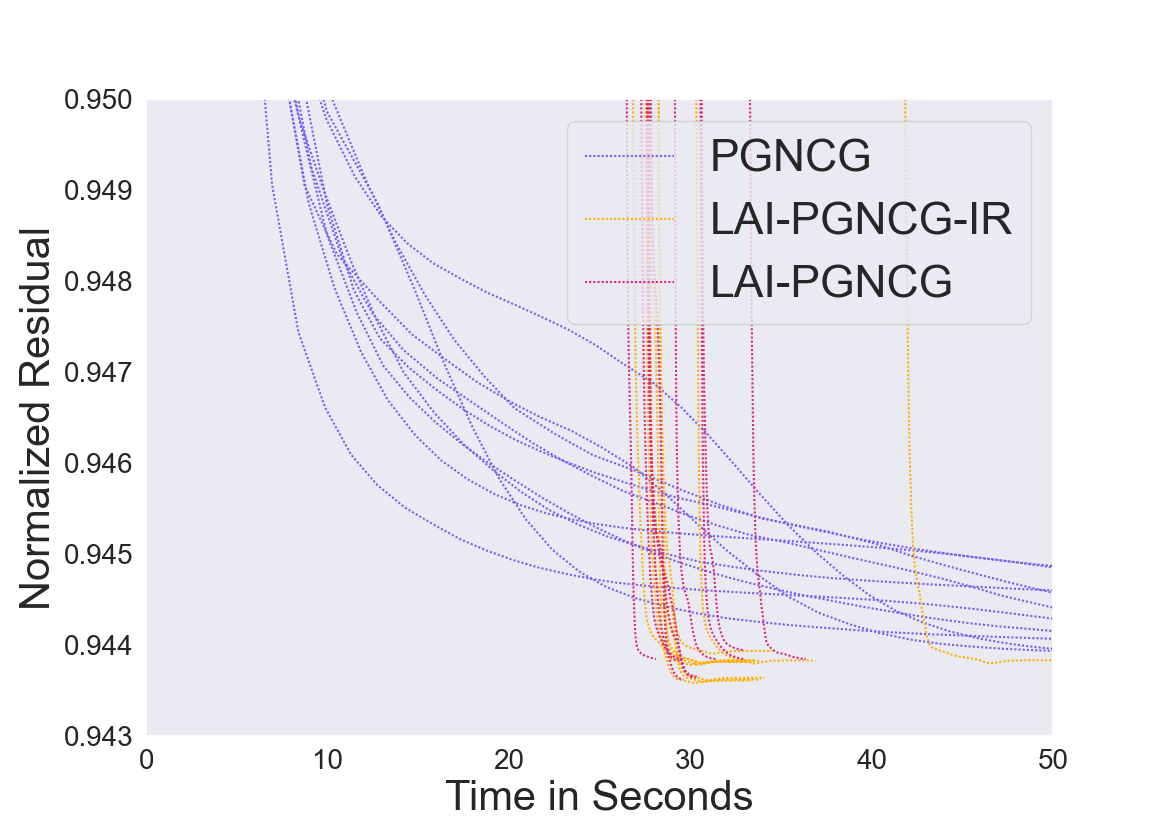}
         \caption{GNCG, p=80}
         \label{fig:WoS_nr_gncg_p40}
     \end{subfigure}
     \hfill
    \caption{Normalized Residual error value for various SymNMF Algorithms on the WoS data set using an EDVW Hypergraph Representation.}
    \label{fig:WoS_NR_plts_vp}
\end{figure}

\begin{table}
\resizebox{\columnwidth}{!}{
\begin{tabular}{l||c|c|c|c|c|}%
\centering
\bfseries Alg. & \bfseries Iters & \bfseries Time & \bfseries Avg. Min-Res & \bfseries Min-Res & \bfseries Mean-ARI % specify table head
\begin{filecontents*}{CSV_Data/WoS_exp31_T.csv}
Alg,AvgIters,Time,AvgMinRes,MinRes,MeanARI,MeanNcut
BPP,38.1,66.95,0.9439,0.9436,0.3224,--
LAI-BPP,70.7,27.25,0.9437,0.9436,0.3234,--
LAI-BPP-IR,76.7,39.366,0.9438,0.9436,0.3242,--
LAI-HALS,85.2,21.214,0.9436,0.9436,0.3292,--
LAI-HALS-IR,89.2,32.429,0.9436,0.9436,0.3284,--
HALS,58.1,92.461,0.9437,0.9436,0.3201,--
GNCG,50.8,80.311,0.944,0.9437,0.3063,--
LAI-GNCG,90.4,21.877,0.9438,0.9436,0.3116,--
Comp-BPP,69.5,22.472,0.9438,0.9436,0.3163,--
LAI-GNCG-IR,97.0,27.053,0.9437,0.9436,0.3107,--
Comp-HALS,75.7,17.533,0.9437,0.9436,0.3254,--
\end{filecontents*}
\csvreader[head to column names]{CSV_Data/WoS_exp31_T.csv}{}% use head of csv as column names
{\\\hline \Alg & \AvgIters & \Time & \AvgMinRes & \MinRes & \MeanARI}% specify your coloumns here
\end{tabular}
}
\caption{Data Table for the WoS data set. Each algorithm was run 10 times. For LIA methods auto-$q$ and $p=40$. The columns record average number of iterations until the stopping criteria is met, average time in seconds, average minimum residual, and minimum residual achieved over all runs.}
\label{tab:WoS_exp_additional_p40}
\end{table}

\begin{table}
\resizebox{\columnwidth}{!}{
\begin{tabular}{l||c|c|c|c|c|}%
\centering
\bfseries Alg. & \bfseries Iters & \bfseries Time & \bfseries Avg. Min-Res & \bfseries Min-Res & \bfseries Mean-ARI % specify table head
\begin{filecontents*}{CSV_Data/WoS_exp32_T.csv}
Alg,AvgIters,Time,AvgMinRes,MinRes,MeanARI,MeanNcut
BPP,38.1,66.95,0.9439,0.9436,0.3224,--
LAI-BPP,66.9,34.824,0.9437,0.9436,0.3271,--
LAI-BPP-IR,72.3,46.136,0.9437,0.9436,0.3225,--
LAI-HALS,82.8,30.167,0.9436,0.9436,0.3298,--
LAI-HALS-IR,87.0,41.765,0.9436,0.9436,0.329,--
HALS,58.1,92.461,0.9437,0.9436,0.3201,--
GNCG,50.8,80.311,0.944,0.9437,0.3063,--
LAI-GNCG,85.5,31.1,0.9438,0.9436,0.3111,--
BComp-PP,61.8,37.139,0.9438,0.9436,0.3149,--
GNCG-IR,89.4,35.702,0.9437,0.9436,0.3117,--
Comp-HALS,74.5,33.205,0.9437,0.9436,0.3257,--
\end{filecontents*}
\csvreader[head to column names]{CSV_Data/WoS_exp32_T.csv}{}% use head of csv as column names
{\\\hline \Alg & \AvgIters & \Time & \AvgMinRes & \MinRes & \MeanARI}% specify your coloumns here
\end{tabular}
}
\caption{Data Table for the WoS data set. Each algorithm was run 10 times. For LIA methods auto-$q$ and $p=80$. The columns record average number of iterations until the stopping criteria is met, average time in seconds, average minimum residual, and minimum residual achieved over all runs.}
\label{tab:WoS_exp_additional_p80}
\end{table}

The second set of experiments are run with $q=2$ and without the use of \cref{alg:adaRRF}.
\cref{tab:WoS_exp_additional_q2} gives data for this run.
One can observe that method without IR do not achieve high quality residual or ARI scores.
While IR can help fix these issues it does so at a higher computational cost.
The convergence plots for these experiments are in \cref{fig:WoS25_residual_GNCG,fig:WoS25_residual_BPP,fig:WoS25_residual_HALS}.
Contrasting these withe the results in \Cref{tab:WoS_symnmf_exp} we conclude that using \ref{alg:adaRRF} to obtain an appropriate $q$ is generally more efficient in terms of reducing residual and overall run time that using a static choice of $q=2$.

\begin{figure}
     \centering
     \begin{subfigure}[b]{0.32\textwidth}
         \centering
         \includegraphics[width=\textwidth]{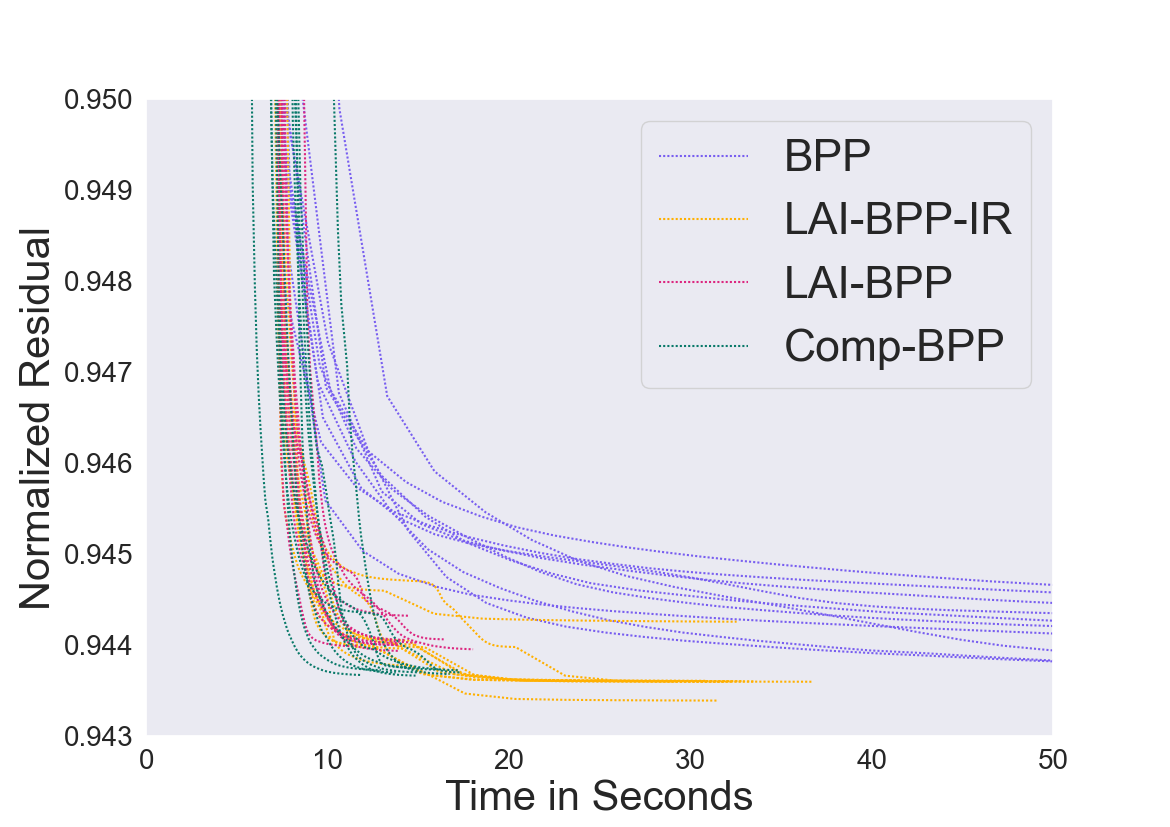}
         \caption{BPP}
         \label{fig:WoS25_residual_BPP}
     \end{subfigure}
     %\hfill     
     \begin{subfigure}[b]{0.32\textwidth}
         \centering
         \includegraphics[width=\textwidth]{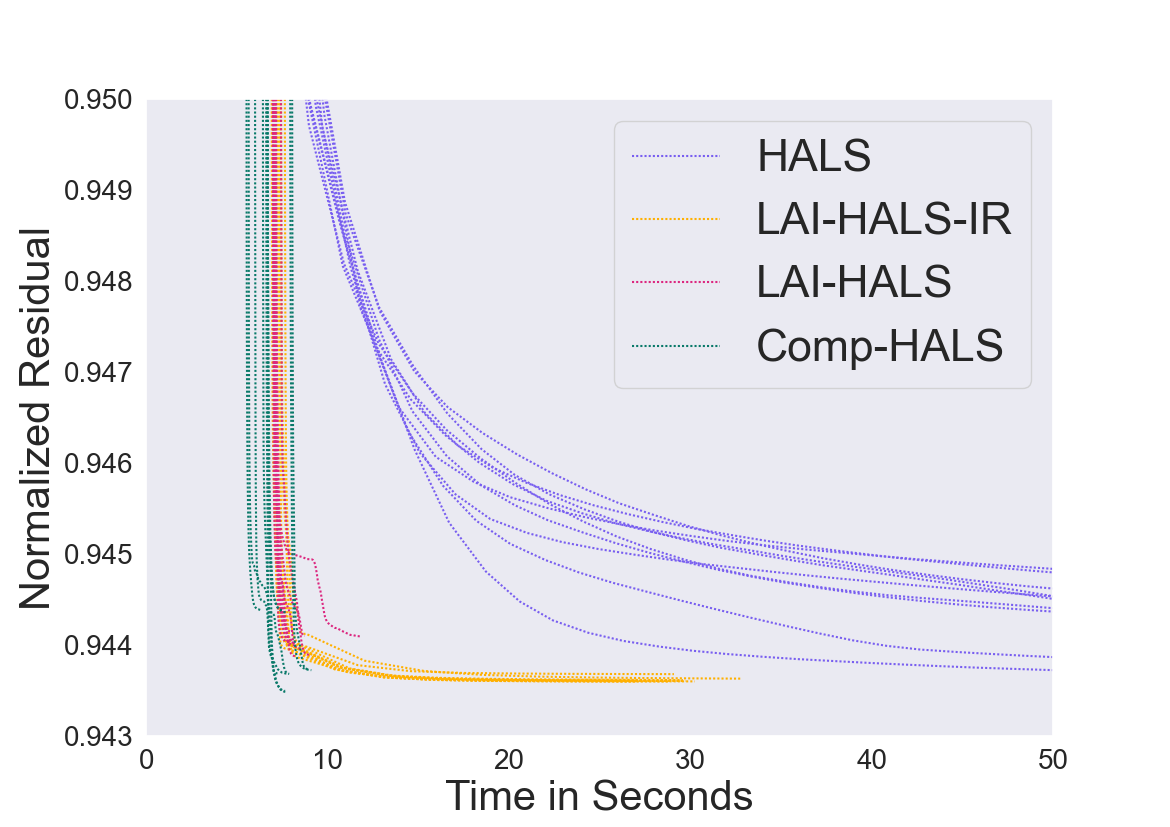}
         \caption{HALS}
         \label{fig:WoS25_residual_HALS}
     \end{subfigure}
     \begin{subfigure}[b]{0.32\textwidth}
         \centering
         \includegraphics[width=\textwidth]{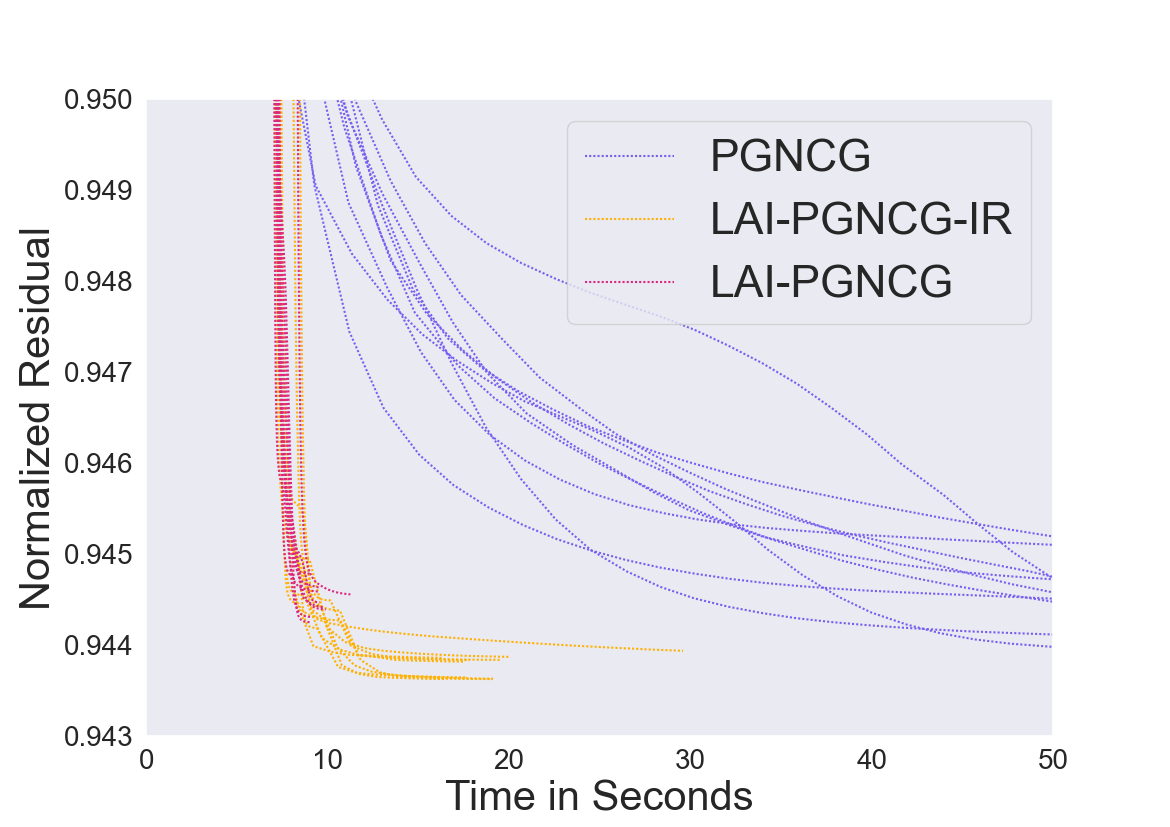}
         \caption{PGNCG}
         \label{fig:WoS25_residual_GNCG}
     \end{subfigure}  
\caption{Normalized Residual Norm Plots for the WoS data set. Three different update rules are shown: BPP, HALS, and PGNCG. These experiments use $q=2$ and do not make use of \cref{alg:adaRRF}.}
\end{figure}

\begin{table}
\resizebox{\columnwidth}{!}{
\begin{tabular}{l||c|c|c|c|c|}%
\centering
\bfseries Alg. & \bfseries Iters & \bfseries Time & \bfseries Avg. Min-Res & \bfseries Min-Res & \bfseries Mean-ARI % specify table head
\begin{filecontents*}{CSV_Data/WoS_all_exp25_T.csv}
Alg,AvgIters,Time,AvgMinRes,MinRes,MeanARI,MeanNcut
GNCG,50.8,99.55,0.944,0.9437,0.3063,--
LAI-GNCG,72.6,9.5,0.9445,0.9443,0.2993,--
LAI-GNCG-IR,77.2,19.336,0.9438,0.9436,0.3072,--
BPP,38.1,83.603,0.9439,0.9436,0.3224,--
LAI-BPP,73.7,14.52,0.944,0.9439,0.3162,--
LAI-BPP-IR,83.5,32.027,0.9436,0.9434,0.3227,--
Comp-BPP,80.0,15.079,0.9437,0.9437,0.3254,--
LAI-HALS,95.3,8.699,0.944,0.9439,0.3242,--
LAI-HALS-IR,79.7,29.024,0.9436,0.9436,0.3297,--
HALS,58.1,111.332,0.9437,0.9436,0.3201,--
Comp-HALS,74.0,7.806,0.9439,0.9435,0.314,--
\end{filecontents*}
\csvreader[head to column names]{CSV_Data/WoS_all_exp25_T.csv}{}% use head of csv as column names
{\\\hline \Alg & \AvgIters & \Time & \AvgMinRes & \MinRes & \MeanARI}% specify your coloumns here
\end{tabular}
}
\caption{Data Table for the WoS data set. Each algorithm was run 10 times. For LAI methods $q=2$, no \ref{alg:adaRRF} is used. The columns record average number of iterations until the stopping criteria is met, average time in seconds, average minimum residual, and minimum residual achieved over all runs. Contrasting this with \cref{fig:WoS_NR_plts} and \cref{tab:WoS_symnmf_exp} shows that using \ref{alg:adaRRF} is beneficial for the WoS data set.}
\label{tab:WoS_exp_additional_q2}
\end{table}

\paragraph{WoS Experiments System Details}
The WoS experiments were run on a MacBook Pro with MATLAB 2021a.
The MacBook Pro has a 2.3 GHz Quad-Core Intel Core i7 processor and 16 GBs of RAM.
MATLAB was given access to all 4 cores.

\newpage
\subsection{Microsoft Open Academic Graph Experiments}
This section contains additional data related to the Microsoft OAG experiments.
\cref{tab:tkw_all0} shows the top 10 words for each cluster found by the HALS algorithm and \cref{tab:tkw_all1} shows the top 10 words for each cluster found by the LVS-HALS algorithm.

\paragraph{System Details}
Experiments on the OAG dataset were run on the Hive Cluster at the Georgia Institute of Technology.
The runs were given access to eight Xeon 6226 CPU @ 2.70GHz on a single shared memory node.
Experiments were run in MATLAB version 2019a.

\label{sec:moag_appendix}
\begin{table}
\resizebox{\columnwidth}{!}{
\begin{tabular}{|l||c|c|c|c|c|c|c|c|c|c|}%
\centering
\bfseries Topic & \bfseries TW1 & \bfseries TW2 & \bfseries TW3 & \bfseries TW4 & \bfseries TW5 & \bfseries TW6 & \bfseries TW7 & \bfseries TW8 & \bfseries TW9 & \bfseries TW10 % specify table head
\begin{filecontents*}{CSV_Data/pd_hals_twords.csv}
Label,twordsA,twordsB,twordsC,twordsD,twordsE,twordsF,twordsG,twordsH,twordsI,twordsJ
0,zone,mrna,receptor,apoptosis,kinase,receptors,peptide,pcr,gene,intracellular
1,teaching,article,calculation,mathematical,summary,finite,engineering,boundary,modern,beam
2,summary,hand,engineering,industry,law,mathematical,article,electric,science,calculation
3,teaching,engineering,integration,fast,summary,boundary,calculation,china,principle,principles
4,summary,equations,article,gas,modern,calculation,teaching,principle,mathematical,reliability
5,beam,teaching,integration,hand,article,summary,equations,mathematical,project,equation
6,hand,teaching,summary,engineering,project,article,mathematical,calculation,electric,technical
7,situation,summary,article,calculation,finite,teaching,engineering,hand,industrial,software
8,teaching,summary,calculation,actual,article,hand,principle,engineering,mathematical,law
9,entire,teaching,summary,project,principle,article,engineering,china,electric,mathematical
10,teaching,hand,article,calculation,summary,engineering,international,beam,finite,industrial
11,summary,teaching,china,article,calculation,construction,species,mathematical,separation,engineering
12,summary,principle,article,china,teaching,beam,calculation,law,technical,equation
13,china,summary,teaching,project,modern,hand,computer,industry,technical,principle
14,finite,teaching,mathematical,engineering,numbers,summary,beam,construction,principle,american
15,industry,summary,calculation,teaching,situation,engineering,equation,industrial,law,supply
\end{filecontents*}
\csvreader[head to column names]{CSV_Data/pd_hals_twords.csv}{}% use head of csv as column names
{\\\hline \Label & \twordsA & \twordsB & \twordsC & \twordsD & \twordsE & \twordsF & \twordsG & \twordsH & \twordsI  & \twordsJ  }% specify your coloumns here
\end{tabular}
}
\caption{Top key words for HALS output. Run on the Microsoft Open Academic Graph.}
\label{tab:tkw_all0}
\end{table}

\begin{table}
\resizebox{\columnwidth}{!}{
\begin{tabular}{|l||c|c|c|c|c|c|c|c|c|c|}%
\centering
\bfseries Topic & \bfseries TW1 & \bfseries TW2 & \bfseries TW3 & \bfseries TW4 & \bfseries TW5 & \bfseries TW6 & \bfseries TW7 & \bfseries TW8 & \bfseries TW9 & \bfseries TW10 % specify table head
\begin{filecontents*}{CSV_Data/pd_lvs_twords.csv}
Label,twordsA,twordsB,twordsC,twordsD,twordsE,twordsF,twordsG,twordsH,twordsI,twordsJ
0,zone,teaching,construction,forward,situation,law,education,chinese,china,industry
1,teaching,industrial,service,supply,actual,summary,industry,calculation,international,simulation
2,thermal,limits,equilibrium,summary,geometry,regional,heat,hand,fast,units
3,transform,signals,machine,voltage,power,parallel,signal,feature,classification,space
4,beam,storage,hand,particle,operation,modes,coupled,calculation,integration,china
5,beam,project,teaching,existence,mathematical,prove,proved,tool,modern,plane
6,equation,stability,functional,prove,physics,beam,fixed,space,weak,equations
7,fast,teaching,summary,taking,supply,form,situation,finite,article,international
8,forms,fundamental,machine,teaching,mathematical,form,measuring,social,past,operating
9,summary,minimal,hand,methyl,industry,operation,maintenance,technical,integration,project
10,matter,observations,emission,cluster,ray,infrared,mass,evolution,physics,simulations
11,optical,voltage,gain,electron,strain,layers,laser,films,crystal,layer
12,determination,organic,summary,block,absorption,iii,compounds,gel,ions,chemistry
13,flow,devices,fluid,integrated,driven,electric,film,device,voltage,liquid
14,equivalent,constant,weight,forms,coefficient,half,proved,plane,establish,form
15,signaling,transcription,kinase,receptor,genes,receptors,pathways,proteins,mediated,pathway
\end{filecontents*}
\csvreader[head to column names]{CSV_Data/pd_lvs_twords.csv}{}% use head of csv as column names
{\\\hline \Label & \twordsA & \twordsB & \twordsC & \twordsD & \twordsE & \twordsF & \twordsG & \twordsH & \twordsI  & \twordsJ  }% specify your coloumns here
\end{tabular}
}
\caption{Topwords for Leverage Score HALS output. Selected topics from the Microsoft Open Academic Graph run using HALS as the update rule with Leverage Score sampling ($\tau=1/s$). The 10 top words in terms of tf-idf association are shown in the table. We can see many of the top key words per topic seem to form a coherent subject matter.}
\label{tab:tkw_all1}
\end{table}

\subsubsection{Leverage Scores and Deterministic Sampling}
\label{sec:lvs_sampling_plts}
\Cref{fig:percent_lvs,fig:percent_det_samps} show two sets of statistics about the LvS-HALS method when applied to the OAG data set.
The two different bars distinguished by L and R denote the Left and Right factors of the SymNMF problem.
\Cref{fig:percent_det_samps} shows the percentage (y-axis) of the total fraction of samples that are being taken deterministically at each iteration $(\frac{s_D}{s_R + s_D})$.
Note that values are only plotted for every 5th iteration (x-axis).
This shows a clear trend towards taking fewer and fewer samples as the iterations progress.
\Cref{fig:percent_lvs} shows that fraction of leverage score mass or $\frac{\theta}{k}$ that is accounted for by the deterministic samples at each iteration.
The amount of leverage score mass being accounted for quickly approaches 1.
This means that a small number of deterministic samples are accounting for nearly all of the leverage score mass in the computed factor $\M{H}$.

\begin{figure}
     \centering
     \begin{subfigure}[b]{0.48\textwidth}
         \centering
         \includegraphics[width=\textwidth]{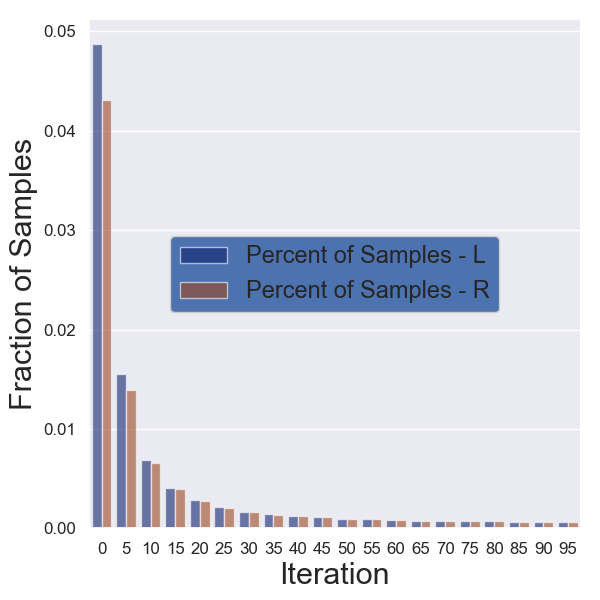}
         \caption{Fraction of samples that are taken deterministically.}
         \label{fig:percent_det_samps}
     \end{subfigure}
     %\hfill     
     \begin{subfigure}[b]{0.48\textwidth}
         \centering
         \includegraphics[width=\textwidth]{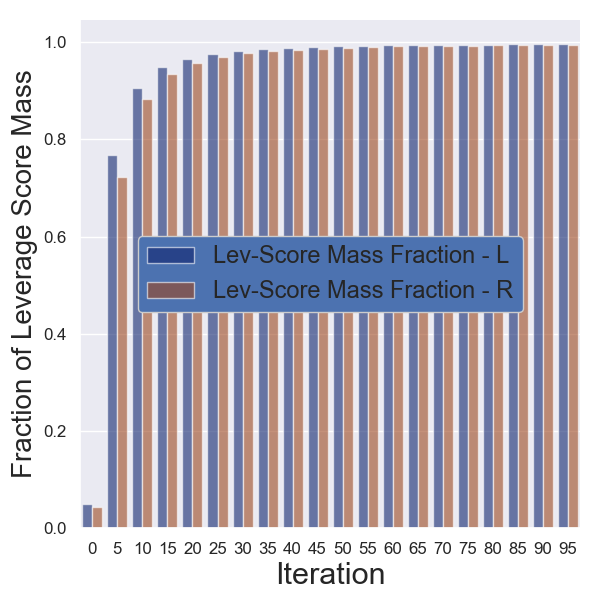}
         \caption{The normalized amount of leverage score `mass' taken deterministically ($\theta/k$) at each iteration.}
         \label{fig:percent_lvs}
     \end{subfigure}
\caption{Two statistics for the hybrid sampling approach in the LvS-HALS algorithm on the OAG data set.}
\end{figure}
%%%%%%%%%%%%%%%%%%%%%%%%%%%%%%%
%\input{Reviewer_Responses}

%\lipsum[71]

\bibliographystyle{siamplain}
\bibliography{references}
\end{document}